\documentclass[english]{article}
\usepackage[T1]{fontenc}
\usepackage[latin9]{inputenc}
\usepackage{color}
\usepackage{babel}
\usepackage{mathtools}
\usepackage{bm}
\usepackage{amsmath}
\usepackage{amsthm}
\usepackage{amssymb}
\usepackage{stmaryrd}
\usepackage{stackrel}
\usepackage{graphicx}
\usepackage{geometry}
\usepackage{setspace}
\PassOptionsToPackage{normalem}{ulem}
\usepackage{ulem}
\usepackage[pdfusetitle,
 bookmarks=true,bookmarksnumbered=false,bookmarksopen=false,
 breaklinks=false,pdfborder={0 0 1},backref=false,colorlinks=true]
 {hyperref}

\makeatletter

\newcommand{\noun}[1]{\textsc{#1}}

\theoremstyle{plain}
\newtheorem{thm}{\protect\theoremname}
\theoremstyle{definition}
\newtheorem{defn}[thm]{\protect\definitionname}
\theoremstyle{remark}
\newtheorem{rem}[thm]{\protect\remarkname}
\theoremstyle{plain}
\newtheorem{lem}[thm]{\protect\lemmaname}
\theoremstyle{plain}
\newtheorem{prop}[thm]{\protect\propositionname}

\date{}

\makeatother

\providecommand{\definitionname}{Definition}
\providecommand{\lemmaname}{Lemma}
\providecommand{\propositionname}{Proposition}
\providecommand{\remarkname}{Remark}
\providecommand{\theoremname}{Theorem}

\begin{document}
\global\long\def\R{\mathbb{R}}%

\global\long\def\C{\mathbb{C}}%

\global\long\def\N{\mathbb{N}}%

\global\long\def\e{{\mathbf{e}}}%

\global\long\def\et#1{{\e(#1)}}%

\global\long\def\ef{{\mathbf{\et{\cdot}}}}%

\global\long\def\x{{\mathbf{x}}}%

\global\long\def\xt#1{{\x(#1)}}%

\global\long\def\xf{{\mathbf{\xt{\cdot}}}}%

\global\long\def\a{{\mathbf{a}}}%

\global\long\def\b{{\mathbf{b}}}%

\global\long\def\c{{\mathbf{c}}}%

\global\long\def\d{{\mathbf{d}}}%

\global\long\def\w{{\mathbf{w}}}%

\global\long\def\b{{\mathbf{b}}}%

\global\long\def\u{{\mathbf{u}}}%

\global\long\def\y{{\mathbf{y}}}%

\global\long\def\n{{\mathbf{n}}}%

\global\long\def\k{{\mathbf{k}}}%

\global\long\def\valpha{\mathbf{\alpha}}%

\global\long\def\yt#1{{\y(#1)}}%

\global\long\def\yf{{\mathbf{\yt{\cdot}}}}%

\global\long\def\z{{\mathbf{z}}}%

\global\long\def\v{{\mathbf{v}}}%

\global\long\def\h{{\mathbf{h}}}%

\global\long\def\q{{\mathbf{q}}}%

\global\long\def\s{{\mathbf{s}}}%

\global\long\def\p{{\mathbf{p}}}%

\global\long\def\f{{\mathbf{f}}}%

\global\long\def\rb{{\mathbf{r}}}%

\global\long\def\rt#1{{\rb(#1)}}%

\global\long\def\rf{{\mathbf{\rt{\cdot}}}}%

\global\long\def\mat#1{{\ensuremath{\bm{\mathrm{#1}}}}}%

\global\long\def\vec#1{{\ensuremath{\bm{\mathrm{#1}}}}}%

\global\long\def\ten#1{{\cal #1}}%

\global\long\def\matN{\ensuremath{{\bm{\mathrm{N}}}}}%

\global\long\def\matX{\ensuremath{{\bm{\mathrm{X}}}}}%

\global\long\def\X{\ensuremath{{\bm{\mathrm{X}}}}}%

\global\long\def\matK{\ensuremath{{\bm{\mathrm{K}}}}}%

\global\long\def\K{\ensuremath{{\bm{\mathrm{K}}}}}%

\global\long\def\matA{\ensuremath{{\bm{\mathrm{A}}}}}%

\global\long\def\A{\ensuremath{{\bm{\mathrm{A}}}}}%

\global\long\def\matB{\ensuremath{{\bm{\mathrm{B}}}}}%

\global\long\def\B{\ensuremath{{\bm{\mathrm{B}}}}}%

\global\long\def\matC{\ensuremath{{\bm{\mathrm{C}}}}}%

\global\long\def\C{\ensuremath{{\bm{\mathrm{C}}}}}%

\global\long\def\matD{\ensuremath{{\bm{\mathrm{D}}}}}%

\global\long\def\D{\ensuremath{{\bm{\mathrm{D}}}}}%

\global\long\def\matE{\ensuremath{{\bm{\mathrm{E}}}}}%

\global\long\def\E{\ensuremath{{\bm{\mathrm{E}}}}}%

\global\long\def\matF{\ensuremath{{\bm{\mathrm{F}}}}}%

\global\long\def\F{\ensuremath{{\bm{\mathrm{F}}}}}%

\global\long\def\matH{\ensuremath{{\bm{\mathrm{H}}}}}%

\global\long\def\H{\ensuremath{{\bm{\mathrm{H}}}}}%

\global\long\def\matP{\ensuremath{{\bm{\mathrm{P}}}}}%

\global\long\def\matG{\ensuremath{{\bm{\mathrm{G}}}}}%

\global\long\def\P{\ensuremath{{\bm{\mathrm{P}}}}}%

\global\long\def\matU{\ensuremath{{\bm{\mathrm{U}}}}}%

\global\long\def\matV{\ensuremath{{\bm{\mathrm{V}}}}}%

\global\long\def\V{\ensuremath{{\bm{\mathrm{V}}}}}%

\global\long\def\U{\ensuremath{{\bm{\mathrm{U}}}}}%

\global\long\def\Y{\ensuremath{{\bm{\mathrm{Y}}}}}%

\global\long\def\matW{\ensuremath{{\bm{\mathrm{W}}}}}%

\global\long\def\matM{\ensuremath{{\bm{\mathrm{M}}}}}%

\global\long\def\M{\ensuremath{{\bm{\mathrm{M}}}}}%

\global\long\def\tenA{\ten A}%

\global\long\def\tenB{\ten B}%

\global\long\def\tenC{\ten C}%

\global\long\def\tenD{\ten D}%

\global\long\def\tenX{\ten X}%

\global\long\def\tenC{\ten C}%

\global\long\def\tenW{\ten W}%

\global\long\def\tenG{\ten G}%

\global\long\def\tenJ{\ten J}%

\global\long\def\tenZ{\ten Z}%

\global\long\def\matQ{{\mat Q}}%

\global\long\def\Q{{\mat Q}}%

\global\long\def\matR{\mat R}%

\global\long\def\matS{\mat S}%

\global\long\def\matY{\mat Y}%

\global\long\def\matI{\mat I}%

\global\long\def\I{\mat I}%

\global\long\def\matJ{\mat J}%

\global\long\def\matZ{\mat Z}%

\global\long\def\Z{\mat Z}%

\global\long\def\matW{{\mat W}}%

\global\long\def\W{{\mat W}}%

\global\long\def\matL{\mat L}%

\global\long\def\manM{{\cal M}}%

\global\long\def\TNormS#1{\|#1\|_{2}^{2}}%

\global\long\def\ITNormS#1{\|#1\|_{2}^{-2}}%

\global\long\def\ITNorm#1{\|#1\|_{2}^{-1}}%

\global\long\def\ONorm#1{\|#1\Vert_{1}}%

\global\long\def\TNorm#1{\|#1\|_{2}}%

\global\long\def\InfNorm#1{\|#1\|_{\infty}}%

\global\long\def\FNorm#1{\|#1\|_{F}}%

\global\long\def\FNormS#1{\|#1\|_{F}^{2}}%

\global\long\def\UNorm#1{\|#1\|_{\matU}}%

\global\long\def\UNormS#1{\|#1\|_{\matU}^{2}}%

\global\long\def\UINormS#1{\|#1\|_{\matU^{-1}}^{2}}%

\global\long\def\ANorm#1{\|#1\|_{\matA}}%

\global\long\def\BNorm#1{\|#1\|_{\mat B}}%

\global\long\def\ANormS#1{\|#1\|_{\matA}^{2}}%

\global\long\def\AINormS#1{\|#1\|_{\matA^{-1}}^{2}}%

\global\long\def\WNorm#1{\|#1\|_{\matW}}%

\global\long\def\T{\textsc{T}}%

\global\long\def\conj{\textsc{*}}%

\global\long\def\pinv{\textsc{+}}%

\global\long\def\Prob{\operatorname{Pr}}%

\global\long\def\Expect{\operatorname{\mathbb{E}}}%

\global\long\def\ExpectC#1#2{{\mathbb{E}}_{#1}\left[#2\right]}%

\global\long\def\VarC#1#2{{\mathbb{\mathrm{Var}}}_{#1}\left[#2\right]}%

\global\long\def\dotprod#1#2#3{(#1,#2)_{#3}}%

\global\long\def\dotprodN#1#2{(#1,#2)_{{\cal N}}}%

\global\long\def\dotprodH#1#2{\langle#1,#2\rangle_{{\cal {\cal H}}}}%

\global\long\def\dotprodsqr#1#2#3{(#1,#2)_{#3}^{2}}%

\global\long\def\Trace#1{{\bf Tr}\left(#1\right)}%

\global\long\def\Vec#1{{\bf vec}\left(#1\right)}%

\global\long\def\nnz#1{{\bf nnz}\left(#1\right)}%

\global\long\def\MSE#1{{\bf MSE}\left(#1\right)}%

\global\long\def\WMSE#1{{\bf WMSE}\left(#1\right)}%

\global\long\def\EWMSE#1{{\bf EWMSE}\left(#1\right)}%

\global\long\def\nicehalf{\nicefrac{1}{2}}%

\global\long\def\nicebetahalf{\nicefrac{\beta}{2}}%

\global\long\def\argmin{\operatornamewithlimits{argmin}}%

\global\long\def\argmax{\operatornamewithlimits{argmax}}%

\global\long\def\norm#1{\Vert#1\Vert}%

\global\long\def\squeeze#1{\operatorname{squeeze}_{#1}}%

\global\long\def\sign{\operatorname{sign}}%

\global\long\def\proj{\operatorname{proj}}%

\global\long\def\diag{\operatorname{diag}}%

\global\long\def\VOPT{\operatorname\{VOPT\}}%

\global\long\def\dist{\operatorname{dist}}%

\global\long\def\diag{\operatorname{diag}}%

\global\long\def\supp{\operatorname{supp}}%

\global\long\def\sp{\operatorname{span}}%

\global\long\def\rank{\operatorname{rank}}%

\global\long\def\grad{\operatorname{grad}}%

\global\long\def\Hess{\operatorname{Hess}}%

\global\long\def\abs{\operatorname{abs}}%

\global\long\def\onehot{\operatorname{onehot}}%

\global\long\def\softmax{\operatorname{softmax}}%

\newcommand*\diff{\mathop{}\!\mathrm{d}} 

\global\long\def\dd{\diff}%

\global\long\def\whatlambda{\w_{\lambda}}%

\global\long\def\Plambda{\mat P_{\lambda}}%

\global\long\def\Pperplambda{\left(\mat I-\Plambda\right)}%

\global\long\def\Mlambda{\matM_{\lambda}}%

\global\long\def\Mlambdafull{\matM+\lambda\matI}%

\global\long\def\Mlambdafullinv{\left(\matM+\lambda\matI\right)^{-1}}%

\global\long\def\Mdaggerlambda{{\mathbf{\mat M_{\lambda}^{+}}}}%

\global\long\def\Xdaggerlambda{{\mathbf{\mat X_{\lambda}^{+}}}}%

\global\long\def\XT{{\mathbf{X}^{\T}}}%

\global\long\def\XXT{{\matX\mat X^{\T}}}%

\global\long\def\XTX{{\matX^{\T}\mat X}}%

\global\long\def\VT{{\mathbf{V}^{\T}}}%

\global\long\def\VVT{{\matV\mat V^{\T}}}%

\global\long\def\VTV{{\matV^{\T}\mat V}}%

\global\long\def\varphibar#1#2{{\bar{\varphi}_{#1,#2}}}%

\global\long\def\varphilambda{{\varphi_{\lambda}}}%

\global\long\def\RE#1{{\bf Re}\left(#1\right)}%

\global\long\def\modeprod#1{\times_{#1}}%

\global\long\def\modeprodrange#1#2{\times_{#1}^{#2}}%

\global\long\def\unfold#1#2{#1_{(#2)}}%

\global\long\def\HOSVD#1#2{{\bf HOSVD}_{#2}\left(#1\right)}%

\title{Higher Order Reduced Rank Regression}
\author{Leia Greenberg and Haim Avron\\
	School of Mathematical Sciences\\
	Tel Aviv University}
\maketitle
\begin{abstract}
Reduced Rank Regression (RRR) is a widely used method for multi-response
regression. However, RRR assumes a linear relationship between features
and responses. While linear models are useful and often provide a
good approximation, many real-world problems involve more complex
relationships that cannot be adequately captured by simple linear
interactions. One way to model such relationships is via multilinear
transformations. This paper introduces Higher Order Reduced Rank Regression
(HORRR), an extension of RRR that leverages multi-linear transformations,
and as such is capable of capturing nonlinear interactions in multi-response
regression. HORRR employs tensor representations for the coefficients
and a Tucker decomposition to impose multilinear rank constraints
as regularization akin to the rank constraints in RRR. Encoding these
constraints as a manifold allows us to use Riemannian optimization
to solve this HORRR problems. We theoretically and empirically analyze
the use of Riemannian optimization for solving HORRR problems. 
\end{abstract}

\section{Introduction}

Regression is a fundamental concept in data science, often used to
understand and model relationships between variables. The primary
goal of regression analysis is to predict or explain the value of
a \emph{dependent} variable (also called the \emph{response}) based
on one or more \emph{independent} variables (also known as \emph{predictors}
or \emph{features}). By examining how changes in these features influence
the response, regression helps identifying trends, making predictions,
and interpreting data patterns. Classical descriptions of regression
typically discuss \emph{single response} regression, i.e. there is
a single dependent variable. However, in many real-world cases we
are interested not just in modeling a single response variable but
multiple response variables simultaneously. This leads to the problem
of \emph{multi-response regression}.

One simple and naive approach for multi-response regression is treating
each response variable separately and building an independent regressor
using a single-response regression method. However, this approach
is sub-optimal when there are correlations or dependencies between
the response variables. Exploiting these relationships can lead to
better models for a given set of data. For example, in multi-class
classification algorithms such as regularized least squares classification,
we often only want to predict one label per example, which requires
accounting for relationships between responses. In general, handling
responses independently is a poor strategy when there's significant
inter-dependence among the variables \cite{reinsel-velu-1998RRR}.

One effective way to handle multiple responses together is through
regularization techniques, which impose constraints on the model to
prevent over-fitting while encouraging simpler, more interpretable
solutions. In the context of least squares linear regression, a fundamental
regularization based multi-response regression method is \emph{Reduced
Rank Regression} (RRR, also referred to as low-rank regression \cite{izenman1975-RRR}).
RRR constrains the rank of the coefficient matrix. Implicitly, RRR
assumes that response vectors belong to some low-dimensional subspace.
RRR is widely used in fields such as chemometrics, for example in
analyzing how multiple reactor conditions affect the qualities of
a produced polymer \cite{Mukherjee2011kernel}; genomics, where it
helps predict multiple gene expressions \cite{she2017robust}; machine
learning \cite{liu2024efficient} and more.\textcolor{blue}{{} }

While assumption made by RRR that the response vectors belong to same
low dimensional subspace might hold for some scenarios, there are
important cases of multi-response regression in which it does not
hold. Returning to the example of multi-class classification using
regularized least squares classification, the response vectors, which
are one-hot encoding of the classes, do not belong to a low-rank subspace
if all classes are present in the training data. Furthermore, RRR
only handles linear models, where the relationship between the independent
variables and the multiple response variables is assumed to be linear.
While linear models are useful and often provide a good approximation,
many real-world problems involve more complex relationships that cannot
be adequately captured by simple linear interactions. In many cases,
the response variables may depend on higher-order interactions between
the input features, such as polynomial relationships. To account for
these more complex dependencies, we need to extend the model beyond
the linear framework. This involves incorporating nonlinear terms
or transformations of the input features to allow for more flexible
modeling of the data. Techniques such as polynomial regression, kernel
methods, or deep learning models enable the capture of these higher-order
interactions. By moving beyond the constraints of linearity, we can
model more intricate and subtle patterns in the data, ultimately improving
the model's ability to predict multiple responses and capture the
true underlying relationships between the variables.

In this paper we propose HORRR, a form of higher order reduced rank
regression, i.e. we extend RRR to capture more complex interactions
by using higher-order terms. Specifically, the coefficients of homogeneous
multivariate polynomials, which represent these higher-order interactions,
can be expressed as tensors. These multidimensional arrays provide
a natural framework for representing the coefficients in a multi-dimensional
structure, allowing for a more flexible and expressive model. By applying
rank constraints to these tensors, we can achieve dimensionality reduction
while still capturing the complex, nonlinear relationships between
the features and responses. This approach combines the interpretability
and efficiency of reduced-rank methods with the power of higher-order
polynomials, making it a promising extension for multi-response regression
problems where higher-order interactions play a significant role.
As side benefit, HORRR does not assume the responses belong to a low-rank
subspace.

Learning HORRR models amounts to minimizing a loss subject to the
tensor multilinear rank constraints (see Section \ref{sec:HORRR problem statement}).
To solve such problems, we propose the use of a Riemannian manifold
optimization approach, as the rank constraint forms a smooth manifold
and provides a natural geometric framework for optimization. Our suggested
algorithm exploits the structure of the manifold for an efficient
algorithm. Multilinear rank constraints amount to representing candidate
coefficients as a Tucker decomposed tensor, i.e. breaking down the
coefficients tensor into a core tensor and factor matrices, allowing
for more manageable computations in high-dimensional spaces. By incorporating
the Tucker decomposition, our algorithm can efficiently handle the
complexity of high-order interactions while maintaining computational
feasibility. Optimizing over this manifold allows us to preserve the
multilinear structure of the regression problem, while ensuring that
the solution adheres to the desired rank. 

Our experiments show that HORRR exhibits robust performance on synthetic
datasets, achieving low relative recovery errors. When used to learn
a regularized least squares classifier for the MNIST dataset, the
model achieved accuracy comparable to exact kernel-based methods using
the polynomial kernel, while using significantly less stroage than
a kernel solver. These illustrates HORRR's potential for downstream
learning workloads.

\subsection{Related Work}

\paragraph{Reduced Rank Regression.}

RRR is a well-known method for linear multi-response regression. It
was first proposed by Anderson \cite{anderson1951RRR}, and later
Izenman \cite{izenman1975-RRR} coined the term \textit{Reduced Rank
Regression} (RRR) for this class of models. Davies and Tso \cite{davies1982-RRR-svd}
offered a solution that involves projection of the OLS estimator to
a lower rank sub-space, that explains the maximum variation in terms
of the Frobenius norm, while relying on the Eckart-Young theorem for
the best lower rank approximation \cite{eckart-young-1936}. Theoretical
analysis of asymptotic distribution of the estimator is given in\textcolor{blue}{{}
\cite{anderson1999asymptotic}}. See Reinsel and Velu \cite{reinsel-velu-1998RRR}
and Izenman \cite{izenman2008modern} for a comprehensive review of
RRR methods and applications. We remark that there are also other
methods for linear multi-response regression, such as Partial Least
Squares (PLS) \cite{wold1975PLS}.

Classical RRR uses a low-rank constraint. Adding a ridge penalty term
is natural, to prevent over-fitting and to avoid numerical instabilities,
and in this paper we consider this form of RRR. RRR with ridge penalty
was first considered by \cite{Mukherjee2011kernel}. Other forms of
additional regularization were studied, such as nuclear norm penalty
\cite{chen2013reduced,yuan2007dimension}, sparsity penalties \cite{chen2012reduced}
and more \cite{levakova2024penalisation}.

\paragraph{Tensor Regression.}

Data with a natural tensor structure are encountered in many scientific
areas. In the context of regression, this leads to several methods
which can fall under the umbrella term ``Tensor Regression''. Typically,
these methods assume tensor structure on the features and/or responses.
The work closest to ours is Higher Order Low Rank Regression (HOLRR)~\cite{rabusseau2016HOLRR}.
In HOLRR, like in HORRR, the regressors are arranged as a low rank
multilinear tensor, and the features are vectors. However, in HOLRR
the responses are arranged as a tensor, while in HORRR they are arranged
as a vector. Furthermore, the regression function in HOLRR is defined
by mode-1 tensor-matrix product, while in RRR we use the tensor apply
operation. In \cite{llosa2022reduced} and \cite{yu2016learning}
both the regressors and features are arranged as tensors, with the
regressors having low multilinear rank, and the responses are scalars.
A form of higher-order PLS, called HOPLS, was proposed in \cite{zhao2012HOPLS},
where both the response and the features are tensors.

\paragraph{Riemannian Optimization on Manifolds of Tensors of Low Multilinear
Rank.}

Our method uses Riemannian optimization to solve HORRR. Riemannian
optimization is a generalization of standard Euclidean optimization
methods to smooth manifolds. The core idea is to reformulate the optimization
problem on a smooth manifold and eliminate the need for explicit constraint
handling, as feasible points are maintained by construction \cite{absil2008optimization-book}.
This optimization provides an elegant way to address the multilinear
rank constraint by utilizing the fact that the set of tensors with
a fixed multilinear rank forms a smooth embedded submanifold. Several
works take advantage of this, e.g. tensor completion \cite{Kressner2013completion,kasai2016low}.
Koch and Lubich \cite{KochLuibch10tensor-manifold} as well as Kressner
et al. \cite{Kressner2013completion,Kressner2016precond} developed
the geometric components necessary for Riemannian optimization on
the manifold of multilinear rank tensors. Heidel and Schulz \cite{heidel2018riemannian}
developed second-order components.

\section{Preliminaries}

\subsection{Notations and Basic Definitions}

We use lower case bold letters for vectors ($\boldsymbol{\mathbf{v,w}},...$),
upper case bold letters for matrices ($\mathbf{X,Y,..}$) and calligraphic
letters for higher order tensors ($\mathcal{A\boldsymbol{\mathrm{,}}B\mathrm{,..}}$).
For $n\in\N$ we denote $[n]\coloneqq\{1,\dots n\}$. The identity
matrix of size $n$ is written as $\mathbf{I}_{n}$. The $i$th row
(respectively column) of a matrix $\boldsymbol{\mathbf{M}}$ will
be denoted by $\mathrm{M}_{i,:}$(respectively $\mathrm{M}_{:,i}$).
For matrices $\matX,\matY$ , we denote $\X\otimes\matY$ as the Kronecker
product, $\matX\odot\matY$ as the the Khatri-Rao product and $\X\circ\matY$
as the Hadamard product. 

For a matrix $\A,$ we denote its low-rank approximation as $[\X]_{r}$,
obtained by truncating the Singular Value Decomposition (SVD) to the
first $r$ components (which is the best low-rank approximation for
$\X$, due to the Eckart--Young Theorem). $\X^{+}$ denotes the Moore--Penrose
pseudo-inverse of a matrix $\X$.

For tensors, we use the same basic definitions and concepts as Kolda
and Bader in their classical review~\cite{KoldaBader09}. A tensor
$\mathcal{A}\in\R^{n_{1}\times n_{2}\times...\times n_{d}}$ is a
multidimensional array of \emph{order }(i.e. dimension) $d$. An order
$d=1$ tensor is just a vector, and an order $d=2$ is a matrix. Each
dimension of a tensor is referred to as a \emph{mode}. A $j$-\emph{mode
unfolding} (or \emph{matricization}) of a tensor $\tenA\in\R^{n_{1}\times\cdots\times n_{d}}$
along mode $j$, is a matrix $\mathbf{\A}_{(j)}\in\R^{n_{j}\times(n_{1}\cdot...\cdot n_{j-1}\cdot n_{j+1}\cdot...n_{d})}$.
The vectorization of a tensor is defined by $\Vec{\tenA}=\Vec{\A_{(1)}}$.
The inner product between two tensors $\mathcal{A}$ and $\mathcal{B}$
(of the same size) is defined by $\left\langle \mathcal{\tenA},\mathcal{\tenB}\right\rangle =\left\langle \Vec{\tenA},\Vec{\tenB}\right\rangle $
and the Frobenius norm is defined by $\left\Vert \mathcal{A}\right\Vert _{F}^{2}=\left\langle \mathcal{A},\mathcal{A}\right\rangle $.
The $j$-\emph{mode product} between a tensor $\mathcal{A}\in\R^{n_{1}\times\cdots\times n_{d}}$
and a matrix $\mathbf{M}\in\R^{m\times n_{j}}$, written as $\mathcal{A}\times_{j}\mathbf{M}$,
and is defined by the equation $(\mathcal{A}\times_{j}\mathbf{M})_{(j)}=\mathbf{M}\mathcal{\tenA}_{(j)}$.
Suppose that $\tenA\in\R^{n_{1}\times\cdots\times n_{d}}$ is a tensor
whose $j$-th dimension is a singleton (i.e., $n_{j}=1$). We denote
by $\squeeze j(\tenA)$ the tensor obtained by removing dimension
$j$. That is, if $\tenB=\squeeze j(\tenA)$ then $\tenB\in\R^{n_{1}\times n_{2}\times\cdots\times n_{j-1}\times n_{j+1}\times\cdots\times n_{d}}$
and $\tenB_{i_{1},\cdots,i_{j-1},i_{j+1},\cdots,i_{d}}=\tenA_{i_{1},\cdots,i_{j-1},1,i_{j+1},\cdots,i_{d}}$.

A tensor of order $d$ is\textit{ rank one} if it can be written as
the outer product of $d$ vectors, i.e, $\tenX=\a^{(1)}\circ\a^{(2)}\circ...\circ\a^{(d)}$.
The\textit{ CANDECOMP/PARAFAC (CP)}\textit{\emph{ }}\emph{decomposition}
factorized a tensor $\tenX\in\R^{n_{1}\times\cdots\times n_{d}}$
into a sum of a finite number of rank-one tensors 
\begin{equation}
\tenX=\mathop{\sum_{r=1}^{R}\omega_{r}\a_{r}^{(1)}\circ\a_{r}^{(2)}\circ...\circ\a_{r}^{(d)}}\eqqcolon\left\llbracket \boldsymbol{\omega};\A^{(1)},\dots,\A^{(d)}\right\rrbracket \label{eq:CP-decomp}
\end{equation}
where\emph{ $\A_{i}\in\R^{n_{i}\times R}$ }are the \emph{factor matrices,}
each consists the columns\emph{ $\a_{1}^{(i)},\a_{2}^{(i)}...\a_{R}^{(i)}$,
}for $i=[d]$. We assume that the columns of the factor matrices are
normalized to length 1, with the weights absorbed into the vector
$\mathbf{\boldsymbol{\omega}}\in\R^{R}$. Note that we introduced
the notation $\left\llbracket \boldsymbol{\omega};\A^{(1)},\dots,\A^{(d)}\right\rrbracket $
for a compact writing of CP decomposition. The \textit{rank} of a
tensor $\tenX$, denoted $\rank(\tenX),$ is defined as the smallest
number of rank-one tensors that generate $\tenX$ as their sum (i.e.
minimal $R$ possible in Eq. (\ref{eq:CP-decomp})). The \textit{rank}
of a tensor should not be confused with the \emph{multilinear rank
}defined later in Eq.\textcolor{blue}{{} }(\ref{def: multilinear rank}). 

For a tensor $\tenA\in\R^{n_{1}\times\cdots\times n_{d}}$ we define
the \textit{tensor transposition} of it by a permutation $\sigma$
on $[n]$ as the tensor $\tenA_{\sigma}$ defined by $(\tenA_{\sigma})_{\sigma(1),\sigma(2),...,\sigma(d)}=\tenA_{i_{1},i_{2},...,i_{d}}$.
A tensor is called \textit{symmetric} if all it's dimensions are equal
and it is invariant under any permutations, and \textit{semi-symmetric}
if all it's dimensions but the first are equal and it is invariant
under any permutations\textit{ }in all indices but the the first.

\subsection{Reduced Rank (Linear) Regression\label{subsec:RRR}}

Multi-response linear regression is the extension of the classical
single response regression model to the case where we have more than
a single variable we want to predict. One can treat multi-response
regression problems as a collection of independent single response
problems. However, this ignores probable correlations between the
dependent variables, and is clearly sub-optimal. RRR is a simple and
computationally efficient method for multi-response regression, that
introduced low-rank constraints on the coefficients of the separate
single-response models arranged as a matrix \cite{izenman2008modern}. 

Concretely, given data in the form of matrices $\X\in\R^{m\times n}$
and $\mathbf{Y}\in\R^{k\times n}$, where $\matX$ collects $n$ samples
from $m$ dimensional predictor variables, and $\matY$ collects the
corresponding $k$ dimensional responses, and given some rank parameter
$r\leq\min(k,m)$, the \emph{RRR Problem} reads:\footnote{Typically, data is arranged as rows, \textcolor{black}{but for our
problem} arrangement as columns is more convenient.}

\begin{equation}
\min_{\W\in\R^{k\times m}}\FNormS{\W\matX-\matY}\;\;\;\text{s.t.}\;\rank(\W)\leq r\label{eq:RRR-Problem}
\end{equation}

The RRR Problem can be solved in closed-form, and there are multiple
formulas for this in the literature. If all the singular values of
projections of the rows of $\matY$ on the row space of $\matX$ are
distinct, there is a unique solution for each $r$ \cite{davies1982-RRR-svd},
which we denote by $\matW_{r}^{(\text{RRR})}$. Davies and Tso \cite{davies1982-RRR-svd}
further show how to use SVD and optimal low-rank approximation (based
on Eckart-Young Theorem) for a closed-form solution for Eq.~(\ref{eq:RRR-Problem}),
based on computing the unconstrained ordinary least squares solution,
and then projecting it to a $r$-dimensional subspace obtained by
SVD of some matrix. In our notation, this reads as:

\begin{equation}
\matW_{r}^{(\text{RRR})}=\left[\mathbf{Y}\X^{+}\X\right]_{r}\matX^{\pinv}\label{eq:RRR-classic-solution}
\end{equation}

Mukherjee and Zhu introduced a ridge regression variant of RRR \cite{Mukherjee2011kernel}.
Consider an additional ridge parameter $\lambda\geq0$, the ridge
RRR problem reads (when the data is arranged in rows): 

\begin{equation}
\matW_{r,\lambda}^{(\text{RRR})}\coloneqq\argmin_{\W\in\R^{k\times m}}\FNormS{\W\matX-\matY}+\lambda\FNormS{\matW}\;\;\;\text{s.t.}\;\rank(\W)\leq r\label{eq:R4-Problem}
\end{equation}
(In the above, we assume the solution is unique. It is possible to
write sufficient and necessary conditions for uniqueness. However,
these are not important for our discussion). They show that the closed
form solution is 
\[
\matW_{r,\lambda}^{(\text{RRR})}=\left[\mathbf{\hat{Y}}\hat{\X}_{\lambda}^{+}\hat{\X}_{\lambda}\right]_{r}\hat{\matX}_{\lambda}^{\pinv}
\]
where 
\[
\hat{\X}_{\lambda}\coloneqq[\X\;\;\sqrt{\lambda}\I_{m}],\;\;\hat{\matY}\coloneqq[\matY\;\;\mathbf{0}_{\mathrm{\mathit{k\times m}}}]
\]
\textcolor{black}{They also introduce a kernel variant of the problem. }

\subsection{Tensor Apply Operation}

We wish to move beyond the linear framework used in RRR, and incorporate
nonlinear terms. This can be done by capturing higher-order interactions,
which can be represented by a tensor, which operates on a vector through
the\emph{ tensor-vector} \textit{apply} operation which we now define.

\begin{defn}
[\cite{Qi05}] Given an order $d+1$ tensor $\tenA\in\R^{k\times m\times\cdots\times m}$
and a vector $\x\in\R^{m}$, \emph{applying} $\tenA$ \emph{to} $\x$
is the vector $\y=\tenA\x\in\R^{k}$ which is defined by\footnote{The operation is not named in it's original introduction~\cite{Qi05}.
We call the operation ``apply'' following \cite{BensonGleich19eigen}. }
\begin{equation}
y_{j}=\sum_{i_{2}=1}^{m}\cdots\sum_{i_{d+1}=1}^{m}a_{ji_{2}\cdots i_{d+1}}x_{i_{2}}x_{i_{3}}\cdots x_{i_{d+1}}\label{eq:apply-def}
\end{equation}
for $j=1,\dots,k$.
\end{defn}

\begin{rem}
\label{rem:apply-and-polynomials}Via the apply operation, a tensor
$\tenA\in\R^{k\times m\times\cdots\times m}$ can be viewed as a non-linear
operator from $\R^{m}$ to $\R^{k}$. Inspecting Eq.~(\ref{eq:apply-def}),
we see that the entries of $\tenA$ can be thought of the coefficients
of $k$ homogeneous multivariate polynomials of degree $d$, and applying
$\tenA$ to a vector $\x$ results in vector that contains the evaluation
of the these polynomials on $\x$.
\end{rem}

\begin{rem}
While the map $(\tenA,\x)\mapsto\tenA\x$ is generically non-linear
in $\x$, it is actually linear in $\tenA$.\textcolor{red}{{} }
\end{rem}

Two additional compact formulas for $\tenA\x$ are the following:
\begin{equation}
\tenA\x=\tenA\modeprod 2\x^{\T}\modeprod 3\x^{\T}\cdots\modeprod{d+1}\x^{\T}\label{eq:apply-alt1}
\end{equation}
\begin{equation}
\tenA\x=\unfold{\tenA}1\x^{\odot d}\label{eq:apply-alt2}
\end{equation}
where power of $\odot d$ denotes $d$-times Khatri-Rao product, e.g.
$\x^{\odot d}=\underbrace{\x\odot\x\odot\cdots\odot\x}_{d\,\text{times}}$. 

Next, we extend the definition of applying a tensor to matrices. Applying
$\tenA$ to $\matX\in\R^{m\times n}$ is the matrix $\matY=\tenA\matX\in\R^{k\times n}$
which is obtained by applying $\tenA$ to each column of $\matX$.
That is, if 
\[
\matX=\left[\begin{array}{cccc}
\vdots & \vdots &  & \vdots\\
\x_{1} & \x_{2} & \cdots & \x_{n}\\
\vdots & \vdots &  & \vdots
\end{array}\right]
\]
then 
\[
\tenA\matX=\left[\begin{array}{cccc}
\vdots & \vdots &  & \vdots\\
\tenA\x_{1} & \tenA\x_{2} & \cdots & \tenA\x_{n}\\
\vdots & \vdots &  & \vdots
\end{array}\right]
\]
From Eq.~(\ref{eq:apply-alt2}) we have 
\begin{equation}
\tenA\matX=\unfold{\tenA}1\matX^{\odot d}.\label{eq:apply-matrix-kr}
\end{equation}
Note that $\matX^{\odot d}\in\R^{m^{d}\times n}$. The cost of computing
$\matX^{\odot d}$ is $O(m^{d}n)$. \footnote{Easily seen from the recursive formula $T(m,n,d)=T(m,n,d-1)+m^{d}n$.}
The cost of the multiplication is $O(knm^{d})$, so the total cost
is $O(knm^{d})$.

\subsection{Riemannian Optimization}

Our proposed algorithms for HORRR are based on Riemannian optimization.
In this section, we review some fundamental definitions of Riemannian
optimization. We do not aim for the introduction to be comprehensive;
we refer the reader to \cite{absil2008optimization-book} for a more
thorough introduction. Our main goal is to introduce the relevant
concepts and notations.

Riemannian optimization is an approach for solving constrained optimization
problems in which the constraints form a smooth manifold. This approach
adapts classical algorithms for unconstrained optimization on a vector
space equipped with an inner product, by generalizing the main components
needed to apply these algorithms, to search spaces that form smooth
manifolds. Riemannian optimization has been applied to solve a wide-array
of problems involving matrix or manifold constraints. However, most
relevant for our purposes is that Riemannian optimization has been
applied to solve problems involving constraints on the multilinear
rank of tensors, specifically using the Tucker decomposition, e.g.,
\cite{Kressner2013completion,Kressner2016precond,KochLuibch10tensor-manifold}.

A Riemannian manifold $\mathcal{M}$ is a real differentiable manifold
with a smoothly varying inner product, called Riemannian metric, on
tangent spaces $T_{\mathcal{A}}\mathcal{M}$ for $\mathcal{A}\in\mathcal{M}.$
Riemannian optimization often deals with constrained problems where
the cost function is defined outside the constrained set. In other
words, the search space is embedded within a larger space, and the
objective function is expressed in the coordinates of this embedding
space. A \textit{Riemannian submanifold} is a Riemannian manifold
in which is an embedded manifold of a larger space and inherits the
Riemannian metric in a natural way \cite[Section 3.3]{absil2008optimization-book}.
A tensor Riemannian manifold is a Riemannian submanifold of the (Euclidean)
vector space of tensors of a fixed order and dimension. 

By encoding the constraints as manifolds, and allowing the optimization
algorithm to be informed by the geometry of the manifold, we can treat
the problem as unconstrained. To that end, we need to generalize various
notions of unconstrained optimization to ones that respect the manifold.
The notions of \textit{Riemannian gradient} and \textit{Riemannian
Hessian} \cite[Sections 3.6 and 5.5]{absil2008optimization-book}
extend the corresponding concepts from the Euclidean setting. For
a smooth (objective) function defined on a manifold, $F:\mathcal{M}\rightarrow\R$,
we denote the Riemannian gradient and Riemannian Hessian at $\mathcal{A\in\mathcal{M}}$
by $\grad F(\mathcal{A})\in T_{\mathcal{A}}\mathcal{M}$ and $\Hess F(\mathcal{A}):T_{\mathcal{A}}\mathcal{M}\longrightarrow T_{\mathcal{A}}\mathcal{M}$
respectively. The notion of \textit{retraction} \cite[Section 4.1]{absil2008optimization-book}
allows Riemannian optimization algorithms to take a step at point
$\mathcal{A\in\mathcal{M}}$ in a direction $\xi_{\mathcal{A}}\in T_{\mathcal{A}}\mathcal{M}$.
A retraction is a map $R_{\mathcal{\mathcal{A}}}:T_{\mathcal{A}}\mathcal{M}\longrightarrow\mathcal{M}$
that upholds a local rigidity condition that preserves gradients at
$\tenA\in\mathcal{M}$, and is a first order approximation of the
notion of \emph{exponential mapping} which generalize the idea of
moving in straight lines. Additional concepts can be defined, e.g.
\textit{vector transport }and\textit{ Riemannian connections,} however
they do not play a role in our presentation, so we do not discuss
them.

These components allow for the natural generalization of various optimization
algorithms for smooth problems. The core concept of Riemannian optimization
algorithms is to approximate the manifold locally by its tangent space
at each iteration. First order methods' iterations moves along the
tangent space first, and then the point is mapped back to the manifold
using a retraction. 

Analyzing Riemannian optimization problems also rely on the concepts
described above. The Riemannian gradient is used for finding \emph{Riemannian
stationary points}, while the Riemannian Hessian classifies them.\textcolor{red}{{}
}A stationary point is such where the Riemannian gradient vanishes
\cite[Section 4.1]{absil2008optimization-book}, \cite[def. 4.5]{boumal2023book}.
Similarly to the Euclidean case, classification of these stationary
points can be done by inspecting the signature of the Riemannian Hessian
(i.e. number of positive, negative and zero eigenvalues). A stationary
point is a strict local (or global) minimizer if the Riemannian Hessian
is positive definite \cite[prop. 6.3]{boumal2023book}, while it is
a saddle point if the Riemannian Hessian is indefinite. For a given
Riemannian optimization problem, it is important to analytically determine
which stationary points are local minimizers, and which are saddle
points, since this determines their stability in the context of Riemannian
first order methods~\cite[Section 4.4]{absil2008optimization-book}.
In general, we can expect a first order Riemannian optimization method
to converge to a stationary point (for most initial points), however,
in general, the global minimum is not the only stationary point of
the problem. Fortunately, for most initial points a first order Riemannian
optimization method will converge to a stable point. Thus, it is preferable
that the global minima is the only stable stationary point. 

\subsection{Manifold of Fixed Tucker Rank Tensors \label{subsec:manifold-tucker}}

We wish to extend the concept of RRR to handle higher-dimensional
regression coefficients. For this end, we use the notion of tensor
multilinear rank, and impose constraints on it to obtain higher order
reduced rank constraints. This is accomplished through the Tucker
decomposition, which can be as a seen of generalization of SVD to
higher order tensors \cite{de2000multilinearSVD}. 

For an order-$d$ tensor $\tenX\in\R^{n_{1}\times\cdots\times n_{d}}$
the \emph{multilinear rank }tuple of $\tenX$ is:
\begin{equation}
\rank(\tenX)\coloneqq\left(\rank\left(\unfold{\tenX}1\right),\dots,\rank\left(\unfold{\tenX}d\right)\right)\label{def: multilinear rank}
\end{equation}
Any tensor $\tenX$ of multilinear rank $\rb=(r_{1},\dots,r_{d})$
can be represented in the so-called \emph{Tucker decomposition}
\[
\tenX=\tenC\modeprod 1\matU_{1}\modeprod 2\matU_{2}\cdots\modeprod d\matU_{d}\eqqcolon\left\llbracket \tenC;\matU_{1},\dots,\matU_{d}\right\rrbracket 
\]
with the \emph{core tensor $\tenC\in\R^{r_{1}\times\cdots\times r_{d}}$,
}and the \emph{factor matrices $\matU_{i}\in\R^{n_{i}\times r_{i}}$.
}Without loss of generality, all the $\matU_{i}$ are orthonormal
($\matU_{i}^{\T}\matU_{i}=\matI_{r_{i}}$).\textcolor{blue}{{} }For
real-valued tensors, uniqueness of the the Tucker decomposition is
up to the sign, resp., multiplication with an orthogonal matrix.

For a tuple $\rb$, let $\manM_{\rb}$ denote the set of tensors of
fixed multilinear rank $\rb$. That is, 
\[
\manM_{\rb}\coloneqq\left\{ \left\llbracket \tenC;\matU_{1},\dots,\matU_{d}\right\rrbracket \,:\,\tenC\in\R^{r_{1}\times\cdots\times r_{d}},\;\matU_{i}\in\R^{n_{i}\times r_{i}},\;\matU_{i}^{\T}\matU_{i}=\matI_{r_{i}}\right\} 
\]
The set $\manM_{\rb}$ forms a smooth embedded sub-manifold of $\R^{n_{1}\times\cdots\times n_{d}}$~\cite{Uschmajew2013}
of dimension $\prod_{j=1}^{d}r_{j}+\sum_{i=1}^{d}(r_{i}n_{i}-r_{i}^{2})$,
and by equipping $\R^{n_{1}\times\cdots\times n_{d}}~$with the standard
inner product as the Riemannian metric and inducing the Riemannian
metric on $\manM_{\rb}$ makes it a Riemannian sub-manifold. 

We now detail a few geometric components whose exact definitions are
necessary for analyzing optimization constrained by $\manM_{\rb}$.
Additional components are necessary in order to implement Riemannian
optimization (on which our proposed algorithms are based), though
their exact definitions are not necessary for our exposition, so we
do not detail them.

\paragraph{Tangent Space.}

The \emph{tangent space} of $\manM_{\rb}$ at $\tenX=\left\llbracket \tenC;\matU_{1},\dots,\matU_{d}\right\rrbracket $
can be parameterized~\cite{KochLuibch10tensor-manifold,Kressner2016precond}
\[
T_{\tenX}\manM_{\rb}=\left\{ \tenG\modeprodrange{j=1}d\matU_{j}+\sum_{i=1}^{d}\tenC\modeprod i\matV_{i}\modeprod{j\neq i}\matU_{j}\,:\,\tenG\in\R^{r_{1}\times\cdots\times r_{d}},\matV_{i}\in\R^{n_{i}\times r_{i}},\matV_{i}^{\T}\matU_{i}=0\right\} 
\]
The parameters are $\ten G,\matV_{1},\dots,\matV_{d}$. We use a shorthand
notation for members of the tangent space: 
\begin{equation}
T_{\tenX}\manM_{\rb}\ni\left\{ \ten G;\matV_{1},\dots,\matV_{d}\right\} _{\left\llbracket \tenC;\matU_{1},\dots,\matU_{d}\right\rrbracket }\coloneqq\tenG\modeprodrange{j=1}d\matU_{j}+\sum_{i=1}^{d}\tenC\modeprod i\matV_{i}\modeprod{j\neq i}\matU_{j}\label{eq:tangent-space-G-V}
\end{equation}
For brevity, we mostly omit the subscript $\left\llbracket \tenC;\matU_{1},\dots,\matU_{d}\right\rrbracket $
and write simply $\left\{ \ten G;\matV_{1},\dots,\matV_{d}\right\} $. 

\textcolor{black}{In the next section, we discuss Riemannian stationary
points, which are points in which the Riemannian gradient vanishes.
The Riemannian gradient is an element of the tangent space. Clearly,
if}\textcolor{red}{{} }$\tenG=0$ and $\matV_{j}=0$ for $j=1,\dots,d$
then $\left\{ \ten G;\matV_{1},\dots,\matV_{d}\right\} =0$, but it
is possible for a tangent vector to be zero without each component
being zero? In the following lemma we show that, under mild conditions,
a tangent vector is zero if and only if all it's components are zero.
\begin{lem}
\label{lem:zero-factors-critical}Assume that $r_{j}\leq\prod_{i\neq j}r_{i}$
for all $j$. Suppose there exist a dimension $k$ for which $n_{k}=r_{k}$.
Then, $\left\{ \ten G;\matV_{1},\dots,\matV_{d}\right\} =0$ if and
only if $\tenG=0$ and $\matV_{j}=0$ for $j=1,\dots,d$.
\end{lem}

\begin{proof}
Since $n_{k}=r_{k}$ we immediately have that $\matV_{k}=0$ since
$\matU_{k}$ is an orthogonal matrix, and $\matV_{k}^{\T}\matU_{k}=0$.
Now take the mode $l$ product of $\left\{ \ten G;\matV_{1},\dots,\matV_{d}\right\} $
with $\matU_{l}$, for $l\neq k$:
\begin{align*}
\left\{ \ten G;\matV_{1},\dots,\matV_{d}\right\} \modeprod{l\neq k}\matU_{l} & =\left(\tenG\modeprodrange{j=1}d\matU_{j}+\sum_{i=1}^{d}\tenC\modeprod i\matV_{i}\modeprod{j\neq i}\matU_{j}\right)\modeprod{l\neq k}\matU_{l}\\
 & =\tenG\modeprodrange{j=1}d\matU_{j}\modeprod{l\neq k}\matU_{l}+\sum_{i=1}^{d}\tenC\modeprod i\matV_{i}\modeprod{j\neq i}\matU_{j}\modeprod{l\neq k}\matU_{l}\\
 & =\tenG\modeprod k\matU_{k}+\tenC\times_{k}\matV_{k}+\sum_{i\neq k}^{d}\tenC\times_{i}\matV_{i}\modeprod{j\neq i}\matU_{j}\modeprod i\matU_{i}\modeprod{l\neq k,l\neq i}\matU_{l}\\
 & =\tenG\modeprod k\matU_{k}+\tenC\times_{k}\matV_{k}+\sum_{i\neq k}^{d}\tenC\times_{i}\matU_{i}^{\T}\matV_{i}\modeprod{j\neq i}\matU_{j}\modeprod{l\neq k,l\neq i}\matU_{l}\\
 & =\tenG\modeprod k\matU_{k}+\tenC\times_{k}\matV_{k}\\
 & =0
\end{align*}
Where in the third equality we used the fact that $\matU_{i}^{\T}\matU_{i}=\matI_{r_{i}}$
for all $i$, and in the fifth equality we used the fact that if $\left\{ \ten G;\matV_{1},\dots,\matV_{d}\right\} \in T_{\tenX}\manM_{\rb}$
then $\matU_{i}^{\T}\matV_{i}=0$ for all $i$. Now take the mode
$k$ unfolding to obtain
\[
\matU_{k}\unfold{\tenG}k+\matV_{k}\unfold{\tenC}k=0
\]
However, we already established that $\matV_{k}=0$ so $\matU_{k}\unfold{\tenG}k=0$.
Now, $\matU_{k}$ is a square orthogonal matrix, so it is invertible,
which implies that $\unfold{\tenG}k=0$ so $\tenG=0$. 

We are left with establishing that $\matV_{j}=0$ for $j\neq k$.
For $j\neq k$, take again the mode $l$ product of $\left\{ \ten G;\matV_{1},\dots,\matV_{d}\right\} $
with $\matU_{l}$ for $l\neq j$:
\begin{align*}
\left\{ \ten G;\matV_{1},\dots,\matV_{d}\right\} \modeprod{l\neq j}\matU_{l} & =\sum_{i=1}^{d}\tenC\modeprod i\matV_{i}\modeprod{r\neq i}\matU_{j}\modeprod{l\neq j}\matU_{l}\\
 & =\tenC\modeprod j\matV_{j}+\sum_{i\neq j}^{d}\tenC\times_{i}\matV_{i}\modeprod{r\neq i}\matU_{j}\modeprod i\matU_{i}\modeprod{l\neq j,l\neq i}\matU_{l}\\
 & =\tenC\modeprod j\matV_{j}\\
 & =0
\end{align*}
where in the first equality we used the fact that $\tenG=0$, in the
second equality we used the fact that $\matU_{i}^{\T}\matU_{i}=\matI_{r_{i}}$
for all $i$, and in the third equality we used the fact that if $\left\{ \ten G;\matV_{1},\dots,\matV_{d}\right\} \in T_{\tenX}\manM_{\rb}$
then $\matU_{i}^{\T}\matV_{i}=0$ for all $i$. Taking the mode $j$
unfolding we find that $\matV_{j}\unfold{\tenC}j=0$. Since $\left\llbracket \tenC;\matU_{1},\dots,\matU_{d}\right\rrbracket \in\manM_{\rb}$
the rank of $\unfold{\tenC}j$ is $r_{j}$, and from $r_{j}\leq\prod_{i\neq j}r_{i}$
we have that $\unfold{\tenC}j$ has full row rank. Thus, we must have
$\matV_{j}=0$.
\end{proof}

\paragraph{Projection on the Tangent Space.}

Let $\matP_{T_{\tenW}\manM_{\rb}}:\R^{n_{1}\times\cdots\times n_{d}}\to T_{\tenW}\manM_{\rb}$
be the orthogonal projection on the tangent space of $\manM_{\rb}$
at $\tenW=\left\llbracket \tenC;\matU_{1},\dots,\matU_{d}\right\rrbracket \in\manM_{\rb}$.

According to \cite{Kressner2013completion}, we have the orthogonal
projection of a tensor $\tenA\in\R^{n_{1}\times\cdots\times n_{d}}$
onto $T_{\tenW}\manM_{\rb}$ :
\[
\matP_{T_{\tenW}\manM_{\rb}}\tenA=\left\{ \ten G;\matV_{1},\dots,\matV_{d}\right\} 
\]
 where 
\begin{equation}
\tenG=\tenA\modeprodrange{j=1}d\matU_{j}^{\T}\label{eq:proj-tan-space-g}
\end{equation}
\begin{equation}
\matV_{i}=\matP_{\matU_{i}}^{\perp}\unfold{\left[\tenA\modeprod{j\neq i}\matU_{j}^{\T}\right]}i\unfold{\tenC}i^{\pinv}\label{eq:proj-tan-space-v}
\end{equation}
In the above, $\matP_{\matU_{i}}^{\perp}\coloneqq\matI_{n_{i}}-\matU_{i}\matU_{i}^{\T}$
denotes the orthogonal projection onto the orthogonal complement of
$\sp(\matU_{i})$.

\paragraph{Riemannian Gradient and Riemannian Hessian.}

Since $\manM_{\rb}$ is a Riemannian sub-manifold of the standard
Euclidean space, the Riemannian gradient of an objective function
$F$ is equal to the projection of the Euclidean gradient of $F$
onto the tangent space \cite[equation (3.37)]{absil2008optimization-book}.
We denote the Riemannian gradient of $F$ at $\tenW\in\manM_{\rb}$
as $\grad F(\tenW)$ and the Euclidean gradient as $\nabla F(\tenW)$,
so 

\[
\grad F(\tenW)=\matP_{T_{\tenW}\manM_{\rb}}\nabla F(\tenW)
\]

In our analysis, we need information about the second derivative of
the objective function to capture the curvature of the optimization
landscape. Specifically, the second derivative, also known as the
Hessian, provides valuable insight into how the function changes in
different directions \cite[Def. 5.5.1]{absil2008optimization-book}.
This information can improve convergence rates by refining search
directions and step sizes in second-order optimization methods. In
the context of our work, the Hessian helps classifying the stationary
points of the objective function, as we see in Section \ref{sec:Stationary-Points}.

Let $\tenW=\left\llbracket \tenC;\matU_{1},\dots,\matU_{d}\right\rrbracket \in\manM_{\rb}$
and $\mathcal{Z}=\left\{ \ten G;\matV_{1},\dots,\matV_{d}\right\} \in T_{\tenW}\manM_{\rb}$,
then the Riemannian Hessian of a real-valued function $F$ on $\mathcal{M_{\rb}}$
is a linear mapping $T_{{\cal W}}\manM_{\rb}\to T_{{\cal W}}\manM_{\rb}$
defined by
\[
\Hess F(\tenW)[\mathcal{Z}]=\nabla_{\mathcal{Z}}\grad F(\tenW)
\]
where $\nabla_{\mathcal{Z}}$ denotes the Riemannian connection on
$\mathcal{M}$ \cite[Definition 5.5.1]{absil2008optimization-book}.
We can further write \cite{absil2013extrinsic-riem-hess}:

\begin{align}
\Hess F(\tenW)[\mathcal{Z}]= & \underbrace{\matP_{T_{\tenW}\manM_{\rb}}\nabla^{2}F(\tenW)[\mathcal{Z}]}_{\text{linear term}}+\underbrace{\matP_{T_{\tenW}\manM_{\rb}}D_{\mathcal{Z}}\matP_{T_{\tenW}\manM_{\rb}}(\nabla F(\tenW))}_{\text{curvature term}}\label{eq:riem Hess}
\end{align}
Where the first part, which we refer to as the \emph{linear term}
of the Hessian, is simply a projection of the Euclidean Hessian $\nabla^{2}F(\tenW)$
onto the tangent space. Whereas the second term, which we refer to
as the \emph{curvature term} of the Hessian, depends on both the the
Euclidean gradient $\nabla F(\tenW$) and the direction $\mathcal{Z}$.
The previous two equations are actually general, and not specific
for $\manM_{\rb}$. For $\manM_{\rb}$, Heidel and Schulz \cite{heidel2018riemannian}
have further developed this curvature term (the second term in the
previous equation) for $\manM_{\rb}$: 

\begin{equation}
\matP_{T_{\tenW}\manM_{\rb}}D_{\mathcal{Z}}\matP_{T_{\tenW}\manM_{\rb}}\mathbf{\tenA}=\tilde{\tenC}\stackrel[i=1]{d}{\modeprod i}\mathbf{U}_{i}+\sum_{i=1}^{d}\tenC\modeprod i\tilde{\mathbf{U}}_{i}\modeprod{i\neq j}\matU_{j}\label{eq:riem hess curve term}
\end{equation}
where \textcolor{blue}{}

\begin{equation}
\tilde{\tenC}=\sum_{j=1}^{d}\left(\mathbf{\tenA}\modeprod j\V_{j}^{\T}\modeprod{l\neq j}\matU_{l}^{T}-\tenC\modeprod j\left(\V_{j}^{\T}\left[\mathbf{\tenA}\modeprod{l\neq j}\matU_{l}^{T}\right]_{(j)}\tenC_{(j)}^{+}\right)\right)\label{eq:c_tilde}
\end{equation}

\begin{equation}
\tilde{\mathbf{U}}_{i}=\P_{U_{i}}^{\perp}\left(\left[\mathbf{\tenA}\modeprod{j\neq i}\matU_{j}^{\T}\right]_{(i)}\left(\I-\tenC_{(i)}^{+}\tenC_{(i)}\right)\tenG_{(i)}^{\T}\tenC_{(i)}^{+\T}+\sum_{l\neq i}\left[\mathbf{\tenA}\modeprod l\V_{l}^{\T}\modeprod{l\neq j\neq i}\matU_{j}^{\T}\right]_{(i)}\right)\tenC_{(i)}^{+}\label{eq:u_tilde}
\end{equation}

\paragraph{Higher Order SVD (HOSVD) and Retraction.}

A retraction in the context of $\manM_{\rb}$ is a map $R_{\mathcal{\mathcal{\tenW}}}:T_{\tenW}\manM_{\rb}\longrightarrow\mathcal{\manM_{\rb}}$.
The choice of this map is not unique. Kressner et al. \cite[Prop. 2.3]{Kressner2013completion}
proposed the use of the Higher Order SVD (HOSVD) for the retraction.
HOSVD, originally proposed in \cite{de2000multilinearSVD}, is defined
as follows (we use the presentation in \cite{KoldaBader09}). Let
$\tenA\in\R^{n_{1}\times n_{2}\times...\times n_{d}}$ and $\mathbf{r}=(r_{1},r_{2},...,r_{d})$
such that $r_{i}\leq n_{i}$ for all $i\in[d]$ . Write $\tenB=\HOSVD{\tenA}{\mathbf{r}}=\left\llbracket \tenC;\matU_{1},\dots,\matU_{d}\right\rrbracket $
where the factors $\U_{i}$ are obtained from the leading $r_{i}$
singular vectors of $\ten A_{(i)}$ for all $i\in[d]$ and the core
$\tenC$ is then obtained by $\tenC=\tenA\modeprod 1\matU_{1}^{\T}\modeprod 2\matU_{2}^{\T}\cdots\modeprod d\matU_{d}^{\T}$. 

\paragraph{Tensor Apply for a Tucker Decomposed Tensor.}

When the tensor is given in a Tucker decomposition form it does not
make sense to actually form the full tensor, but rather utilize the
decomposition when applying the tensor to a vector. Suppose that $\tenC\in\R^{r_{1}\times r_{2}\cdots\times r_{d}},\;\matU_{1}\in\R^{k\times r_{1}},\;\matU_{j}\in\R^{m\times r_{j}}\;(j=2,\dots,d)$
and $\matX\in\R^{m\times n}$. (Note that the dimension of $\left\llbracket \tenC;\matU_{1},\dots,\matU_{d}\right\rrbracket $
is $\R^{k\times m\times m\times\cdots\times m}$ ). The following
holds: 
\begin{equation}
\left\llbracket \tenC;\matU_{1},\dots,\matU_{d}\right\rrbracket \matX=\matU_{1}\unfold{\tenC}1\left[\matU_{d}^{\T}\matX\odot\matU_{d-1}^{\T}\matX\odot\cdots\odot\matU_{2}^{\T}\matX\right]\label{eq:ten-apply-tucker}
\end{equation}
The cost of the computation using the formula is 
\[
O\left(mn\sum_{j=2}^{d}r_{j}+n\prod_{j=1}^{d}r_{j}\right)
\]
If we further make the simplifying assumption that $r_{1}=k$, $\matU_{1}=\matI_{k},r_{2}=r_{3}=\cdots=r_{d}=r$,
the cost is 
\[
O\left(mndr+nkr^{d}\right).
\]

\section{Higher-Order Reduced Rank Regression: Problem Statement\label{sec:HORRR problem statement}}

The goal of this section is to state the Higher-Order Reduced Rank
Regression (HORRR) problem, and connect it to polynomial regression
and kernel ridge regression with the polynomial kernel. 

We develop the problem in stages. First, let's consider the input
data. We are given training data $(\x_{1},\y_{1}),\dots,(\x_{n},\y_{n})\in\R^{m}\times\R^{k}$.
The $\x$ vectors are features vectors, while $\y$s are the response,
with $k$ being the number of response variables. We are primarily
interested in the multi-response case, i.e. $k>1$. Note that we denote
the number of features by $m$, since we reserve the letter $d$ for
tensor orders. We arrange data as columns in matrices $\matX\in\R^{m\times n}$
and $\matY\in\R^{k\times n}$, i.e. these are feature-by-sample and
response-by-sample matrices.

\paragraph{Unconstrained Single-Response Tensor Regression.}

First, consider the case where $k=1$. That is responses are scalars,
$y_{1},\dots,y_{n}\in\R$, and are arranged in a row vector $\y\in\R^{1\times n}$.
Consider the problem of finding a $\ten W^{\star}$ such that 
\begin{equation}
\tenW^{\star}\in\argmin_{\tenW\in\R^{1\times m\times\cdots\times m}}\FNormS{\tenW\matX-\y}+\lambda\FNormS{\tenW}\label{eq:prob-single-unconstrainted}
\end{equation}
where the order of ${\cal W}$ is $d+1$. When $d=1$ the tensor $\tenW$
is actually a (row) vector, and we are dealing with (single response)
ordinary least squares (for $\lambda=0$) or ridge regression (for
$\lambda>0$, in which case $\lambda$ serves as a regularization
parameter). For $d>1$ we have multivariate homogeneous polynomial
regression with degree $d$ (see Remark~\ref{rem:apply-and-polynomials}).
Assume the problem has a unique solution (this is always the case
when $\lambda>0$, and holds for $\lambda=0$ under conditions that
will become clear in the following). It is well known that the problem
can be solved efficiently using kernel ridge regression with the polynomial
kernel. To see this, use Eq.~(\ref{eq:apply-matrix-kr}) to rewrite
Eq.~(\ref{eq:prob-single-unconstrainted}): 
\begin{align*}
\R^{1\times m^{d}}\ni\unfold{\tenW^{\star}}1 & =\argmin_{\tenW_{(1)}\in\R^{1\times m^{d}}}\FNormS{\unfold{\tenW}1\matX^{\odot d}-\y}+\lambda\FNormS{\unfold{\tenW}1}\\
 & =\argmin_{\tenW_{(1)}\in\R^{1\times m^{d}}}\FNormS{(\matX^{\odot d})^{\T}\unfold{\tenW^{\T}}1-\y^{\T}}+\lambda\FNormS{\unfold{\tenW^{\T}}1}\\
 & =\left[\left(\matX^{\odot d}(\matX^{\odot d})^{\T}+\lambda\matI_{m^{d}}\right)^{-1}\matX^{\odot d}\y^{\T}\right]^{\T}\\
 & =\left[\matX^{\odot d}\left((\matX^{\odot d})^{\T}\matX^{\odot d}+\lambda\matI_{n}\right)^{-1}\y^{\T}\right]^{\T}\\
 & =\left[\matX^{\odot d}\left(\matK+\lambda\matI_{n}\right)^{-1}\y^{\T}\right]^{\T}
\end{align*}
where $\matK=(\matX^{\odot d})^{\T}\matX^{\odot d}$. The fourth equality
follows from the identity $\left(\matX^{\odot d}(\matX^{\odot d})^{\T}+\lambda\matI_{m^{d}}\right)^{-1}\matX^{\odot d}=\matX^{\odot d}\left((\matX^{\odot d})^{\T}\matX^{\odot d}+\lambda\matI_{n}\right)^{-1}$,
which can be easily proven using the SVD of $\matX^{\odot d}$. We
now observe that 
\[
\matK_{ij}=(\x_{i}^{\odot d})^{\T}\x_{j}^{\odot d}=(\x_{i}^{\T}\x_{j})^{\odot d}=(\x_{i}^{\T}\x_{j})^{d}.
\]
That is, $\matK$ is the Gram matrix of the polynomial kernel. This
allows us to compute $\matK$ and $\valpha\coloneqq\left(\matK+\lambda\matI_{n}\right)^{-1}\y^{\T}$
in $O(n^{3})$ operations. The tensor $\tenW^{\star}$ is not computed
explicitly, but rather the prediction function $f(\x)\coloneqq\tenW^{\star}\x=\unfold{\tenW^{\star}}1\x^{\odot d}=(\x^{\odot d})^{\T}\matX^{\odot d}\valpha$
can be computed in $O(mn)$ by applying the same trick to efficiently
compute $(\x^{\odot d})^{\T}\matX^{\odot d}$. All of the above are
under the assumption of that $\matK+\lambda\matI_{n}$ is invertible,
in which case the solution is indeed unique. This always holds if
$\lambda>0$, and for $\lambda=0$ holds if and only if $\matX^{\odot d}$
has full column rank.

\paragraph{Unconstrained Multi-response Tensor Regression.}

Now consider arbitrary $k$. The problem reads:
\begin{equation}
\tenW^{\star}\in\argmin_{\tenW\in\R^{k\times m\times\cdots\times m}}\FNormS{\tenW\matX-\matY}+\lambda\FNormS{\tenW}\label{eq:prob-multi-unconstrainted}
\end{equation}
Under the uniqueness assumptions discussed above, we have 
\[
\unfold{\tenW^{\star}}1=\left[\matX^{\odot d}\left(\matK+\lambda\matI_{n}\right)^{-1}\matY^{\T}\right]^{\T}
\]
We see that there is no dependence between the $k$ ``rows'' of
$\tenW^{\star}$ ($\tenW_{1,:,\dots,:}^{\star},\dots\tenW_{k,:,\dots,:}^{\star})$,
and this problem is equivalent to solving $k$ independent multivariate
homogeneous polynomial regression problems.

\paragraph{Multilinear Rank Constrained Multi-response Tensor Regression - Higher
Order Reduced Rank Regression.}

To introduce dependence between the predictors, we can impose a multilinear
rank constraint. Given a rank tuple $\rb=(r_{1},r_{2},\dots,r_{d+1})$
, we consider the problem 
\begin{equation}
\tenW_{\rb,\lambda}^{\star}\in\argmin_{\tenW\in\manM_{\rb}}\FNormS{\tenW\matX-\matY}+\lambda\FNormS{\tenW}\label{eq:horrr}
\end{equation}
The constraint $\tenW\in\manM_{\rb}$ ties the different predictors
for the different responses. We further impose that $r_{1}=k$, i.e.
assuming the first factor is of full rank. To understand why this
make sense, inspect Eq.~(\ref{eq:ten-apply-tucker}). We see that
the first factor, $\matU_{1}$, does not effect the interactions of
the different predictors (as embodied by multiplying on the left of
$\matX$), and on the other hand, when $r_{1}<k$ then the result
is needlessly restricted on some subspace of $\R^{k}$, when the usual
case is that $\matY$ has full row rank. In the complexity analyses
it is usually easier to assume that $r_{2}=r_{3}=\cdots=r_{d+1}=r$,
so henceforth we assume this is the case, though the results can be
easily generalized to varying ranks. To simplify statement of the
results, we denote $\manM_{[k,r,d]}\coloneqq\manM_{(k,r,\dots,r)}$
where $r$ is repeated $d$ times (and the tensor is order $d+1$,
representing $k$ polynomials of order $d$).

We call Eq.~(\ref{eq:horrr}) \emph{Higher Order Reduced Rank Regression}.
To see that it is a generalization of Reduced Rank Regression, consider
the case of $d=1$, i.e. $\tenW$ is a tensor of order 2, a matrix.
The set $\manM_{[k,r,1]}$ is the set of $k\times r$ matrices of
rank $\min(k,r)$ and Eq.~(\ref{eq:horrr}) is an instance of Reduced
Rank Regression. Going to $d>1$, we are dealing with homogeneous
polynomials as prediction functions, and dependence is achieved using
a higher order generalization of rank. We chose the multilinear rank,
which corresponds to the Tucker decomposition, and not tensor rank,
which corresponds to the CP decomposition, since the Tucker decomposition
tends to be more stable (finding the best CP decomposition is an ill-posed
problem~\cite{silva-2008-tensor-rank}).

\section{Solving HORRR using Riemannian Optimization}

Let 
\begin{equation}
F_{\lambda}(\tenW)\coloneqq\frac{1}{2}\left(\FNormS{\tenW\matX-\matY}+\lambda\FNormS{\tenW}\right)\label{eq:target function}
\end{equation}
HORRR corresponds to minimizing $F_{\lambda}(\tenW)$ subject to $\tenW\in\manM_{[k,r,d]}$.
The constraint $\tenW\in\manM_{[k,r,d]}$ is a smooth manifold, so
it is natural to consider using Riemannian optimization to solve HORRR,
since it allows us to treat the problem as unconstrained. The goal
of this section is to describe the various components necessary for
using Riemannian optimization to solve HORRR. Throughout the section,
we use the metric on $\manM_{[k,r,d]}$ induced by the inner product
for tensors when viewing $\manM_{[k,r,d]}$ as a Riemannian sub-manifold.

\subsection{Riemannian Gradient}

First, we develop the Euclidean gradient, since Riemannian gradients
are obtained by projecting Euclidean gradients on tangent spaces of
$\manM_{[k,r,d]}$. From Eq.~(\ref{eq:apply-matrix-kr}) we see that
\begin{equation}
\left[\nabla F_{\lambda}(\tenW)\right]_{(1)}=(\tenW\matX-\matY)(\matX^{\odot d})^{\T}+\lambda\unfold{\tenW}1.\label{eq:unfold-euc-grad}
\end{equation}
Note that $\nabla F_{\lambda}(\tenW)$ has the same structure as $\tenW$,
so it is a $k\times m\times\dots\times m$ tensor. 

For $\lambda=0$, a CP decomposition (\cite{KoldaBader09}) of the
Euclidean gradient immediately follows from Eq.~(\ref{eq:unfold-euc-grad}):
\begin{align}
\nabla F_{0}(\tenW) & =\left\llbracket \mathbf{1}_{n};\tenW\matX-\matY,\matX,\dots,\matX\right\rrbracket \label{eq:CP-euc-grad}\\
\nabla F_{\lambda}(\tenW) & =\nabla F_{0}(\tenW)+\lambda\tenW\nonumber 
\end{align}

We now constrain $\tenW\in\manM_{[k,r,d]}$. The canonical way to
compute the Riemannian gradient is via the formula 
\begin{equation}
\grad F_{\lambda}(\tenW)=\matP_{T_{\tenW}\manM_{[k,r,d]}}\nabla F_{\lambda}(\tenW).\label{eq:r-grad-formula}
\end{equation}
However, this is clearly very problematic in our case due to the product
with $\matX^{\odot d}$. Even though this matrix is fixed, and can
be computed once, it is $m^{d}\times n$, so there is a huge storage
cost and computational cost in forming it. It is noted in \cite{NRO22}
that in general it is inefficient to use Eq.~(\ref{eq:r-grad-formula})
sequentially (i.e. compute the Euclidean gradient and then apply the
projection), and it is better to compute the gradient at once. They
suggest using Automatic Differentiation (AD), and propose an algorithm
when the constraints are low-rank tensor trains. Instead of developing
a general AD method for low multilinear rank constraints, we tailor
build an efficient method for computing the Riemannian gradient for
our specific case. 

To work efficiently with the Riemannian gradient, we avoid working
with ambient coordinates, and use the parameterization from Section~\ref{subsec:manifold-tucker},
originally due to \cite{KochLuibch10tensor-manifold}, but we use
the notation offered by \cite{Kressner2013completion}. Let us write
\[
\grad F_{\lambda}(\tenW)=\left\{ \tenG_{\lambda};0_{k\times k},\matV_{2},\dots,\matV_{d+1}\right\} 
\]
where $\tenG_{\lambda}\in\R^{k\times r\times\cdots\times r}$, $\matV_{i}\in\R^{m\times r}$
for $i=2,...,d+1$. Note that there is no subscript $\lambda$ on
the $\matV_{i}$'s. This is because, as we shall see, these matrices
do not depend on $\lambda$. Also note that the first factor in the
gradient $\V_{1}$ is $0_{k\times k}$, which follows from the condition
that it is orthogonal to the orthogonal matrix $\matU_{1}$. Our goal
is to find efficient formulas for $\tenG,\matV_{2},\dots,\matV_{d+1}$.
First, recall that 
\begin{align}
\nabla F_{0}(\tenW) & =\left\llbracket \mathbf{1}_{n};\tenW\matX-\matY,\matX,\dots,\matX\right\rrbracket \nonumber \\
 & =\left\llbracket \tenJ_{n}^{(d+1)};\tenW\matX-\matY,\matX,\dots,\matX\right\rrbracket \label{eq:tucker-euc-grad}\\
 & =\tenJ_{n}^{(d+1)}\modeprod 1\left(\tenW\matX-\matY\right)\modeprodrange{j=2}{d+1}\matX\nonumber 
\end{align}
where $\tenJ_{n}^{(d+1)}\in\R^{n\times\text{(\ensuremath{d+1} times)}\times n}$
has $1$ on the super-diagonal and $0$ elsewhere. Now, using Eq.~(\ref{eq:proj-tan-space-g}):
\begin{align*}
\tenG_{\lambda} & =\nabla F_{\lambda}(\tenW)\modeprodrange{j=1}{d+1}\matU_{j}^{\T}\\
 & =(\nabla F_{0}(\tenW)+\lambda\tenW)\modeprodrange{j=1}{d+1}\matU_{j}^{\T}\\
 & =\left(\tenJ_{n}^{(d+1)}\modeprod 1\left(\tenW\matX-\matY\right)\modeprodrange{j=2}{d+1}\matX\right)\modeprodrange{j=1}{d+1}\matU_{j}^{\T}+\lambda\tenW\modeprodrange{j=1}{d+1}\matU_{j}^{\T}\\
 & =\tenJ_{n}^{(d+1)}\modeprod 1\matU_{1}^{\T}\left(\tenW\matX-\matY\right)\modeprodrange{j=2}{d+1}\matU_{j}^{\T}\matX+\lambda\tenW\modeprodrange{j=1}{d+1}\matU_{j}^{\T}\\
 & =\left\llbracket \mathbf{1}_{n};\matU_{1}^{\T}\left(\tenW\matX-\matY\right),\matU_{2}^{\T}\matX,\dots,\matU_{d+1}^{\T}\matX\right\rrbracket +\lambda\tenC
\end{align*}
Thus, following the equation for unfolding of CP decomposition from
\cite{KoldaBader09}:
\begin{equation}
\unfold{\left[\tenG_{\lambda}\right]}1=\matU_{1}^{\T}(\tenW\matX-\matY)(\matU_{d+1}^{\T}\matX\odot\cdots\odot\matU_{2}^{\T}\matX)^{\T}+\lambda\unfold{\tenC}1\label{eq:core-rgrad-formula}
\end{equation}
Note that the matrix $\matU_{d+1}^{\T}\matX\odot\cdots\odot\matU_{2}^{\T}\matX$
is already computed as part of Eq.~(\ref{eq:ten-apply-tucker}),
so the additional cost when computing $\tenG$ after computing $\tenW\matX-\matY$
is $O(knr^{d})$ which is subsumed by the cost of computing $\tenW\matX-\matY$.

Now consider $\matV_{i},$ for $i=2,...,d+1$. Using Eq.~(\ref{eq:proj-tan-space-v}):
\begin{align*}
\matV_{i} & =\matP_{\matU_{i}}^{\perp}\unfold{\left[\nabla F_{\lambda}(\tenW)\modeprod{j\neq i}\matU_{j}^{\T}\right]}i\unfold{\tenC}i^{\pinv}\\
 & =\matP_{\matU_{i}}^{\perp}\unfold{\left[\left(\nabla F_{0}(\tenW)+\lambda\tenW\right)\modeprod{j\neq i}\matU_{j}^{\T}\right]}i\unfold{\tenC}i^{\pinv}\\
 & =\matP_{\matU_{i}}^{\perp}\unfold{\left[\nabla F_{0}(\tenW)\modeprod{j\neq i}\matU_{j}^{\T}\right]}i\unfold{\tenC}i^{\pinv}\\
 & =\matP_{\matU_{i}}^{\perp}\unfold{\left[\left(\tenJ_{n}^{(d+1)}\modeprod 1\tenW\matX-\matY\modeprodrange{j=2}{d+1}\matX\right)\modeprod{j\neq i}\matU_{j}^{\T}\right]}i\unfold{\tenC}i^{\pinv}\\
 & =\matP_{\matU_{i}}^{\perp}\unfold{\left[\tenJ_{n}^{(d+1)}\modeprod 1\matU_{1}^{\T}\left(\tenW\matX-\matY\right)\modeprod i\X\modeprodrange{j\neq i}{}\matU_{j}^{\T}\matX\right]}i\unfold{\tenC}i^{\pinv}\\
 & =\matP_{\matU_{i}}^{\perp}\unfold{\left[\left\llbracket \mathbf{1}_{n};\matU_{1}^{\T}\left(\tenW\matX-\matY\right),\matU_{2}^{\T}\matX,\dots,\matU_{i-1}^{\T}\matX,\matX,\matU_{i+1}^{\T}\matX,\dots,\matU_{d+1}^{\T}\matX\right\rrbracket \right]}i\unfold{\tenC}i^{\pinv}
\end{align*}
where the third equality follows from the fact that $\matP_{\matU_{i}}^{\perp}\unfold{\tenW}i=0$.
Unfolding of the CP decomposition, we finally have 

\begin{equation}
\matV_{i}=\matP_{\matU_{i}}^{\perp}\matX\left(\matU_{d+1}^{\T}\matX\odot\cdots\odot\matU_{i+1}^{\T}\matX\odot\matU_{i-1}^{\T}\matX\odot\cdots\odot\matU_{1}^{\T}\matX\odot\matU_{1}^{\T}\left(\tenW\matX-\matY\right)\right)^{\T}\unfold{\tenC}i^{\pinv}\label{eq:factor-rgrad-formula}
\end{equation}

Assuming that for all $j=2,...,d+1$ that $\matU_{j}^{\T}\matX$ and
$\unfold{\tenC}j^{\pinv}$ have been precomputed (the former we already
need for the apply operation, and the latter is actually computed
in MANOPT\footnote{MANOPT\textcolor{red}{{} }is a well-known MATLAB toolbox for optimization
on manifolds and linear spaces (\href{https://www.manopt.org}{https://www.manopt.org}) }) the cost of computing a single $\matV_{i}$ is $O(k^{2}r^{d-1}(n+r))$.
The cost for all $\matV_{i}$s is $O(dk^{2}r^{d-1}(n+r))$.

\subsection{Riemannian Hessian}

A general formula for the Hessian at $\manM_{\rb}$ was introduced
in section \ref{subsec:manifold-tucker}, and we further apply it
to our HORRR problem, i.e. for $\manM_{[k,r,d]}$. Let $\tenW=\left\llbracket \tenC;\matU_{1},\dots,\matU_{d+1}\right\rrbracket \in\manM_{[k,r,d]}$
and $\tenZ=\left\{ \ten G;0,\matV_{2},\dots,\matV_{d+1}\right\} \in T_{\tenW}\manM_{[k,r,d]}$,
recall the general formula from Eq. (\ref{eq:riem Hess}):

\[
\Hess F_{\lambda}(\tenW)[\tenZ]=\underbrace{\matP_{T_{\tenW}\manM_{[k,r,d]}}\nabla^{2}F_{\lambda}(\tenW)[\tenZ]}_{\text{linear term}}+\underbrace{\tilde{\tenC}\modeprodrange{i=1}{d+1}\matU_{i}+\sum_{i=1}^{d+1}\tenC\modeprod i\tilde{\mathbf{U}}_{i}\modeprod{i\neq j}\matU_{j}}_{\text{curvature term}}
\]

\begin{doublespace}
\noindent The Euclidean Hessian (as an operator on the tangent spaces)
in ambient coordinates is:\textcolor{red}{{} }

\noindent
\begin{align}
\nabla^{2}F_{\lambda}(\tenW)[\tenZ] & =\left\llbracket \mathbf{1}_{n};\tenZ\matX,\matX,...,\matX\right\rrbracket +\lambda\tenZ\label{eq:CP-euc-Hess}
\end{align}

\noindent For the linear term, we need to project it on the tangent
space. Writing $\matP_{T_{\tenW}\manM_{[k,r,d]}}\nabla^{2}F_{\lambda}(\tenW)[\tenZ]=\left\{ \ten G_{\lambda}^{\text{Euc}};0_{k\times k},\matV_{2}^{\text{Euc}},\dots,\matV_{d+1}^{\text{Euc}}\right\} $,
we have
\end{doublespace}

\begin{align*}
\ten G_{\lambda}^{\text{Euc}} & =\nabla^{2}F_{\lambda}(\tenW)[\tenZ]\modeprodrange{j=1}{d+1}\matU_{j}^{\T}=\left\llbracket \mathbf{1}_{n};\matU_{1}^{\T}\tenZ\matX,\matU_{2}^{\T}\matX,\dots,\matU_{d+1}^{\T}\matX\right\rrbracket +\lambda\tenZ\modeprodrange{j=1}{d+1}\matU_{j}^{\T}\\
\matV_{i}^{\text{Euc}} & =\matP_{\matU_{i}}^{\perp}\matX\left(\matU_{d+1}^{\T}\matX\odot\cdots\odot\matU_{i+1}^{\T}\matX\odot\matU_{i-1}^{\T}\matX\odot\cdots\odot\matU_{2}^{\T}\matX\odot\matU_{1}^{\T}\tenZ\matX\right)^{\T}\unfold{\tenC}i^{\pinv}
\end{align*}
where the calculations for the $\V_{i}^{\mathrm{Euc}}$'s are similar
to those in the previous sub-section. Again, the $\mat V_{i}^{\text{Euc}}$
components do not depend on $\lambda$.

Next, we apply the Euclidean gradient formulas Eqs. (\ref{eq:CP-euc-grad})
and (\ref{eq:tucker-euc-grad}) for the curvature term. Full calculations
are given in Appendix \ref{appendix: hessian}. An interesting observation
is that the curvature term does not depend on the the regularization
parameter. 

\subsection{Finding Stationary Points using Riemannian Optimization}

To analyze the convergence of our proposed method, we can appeal to
the general convergence theory for Riemannian optimization \cite[Section 4.1]{absil2008optimization-book}.
In particular, \cite[Theorem 4.3.1]{absil2008optimization-book} implies
the following result.
\begin{prop}
\label{prop: convergence}Let $\left(\tenW_{k}\right)_{k\in\N}$ be
an infinite sequence of iterates generated by HORRR Algorithm. Then,
every accumulation point $\tenW_{*}\in\manM_{\rb}$ of $\tenW_{k}$
is a critical point of the cost function $F_{\lambda}$ and hence
satisfies $gradF_{\lambda}(\tenW_{k})=0$. 
\end{prop}

Unfortunately, the set $\manM_{\mathbf{r}}$ is not closed - a sequence
of tensors with multilinear rank $\mathbf{r}=(r_{1},...,r_{d+1})$
may approach a tensor for which the $i$th rank is less than $r_{i}$.
To address this issue, we use the same analysis technique as in \cite{vandereycken2013low,Kressner2013completion}.
The basic idea is to add a regularization term the ensures that all
iterates stay within a compact set. This ensure convergence to a stationary
point for the regularized problem. Next, we drive the regularization
term to zero, and argue that if the iterates do not approach the boundary
of $\manM_{\rb}$ then we must converge to a stationary point. The
steps are almost identical to the ones in \cite{Kressner2013completion}.
For completeness, we include them in Appendix \ref{appendix: convergence}\textcolor{blue}{.}

\section{Stationary Points\label{sec:Stationary-Points}}

Riemannian optimization methods are fit for finding Riemannian stationary
points. Typically, the desired solution (i.e. minimizer of the cost
function) is not the only stationary point of the problem. If it is,
however, the only \emph{stable }stationary point, then Riemannian
optimization will almost surely find that point, and we have an effective
way of solving the problem (assuming the Riemannian optimization can
be implemented efficiently for the problem)~\cite{absil2008optimization-book}.
The goal of this section is to characterize the stationary points
of HORRR.

The following is a necessary condition for a point to be a stationary
point. 
\begin{lem}
Let $\tenW=\left\llbracket \tenC;\matU_{1}\dots,\matU_{d+1}\right\rrbracket $
be a stationary point of $F_{\lambda}$ subject to $\tenW\in\manM_{[k,r,d]}$,
and denote the residual $\matR\coloneqq\tenW\matX-\matY$. Then, 
\begin{equation}
\matR(\matU_{d+1}^{\T}\matX\odot\cdots\odot\matU_{2}^{\T}\matX)^{\T}=-\lambda\matU_{1}\unfold{\tenC}1\label{eq:residual-condition}
\end{equation}
\end{lem}

\begin{proof}
Follows immediately from Eq.~(\ref{eq:core-rgrad-formula}) and Lemma~\ref{lem:zero-factors-critical}.
\end{proof}
It is instructive to consider the unconstrained case, for which the
residual upholds the equation
\[
\matR(\matX^{\odot d})^{\T}=-\lambda\unfold{\tenW}1
\]
which is a multilinear generalization of the classical result in OLS
that the residual is orthogonal to the space spanned by the data (for
$\lambda=0$). In contrast, the stationary points of $F_{\lambda}$
subject to $\tenW\in\manM_{[k,r,d]}$ when $\lambda=0$ are orthogonal
to transformed ($\matU_{1}^{\T}\matX,\dots,\matU_{d+1}^{\T}\matX$),
where the transformation is learned for data. 

The above necessary condition is not a sufficient one. The following
fully characterizes the stationary points. 
\begin{lem}
\label{lem:condition for stationary points}A point $\tenW=\left\llbracket \tenC;\matU_{1},\dots,\matU_{d+1}\right\rrbracket \in\manM_{[k,r,d]}$
is a stationary point of $F_{\lambda}$ subject to $\tenW\in\manM_{[k,r,d]}$
if and only if
\begin{align*}
\lambda\matU_{i} & =\matX(\matU_{d+1}^{\T}\matX\odot\cdots\matU_{i+1}^{\T}\matX\odot\matU_{i-1}^{\T}\matX\odot\cdots\odot\matU_{2}^{\T}\matX\odot\matU_{1}^{\T}(\tenW\matX-\matY))^{\T}\unfold{\tenC}i^{\pinv}
\end{align*}
 for $i=2,...,d$, and 
\begin{align*}
\lambda\unfold{\tenC}1 & =-\matU_{1}^{\T}(\tenW\matX-\matY)(\matU_{d+1}^{\T}\matX\odot\cdots\odot\matU_{2}^{\T}\matX)^{\T}
\end{align*}
\end{lem}

\begin{proof}
Due to Lemma~\ref{lem:zero-factors-critical} a point is a stationary
point if and only if the gradient $\left\{ \tenG_{\lambda};0_{k\times k},\matV_{2},\dots,\matV_{d+1}\right\} =\left\{ 0_{[k,r,d]};0_{k\times k},0_{m\times r},\dots,0_{m\times r}\right\} $.
For $\matV_{i}$, $i=2,...,d$, we have:

\begin{align*}
\matV_{i} & =\matP_{\matU_{i}}^{\perp}\unfold{\left[\tenJ_{n}^{(d+1)}\modeprod 1\matU_{1}^{\T}(\tenW\matX-\matY)\modeprod i\matX\modeprodrange{j=2,j\neq i}{d+1}\matU_{j}^{\T}\matX\right]}i\unfold{\tenC}i^{\pinv}\\
 & =(\matI_{m}-\matU_{i}\matU_{i}^{\T})\unfold{\left[\tenJ_{n}^{(d+1)}\modeprod 1\matU_{1}^{\T}(\tenW\matX-\matY)\modeprod i\matX\modeprodrange{j=2,j\neq i}{d+1}\matU_{j}^{\T}\matX\right]}i\unfold{\tenC}i^{\pinv}
\end{align*}
Now, 
\begin{align*}
\matU_{i}^{\T}\unfold{\left[\tenJ_{n}^{(d+1)}\modeprod 1\matU_{1}^{\T}(\tenW\matX-\matY)\modeprod i\matX\modeprodrange{j=2,j\neq i}{d+1}\matU_{j}^{\T}\matX\right]}i & =\\
\unfold{\left[\tenJ_{n}^{(d+1)}\modeprod 1\matU_{1}^{\T}(\tenW\matX-\matY)\modeprod i\matU_{i}^{\T}\matX\modeprodrange{j=2,j\neq i}{d+1}\matU_{j}^{\T}\matX\right]}i & =\\
\unfold{\left[\tenJ_{n}^{(d+1)}\modeprod 1\matU_{1}^{\T}(\tenW\matX-\matY)\modeprodrange{j=2}{d+1}\matU_{j}^{\T}\matX\right]}i & =\\
\unfold{\left[\tenG_{\lambda}-\lambda\tenC\right]}i & =-\lambda\unfold{\tenC}i
\end{align*}
where we used $\tenG_{\lambda}=0$ . So 
\begin{align*}
\matV_{i} & =\unfold{\left[\tenJ_{n}^{(d+1)}\modeprod 1\matU_{1}^{\T}(\tenW\matX-\matY)\modeprod i\matX\modeprodrange{j=2,j\neq i}{d+1}\matU_{j}^{\T}\matX\right]}i\unfold{\tenC}i^{\pinv}-\lambda\matU_{i}\unfold{\tenC}i\unfold{\tenC}i^{\pinv}\\
 & =\matX(\matU_{d+1}^{\T}\matX\odot\cdots\matU_{i+1}^{\T}\matX\odot\matU_{i-1}^{\T}\matX\odot\cdots\odot\matU_{2}^{\T}\matX\odot\matU_{1}^{\T}(\tenW\matX-\matY))^{\T}\unfold{\tenC}i^{\pinv}-\lambda\matU_{i}
\end{align*}
where the second equality follows from $\unfold{\tenC}i\unfold{\tenC}i^{\pinv}=\unfold{\tenC}i\unfold{\tenC}i^{\T}\left(\unfold{\tenC}i\unfold{\tenC}i^{\T}\right)^{-1}=\I_{r}$.
If $\matV_{i}=0$ then 

\begin{align*}
0 & =\unfold{\left[\tenJ_{n}^{(d+1)}\modeprod 1\matU_{1}^{\T}(\tenW\matX-\matY)\modeprod i\matX\modeprodrange{j=2,j\neq i}{d+1}\matU_{j}^{\T}\matX\right]}i\unfold{\tenC}i^{\pinv}-\lambda\matU_{i}
\end{align*}
Finally we have
\[
\matX(\matU_{d+1}^{\T}\matX\odot\cdots\matU_{i+1}^{\T}\matX\odot\matU_{i-1}^{\T}\matX\odot\cdots\odot\matU_{2}^{\T}\matX\odot\matU_{1}^{\T}(\tenW\matX-\matY))^{\T}\unfold{\tenC}i^{\pinv}=\lambda\matU_{i}
\]
Proving the other side, from Eq.~(\ref{eq:factor-rgrad-formula})
we have 
\[
\matV_{i}=\matP_{\matU_{i}}^{\perp}\matX(\matU_{d+1}^{\T}\matX\odot\cdots\matU_{i+1}^{\T}\matX\odot\matU_{i-1}^{\T}\matX\odot\cdots\odot\matU_{2}^{\T}\matX\odot\matU_{1}^{\T}(\tenW\matX-\matY))^{\T}\unfold{\tenC}i^{\pinv}=\lambda\matP_{\matU_{i}}^{\perp}\matU_{i}=0
\]
From Eq.~(\ref{eq:core-rgrad-formula}) we obtain
\begin{eqnarray*}
\unfold{\left[\tenG_{\lambda}\right]}1 & = & \matU_{1}^{\T}(\tenW\matX-\matY)(\matU_{d+1}^{\T}\matX\odot\cdots\odot\matU_{2}^{\T}\matX)^{\T}+\lambda\unfold{\tenC}1\\
 & = & \matU_{1}^{\T}(\tenW\matX-\matY)(\matU_{d+1}^{\T}\matX\odot\cdots\odot\matU_{2}^{\T}\matX)^{\T}-\matU_{1}^{\T}(\tenW\matX-\matY)(\matU_{d+1}^{\T}\matX\odot\cdots\odot\matU_{2}^{\T}\matX)^{\T}\\
 & = & 0
\end{eqnarray*}
\end{proof}

\subsection{Reduced Rank Regression: $d=1$ (matrix case)}

We now look at the case where $d=1$, i.e. elements of $\manM_{[k,r,1]}$
are matrices. In this case, HORRR coincides with Regularized RRR,
which is known to have a closed-form solution under simple conditions
(see Subsection~\ref{subsec:RRR} and Eq.~(\ref{eq:RRR-classic-solution})).
For the rest of this subsection, we assume the Regularized RRR solution
for each $r$ is unique, i.e. all the singular values of $\mathbf{\hat{Y}}\hat{\X}_{\lambda}^{+}\hat{\X}_{\lambda}$
are distinct (see Section \ref{subsec:RRR}), and for $\lambda=0$
we also assume that $\matX$ has full row rank. We show that the closed-form
solution is the only local minimum of $F_{\lambda}$ subject to $\tenW\in\manM_{[k,r,1]}$,
which justifies using Riemannian optimization for Regularized RRR.
In fact, all stationary points correspond to an eigenpair of a certain
matrix pencil (also a singular triplet of another matrix), and the
Regularized RRR solution corresponds to the eigenpair with minimal
eigenvalue (or singular triplet with maximal singular value).

Since we are dealing with matrices, we use matrix notation in this
section, i.e. the coefficient tensor $\ten W$ is now a matrix $\matW$.
Since matrices have only a single rank and not a multilinear rank,
per our assumption that $r\leq k$, the multilinear rank $(k,r)$
is actually the matrix rank $r$, hence the constrains manifold $\manM_{[k,r,1]}$
is equal to the manifold of $k\times m$ matrices of rank $r$. An
important difference from the case $d>1$ is that Lemma~\ref{lem:zero-factors-critical}
is not relevant for this case.

Consider the case that $\lambda=0$. In this case, for each $r,$
the solution $\matW_{r,0}^{\star}$ to HORRR is unique and is equal
to the RRR solution: 
\begin{equation}
\matW_{r}^{(\textrm{RRR})}=\left[\matY\matX^{\pinv}\matX\right]_{r}\matX^{\pinv}\label{eq:w_rrr_formula}
\end{equation}
$\matW_{r}^{(\textrm{RRR})}$ is a solution (optimal point) for $\min_{\matW\in\manM_{[k,r,1]}}F_{0}(\matW$).
Since it is an optimal point, it is also a stationary point. Recall
$\grad F_{\lambda}(\tenW)=\matP_{\manM_{[k,r,1]}}\left(\nabla F_{\lambda}(\tenW)\right)$.
The orthogonal projection of some $\matA\in\R^{k\times m}$, in ambient
coordinates, is given by:
\begin{eqnarray*}
\matP_{T_{\tenW}\manM_{[k,r,1]}}(\matA) & = & \tenG\modeprod 1\matU\modeprod 2\matV+\Sigma\modeprod 1\matV_{1}\modeprod 2\matV+\Sigma\modeprod 1\matU\modeprod 2\matV_{2}\\
 & = & \matU\tenG\matV^{\T}+\matV_{1}\Sigma\matV^{\T}+\matU\Sigma\matV_{2}^{\T}
\end{eqnarray*}
for $\matW=\matU\Sigma\matV^{\T}$ and $\ten G,\matV_{1},\matV_{2}$
are given by Eqs.~(\ref{eq:proj-tan-space-g})~and~(\ref{eq:proj-tan-space-v}).
So, 
\[
\matP_{T_{\tenW}\manM_{[k,r,1]}}(\matA)=\matA\matW^{\pinv}\matW+\matW\matW^{\pinv}\matA-\matW\matW^{\pinv}\matA\matW^{\pinv}\matW
\]

We have $\nabla F_{0}(\matW)=(\matW\matX-\matY)\matX^{\T}$, hence
the Riemannian gradient is given in ambient coordinates by 
\begin{equation}
\grad F_{0}(\matW)=\matW\matX\matX^{\T}-\matW\matW^{\pinv}\matY\matX^{\T}-(\matI-\matW\matW^{\pinv})\matY\matX^{\T}\matW^{\pinv}\matW\label{eq:grad d=00003D1 lambda=00003D0}
\end{equation}
A stationary point $\matW\in\manM_{[k,r,1]}$ is such that $\grad F_{0}(\matW)=0$.
However, the equation 
\begin{equation}
\matW\matX\matX^{\T}-\matW\matW^{\pinv}\matY\matX^{\T}-(\matI-\matW\matW^{\pinv})\matY\matX^{\T}\matW^{\pinv}\matW=0\label{eq:rgrad-zero-rrr}
\end{equation}
 itself does not have a unique solution, since $\matW_{r}^{(\textrm{RRR})}$
upholds it for \emph{any} $r$. Furthermore, we will see that there
are multiple stationary points for every $r$. First, let's state
a simple observation.
\begin{prop}
\label{claim:critical d=00003D1}$\matW\matX\matX^{\T}-\matW\matW^{\pinv}\matY\matX^{\T}-(\matI-\matW\matW^{\pinv})\matY\matX^{\T}\matW^{\pinv}\matW=0$
if and only if 
\begin{equation}
\matW\matX\matX^{\T}-\matW\matW^{\pinv}\matY\matX^{\T}=0\label{eq:criticial-d1-eq1}
\end{equation}
and
\begin{equation}
(\matI-\matW\matW^{\pinv})\matY\matX^{\T}\matW^{\pinv}\matW=0\label{eq:criticial-d1-eq2}
\end{equation}
\end{prop}

\begin{proof}
Note that $\matI-\matW\matW^{\pinv}$ is a projector on the space
orthogonal to the column space of $\matW$. However, $\matW\matX\matX^{\T}-\matW\matW^{\pinv}\matY\matX^{\T}$
is in that column space. Thus, $\matW\matX\matX^{\T}-\matW\matW^{\pinv}\matY\matX^{\T}$
is orthogonal to $(\matI-\matW\matW^{\pinv})\matY\matX^{\T}\matW^{\pinv}\matW$.
As a consequence, their sum can be zero if and only if both are zero.
\end{proof}

\paragraph{The case of $r=1$. }

We first consider the case that $r=1$, i.e the constraint is $\tenW\in\manM_{[k,1,1]}$.
Recall that we assume that $\matX$ has full row rank, which makes
sense otherwise $\matY$is not in the range of the features.  Suppose
the rank one matrix $\matW$ is a solution to Eq.~(\ref{eq:rgrad-zero-rrr}),
which implies that both Eqs.~(\ref{eq:criticial-d1-eq1}) and (\ref{eq:criticial-d1-eq2})
hold. Since $\matW$ has rank one we can write it as
\[
\matW=\sigma\u\v^{\T}
\]
where $\sigma>0$ and both $\u$ and $\v$ have unit norm. We have
$\matW\matW^{\pinv}=\u\u^{\T}$ and $\matW^{\pinv}\matW=\v\v^{\T}$.
Eq.~(\ref{eq:criticial-d1-eq2}) now reads:
\[
(\matI-\u\u^{\T})\matY\matX^{\T}\v\v^{\T}=0
\]
Multiplying on the right by $\v$, and rearranging, we find that 
\begin{equation}
\matY\matX^{\T}\v=(\u^{\T}\matY\matX^{\T}\v)\cdot\u\label{eq:v-to-u}
\end{equation}
Denote $\z=\matY\matX^{\T}\v$. The equation reads $\z=(\u^{\T}\matY\matX^{\T}\v)\cdot\u$.
Since $\u^{\T}\matY\matX^{\T}\v$ is a scalar, Eq.~(\ref{eq:v-to-u})
shows that $\u$ and $\z$ are the same up to scaling by a scalar.
Thus, we can write $\u=\alpha\z$. Recalling that $\u$ has unit norm
we find that $\alpha=1/\TNorm{\z}$. To conclude, we find that given
$\v$, the vector $\u$ is completely determined by the equation:
\[
\u=\ITNorm{\matY\matX^{\T}\v}\matY\matX^{\T}\v
\]

Next, consider Eq.~(\ref{eq:criticial-d1-eq1}). After substituting
$\matW=\sigma\u\v^{\T}$ we have 
\[
\sigma\u\v^{\T}\matX\matX^{\T}=\u\u^{\T}\matY\matX^{\T}
\]
Multiply on the left by $\u^{\T}$ to get 
\[
\sigma\v^{\T}\matX\matX^{\T}=\u^{\T}\matY\matX^{\T}
\]
To get rid of $\u$ we substitute the equation for it from $\v$ to
get
\[
\sigma\v^{\T}\matX\matX^{\T}=\ITNorm{\matY\matX^{\T}\v}\v^{\T}\matX\matY^{\T}\matY\matX^{\T}
\]
Take transpose, and move $\sigma$ to the right side: 
\[
\matX\matX^{\T}\v=\sigma^{-1}\ITNorm{\matY\matX^{\T}\v}\matX\matY^{\T}\matY\matX^{\T}\v
\]
We see that $\v$ is a unit-norm generalized eigenvector of the matrix
pencil $(\matX\matX^{\T},\matX\matY^{\T}\matY\matX^{\T})$\footnote{A pair $(\lambda,\x)$ is a \emph{finite} generalized eigenpair of
$(\matS,\mat T)$ if $\lambda$ is finite, $\x$ is non-zero and $\matS\x=\lambda\mat T\x$
and $\mat T\x\neq0$. A pair $(\infty,\x)$ is a generalized eigenpair
if $\mat T\x=0$ but $\matS\x\neq0$. Note that in our case $\matX\matY^{\T}\matY\matX^{\T}$
is singular while $\matX\matX^{\T}$ is not, so there are non-finite
generalized eigenpairs.}.

Going backward, any \emph{finite} unit-norm generalized eigenpair
$(\gamma,\v)$ of the pencil $(\matX\matX^{\T},\matX\matY^{\T}\matY\matX^{\T})$
can be used to generate a stationary point. The point is $\matW_{\gamma,\v}=\sigma\u\v^{\T}$
where 
\begin{equation}
\sigma=\gamma^{-1}\ITNorm{\matY\matX^{\T}\v}>0\label{eq:d=00003D1,r=00003D1 sigma formula}
\end{equation}
(where we used the fact that $\matX$ has full row rank hence $\gamma>0$)
and 
\[
\u=\ITNorm{\matY\matX^{\T}\v}\matY\matX^{\T}\v
\]
obtaining 
\[
\matW_{\gamma,\v}=\frac{1}{\gamma\TNormS{\matY\matX^{\T}\v}}\matY\matX^{\T}\v\v^{\T}
\]

We now relate the generalized eigenpairs of the matrix pencil $(\matX\matX^{\T},\matX\matY^{\T}\matY\matX^{\T})$
to singular triplets of $\matY\matX^{+}\matX$, which is exactly the
matrix for which we take the best rank $r$ approximation in the formula
for $\matW_{r}^{(\textrm{RRR})}$ (Eq. (\ref{eq:w_rrr_formula})).
Denote $\matA\coloneqq\matY\matX^{+}\matX$. Let $(\gamma,\v)$ be
a finite eigenpair of the matrix pencil $(\matX\matX^{\T},\matX\matY^{\T}\matY\matX^{\T})$.
We have 
\begin{eqnarray*}
\matA^{\T}\matA\matX^{\T}\v & = & \matX^{\T}(\matX\matX^{\T})^{-1}(\matX\matY^{\T}\matY\matX^{\T})\v\\
 & = & \matX^{\T}(\matX\matX^{\T})^{-1}\left(\gamma^{-1}\matX\matX^{\T}\right)\v\\
 & = & \gamma^{-1}\matX^{\T}\v
\end{eqnarray*}
where in the first equality we used the fact that $\matX^{\pinv}\matX\matX^{\T}=\matX^{\T}$
(due to $\matX$ having full row rank). We see that $(\gamma^{-1},\matX^{\T}\v$)
is an eigenpair of $\matA^{\T}\matA$. Therefore, $\sigma_{\matA,\gamma}\coloneqq1/\sqrt{\gamma}$
is a singular value of $\matA$, and $\tilde{\v}=\ITNorm{\matX^{\T}\v}\matX^{\T}\v$
is the corresponding right singular vector. Similarly, we have 

\[
\matA\matA^{\T}\matY\matX^{\T}\v=\gamma^{-1}\matY\matX^{\T}\v
\]
and $\tilde{\u}=\ITNorm{\matY\matX^{\T}\v}\matY\matX^{\T}\v$ is a
left singular vector of $\matA$ corresponding to the singular value
$1/\sqrt{\gamma}$ . Note that 
\begin{eqnarray*}
\TNormS{\matY\matX^{\T}\v} & = & \v^{\T}\matX\matY^{\T}\matY\matX^{\T}\v\\
 & = & \v^{\T}\left(\gamma^{-1}\matX\matX^{\T}\right)\v\\
 & = & \gamma^{-1}\TNormS{\matX^{\T}\v}
\end{eqnarray*}
Thus, 
\begin{eqnarray*}
\sigma_{\matA,\gamma}\tilde{\u}\tilde{\v}^{\T}\matX^{\pinv} & = & \frac{\sigma_{\matA,\gamma}\matY\matX^{\T}\v\v^{\T}\matX\matX^{\pinv}}{\TNorm{\matY\matX^{\T}\v}\TNorm{\matX^{\T}\v}}\\
 & = & \frac{\sigma_{\matA,\gamma}\matY\matX^{\T}\v\v^{\T}}{\sqrt{\gamma}\TNormS{\matY\matX^{\T}\v}}\\
 & = & \frac{\matY\matX^{\T}\v\v^{\T}}{\gamma\TNormS{\matY\matX^{\T}\v}}=\matW_{\gamma,\v}
\end{eqnarray*}
The same process can we worked backwards. So, for every singular triplet
$(\tilde{\sigma},\tilde{\u},\tilde{\v})$ of $\matA$ we have 
\[
\tilde{\sigma}\tilde{\u}\tilde{\v}^{\T}\matX^{\pinv}=\matW_{1/\tilde{\sigma}^{2},\v}
\]
 where $\v$ is the unique $\v$ such that $\tilde{\v}=\ITNorm{\matX^{\T}\v}\matX^{\T}\v$.
This leads to the following lemma.
\begin{lem}
\label{lem:RRR-min-eigenvalue}The RRR solution for the case $r=1$
and full row rank $\matX$, is equal to $\matW_{\gamma,\v}$ where
$(\gamma,\v)$ is the eigenpair of the pencil $(\matX\matX^{\T},\matX\matY^{\T}\matY\matX^{\T})$
with minimum eigenvalue.
\end{lem}

\begin{proof}
We have $\matW_{1}^{(\textrm{RRR})}=\left[\matA\right]_{1}\matX^{\pinv}=\sigma_{\max}\tilde{\u}\tilde{\v}^{\T}\matX^{\pinv}$where
$\tilde{\u}$ and $\tilde{\v}$ are the left and right singular values
of $\matA$ corresponding to the largest singular value of $\matA$:
$\sigma_{\max}(\matA)$. We see that $\matW_{1}^{(\textrm{RRR})}=\matW_{1/\sigma_{\max}^{2},\v}$.
The eigenvalue $1/\sigma_{\max}^{2}$ must be the minimal eigenvalue,
since $\sigma_{\max}$ is the maximum singular value.
\end{proof}

\paragraph{Classification of the stationary points for $r=1$.}

Every singular triplet of $\matY\matX^{+}\matX$ corresponds to a
stationary point, but only the singular triplet corresponding to the
maximum singular value is related to the RRR solution. Obviously,
the RRR solution is a global minimum of $F_{0}(\matW)$ subject to
$\matW\in\manM_{[k,1,1]}$. We now show that this point is also the
only \emph{local} minimum. This is important for justifying the use
of Riemannian optimization, which almost surely converges to a stable
stationary point, and those are only local minimizers (for minimization
points). To show that a point is either a local minimizer or not,
we inspect the spectrum of the Riemannian Hessian of the stationary
point of interest, and establish whether it is positive semi-definite
(in which case the point is a local minimum), or not (in which case
it is a local maximum or a saddle point). To that end, we inspect
the Rayleigh quotient, relying on the Courant-Fischer theorem. 

In \cite{absil2013extrinsic-riem-hess}, the following formula for
the Riemannian Hessian operating on some direction $\matZ\in T_{\matW}\manM_{[k,1,1]}$
is given by: 
\[
\Hess F_{0}(\matW)[\matZ]=\matP_{T_{\tenW}\manM_{[k,1,1]}}\left(\nabla^{2}F_{0}(\matW)\right)[\matZ]+\matP_{T_{\tenW}\manM_{[k,1,1]}}\matD_{Z}\matP\nabla F_{0}(\matW)
\]
where the second term is further shown, for some $\matA\in T_{\matW}^{\perp}\manM_{[k,1,1]}$,
to be equal to
\begin{equation}
\matP_{T_{\matW}\manM_{[k,1,1]}}\matD_{Z}\matP\matA=\matA\matZ^{\T}(\matW^{+})^{\T}+(\matW^{+})^{\T}\matZ^{\T}\matA\label{eq:weingraten map d1 r1 pinvW}
\end{equation}
 For some $\matW=\sigma\u\v^{\T}\neq0$, we have $\matW^{+^{\T}}=\sigma^{-1}\u\v^{\T}=\sigma^{-2}\matW$,
thus the previous equation becomes
\begin{equation}
\matP_{T_{\matW}\manM_{[k,1,1]}}\matD_{Z}\matP\matA=\sigma^{-2}\left(\matA\matZ^{\T}\matW+\matW\matZ^{\T}\matA\right)\label{eq:weingraten map d1 r1 W}
\end{equation}
At a stationary point $\matW$, the Riemannian gradient vanishes.
Since the Riemannian gradient is equal to the Euclidean gradient projected
to the tangent space, we see that the Euclidean gradient is in $T_{\matW}^{\perp}\manM_{[k,1,1]}$,
so we can apply Eq.~(\ref{eq:weingraten map d1 r1 W}) to obtain
that at a stationary point $\matW$: 
\begin{align*}
\Hess F_{0}(\matW)[\matZ] & =\matP_{T_{\matW}\manM_{[k,1,1]}}\left(\nabla^{2}F_{0}(\matW)\right)[\matZ]+\sigma^{-2}\left(\nabla F_{0}(\matW)\matZ^{\T}\matW+\matW\matZ^{\T}\nabla F_{0}(\matW)\right)
\end{align*}

It is easier to analyze the Rayleigh quotient under the assumption
that $\matX$ has orthogonal rows. To justify that we can make this
assumption, note that we can consider a surrogate problem with orthonormal
rows in $\matX$, creating a one-to-one correspondence between the
stationary points, each having the same classification (as a local
minimum or not). Specifically, for a non orthogonal $\matX$, let
$\X=\mathbf{\matL\Q}$ such that $\Q$ has orthonormal rows and $\mathbf{\matL}$
is lower triangular and invertible. We can reformulate Eq.~(\ref{eq:prob-multi-unconstrainted}):
\[
\matW_{0}^{\star}\coloneqq\argmin_{\matW\in\manM_{[k,1,1]}}\FNormS{\matW\matX-\matY}=\argmin_{\matW\in\manM_{[k,1,1]}}\FNormS{\matW\matL\Q-\matY}
\]
Define $\mathit{\mathcal{\matW_{\mathrm{L}}}:=\matW\mathbf{\matL}}$
to obtain
\begin{equation}
\matW_{0,\matL}^{*}\coloneqq\argmin_{\matW\in\manM_{[k,1,1]}}\FNormS{\matW\Q-\matY}\label{eq: problem d1 ortho}
\end{equation}
and we can retrieve the solution easily by$\mathcal{\matW}_{0}^{*}=\matW_{0,\matL}^{*}\matL^{-1}$.
This transformation does not change the topology of the original problem;
see Appendix \ref{appendix: Invariance under orthogonalization}.
So a local minimum stays a local minimum. This transformation does
not change the eigenvalues (only the eigenvectors), so the order of
the generalized eigenvalues remains the same. 

Recall that every stationary point $\sigma\u\v^{\T}$ corresponds
to a generalized eigenpair $(\gamma,\v)$ of the matrix pencil $(\matX\matX^{\T},\matX\matY^{\T}\matY\matX^{\T})$.
Without loss of generality, we can assume that the stationary point
$\sigma\u\v^{\T}$ that corresponds to the eigenpair with the minimal
eigenvalue satisfies $\sigma=1$. This is because we can rescale the
problem $\min\FNorm{\matW\matX-\matY}^{2}\longrightarrow\min\FNorm{\matW\matX-c\matY)}^{2}$
for any scalar $c\neq0$ obtaining a one-to-one correspondence of
the stationary points with the transformation $\matW\to c\matW$ without
changing their classification and ordering among the stationary points
based on the eigenpairs of $(\matX\matX{}^{\T},c^{2}\matX\matY^{\T}\matY\matX{}^{\T})$.
Since $r=1$, a stationary point can be written in the form of $\sigma\u\v^{\T}$.
Scaling by $c$ changes only $\sigma:$ ($\sigma\to c\sigma$), and
does not change $\u$ and $\v$. By choosing $c$ carefully, we can
assure that $\sigma=1$ for one desired stationary point. If the rows
of $\matX$ are orthonormal, e.g. when solving Eq.~(\ref{eq: problem d1 ortho}),
the corresponding generalized eigenvalue $\gamma$ is equal to $1$
as well. To see this, we have $\gamma^{-1/2}=\norm{\matY\Q^{\T}\v}_{2}$
and $\sigma=\gamma^{-1}\ITNorm{\matY\Q^{\T}\v}=\gamma^{-1/2}=1$.

If we first orthogonlize and then rescale, we have both $\sigma=1$
and $\gamma=1$. Considering Eq.~(\ref{eq: problem d1 ortho}), there
is a one--to-one correspondence between a stationary point and a
finite unit-norm \textit{generalized} eigenpair of $(\I_{m},\Q\matY^{\T}\matY\Q^{\T})$.
Such eigenpairs can be viewed as \textit{non-zero} unit-norm \textit{classic}
eigenpairs of the matrix $\Q\matY^{\T}\matY\Q^{\T}$. The following
lemma and theorem concludes our discussion for the case $\manM_{[k,1,1]}.$
\begin{lem}
\label{lem:<Z, H=00005BZ=00005D<0 d=00003D1 r=00003D1}For any stationary
point $\W'\neq\matW_{1}^{(\textrm{RRR})}$, there exists some $\matZ\in T_{\matW'}\manM_{[k,1,1]}$
such that $\left\langle \matZ,\matH_{\matW'}[\matZ]\right\rangle <0$,
where $\left\langle \cdot,\cdot\right\rangle $ is the standard matrix
inner product. 
\end{lem}

\begin{proof}
We assume that the rows of $\matX$ are orthonormal, and the RRR solution
corresponds to the eigenvalue $\sigma=1$\textcolor{blue}{. }See the
above discussion for justification of these assumptions. Let $(\gamma,\v)$
be the eigenpair corresponding to the RRR solution. Note that $\matW_{1}^{(\textrm{RRR})}=\matW_{\gamma,\v}$
due to Lemma~\ref{lem:RRR-min-eigenvalue}. Let $(\gamma',\v')$
be some other eigenpair, and we denote the stationary point corresponding
to this eigenpair by $\W':=\matW_{\gamma',\v'}=\sigma'\u'\v'{}^{\T}$.
We have $\gamma^{-1}\v=\Q\matY^{\T}\matY\Q^{\T}\v$ and $\gamma'{}^{-1}\v'=\Q\matY^{\T}\matY\Q^{\T}\v'$.
Since $\Q\matY^{\T}\matY\Q^{\T}$ is a real symmetric matrix, its
eigenvectors $\v,\v'$ are mutually orthonormal. 

For any $\matZ\in T_{\W'}\manM_{[k,1,1]}$, we can write $\Z=\sigma'\dot{\u}\v'{}^{\T}+\dot{\sigma}\u'\v'{}^{\T}+\sigma\u'\dot{\v}^{\T}$
\cite{Kressner2016precond,KochLuibch10tensor-manifold} where $\dot{\u}^{\T}\u'=0=\dot{\v}^{\T}\v'$.
$\dot{\sigma}$ arbitrary. The general expression (assumptions free)
for $\left\langle \matZ,\matH_{\matW'}[\matZ]\right\rangle $ is
\begin{equation}
\left\langle \matZ,\matH_{\matW'}[\matZ]\right\rangle =\left(\sigma'{}^{2}\norm{\dot{\u}}_{2}^{2}+\dot{\sigma}^{2}\right)\left(\v'{}^{\T}\X\X^{\T}\v'\right)+2\sigma\dot{'\sigma}\left(\v'{}^{\T}\X\X^{\T}\dot{\v}\right)+\sigma'{}^{2}\left(\dot{\v}^{\T}\X\X^{\T}\dot{\v}\right)-2\sigma'\left(\dot{\u}^{\T}\matY\matX^{\T}\dot{\v}\right)\label{eq:d1 r1 <Z,H=00005BZ=00005D>}
\end{equation}
Applying the assumptions of $\matX$'s orthogonality and $\sigma=1$,
we now have

\[
\left\langle \matZ,\matH_{W'}[\matZ]\right\rangle =\left(\sigma'{}^{2}\norm{\dot{\u}}_{2}^{2}+\dot{\sigma}^{2}\right)+2\sigma'\dot{\sigma}\left(\v'{}^{\T}\dot{\v}\right)+\sigma'{}^{2}\norm{\dot{\v}}_{2}^{2}-2\sigma'\left(\dot{\u}^{\T}\matY\Q^{\T}\dot{\v}\right)
\]
We prove the claim by choosing $\Z$ such that $\dot{\sigma}=0$,
$\dot{\u}=\ITNorm{\matY\Q^{\T}\v}\matY\Q^{\T}\v=\u$ and $\dot{\v}=\v$,
thus 

\begin{align*}
\left\langle \matZ,\matH_{\matW'}[\matZ]\right\rangle  & =\sigma'{}^{2}\norm{\dot{\u}}_{2}^{2}+\sigma'{}^{2}\norm{\dot{\v}}_{2}^{2}-2\sigma'\left(\dot{\u}^{\T}\matY\Q^{\T}\dot{\v}\right)\\
 & =2\sigma'{}^{2}-2\sigma'\ITNorm{\matY\Q^{\T}\v}\v^{\T}\Q\matY^{\T}\matY\Q^{\T}\v\\
 & =2\sigma'{}^{2}-2\sigma'\TNorm{\matY\Q^{\T}\v}\\
 & =2\sigma'\left(\sigma'-\TNorm{\matY\Q^{\T}\v}\right)
\end{align*}
We have that $\gamma^{-1/2}=\norm{\matY\Q^{\T}\v}_{2}$ and $\gamma'{}^{-1/2}=\norm{\matY\Q^{\T}\v'}_{2}$
and from our assumption, $\sigma'<\sigma=\gamma^{-1}\ITNorm{\matY\Q^{\T}\v}=\gamma^{-1/2}=1$,
thus $\norm{\matY\Q^{\T}\v}_{2}=1$, therefore 

\[
2\sigma'\left(\sigma'-\left\Vert \matY\Q^{\T}\v\right\Vert _{2}\right)=2\sigma'\left(\sigma'-1\right)<0
\]
and finally we $\left\langle \matZ,\matH_{W}[\matZ]\right\rangle <0$
since $\sigma'<1$.
\end{proof}
\begin{thm}
The Reduced Rank Regression solution to the problem is the only strict
local minimum of the problem.
\end{thm}

\begin{proof}
Let $\W\neq\matW_{1}^{(\textrm{RRR})}$ be a critical point and $\matZ\in T_{\matW}\manM_{[k,1,1]}$.
Considering the Rayleigh quotient of the Riemannian Hessian at critical
points, when viewed as a linear operator on the tangent space:
\[
R(\matH_{W}[\matZ],\matZ)=\frac{\left\langle \matZ,\matH_{W}[\matZ]\right\rangle }{\left\langle \matZ,\matZ\right\rangle }
\]
Since $\matZ\neq0$, we have $\left\langle \matZ,\matZ\right\rangle =\Trace{\matZ^{\T}\matZ}>0$.
From Lemma \ref{lem:<Z, H=00005BZ=00005D<0 d=00003D1 r=00003D1},
for any such $\matW$, there exists some $\matZ\in T_{\matW}\manM_{[k,1,1]}$
such that $\left\langle \matZ,\matH_{W}[\matZ]\right\rangle <0,$
therefore $R(\matH_{W}[\matZ],\matZ)<0$. Then $\matW$ is not a local
minimum of $F(\matW)$, concluding that $\matW_{1}^{(\textrm{RRR})}$
is the only strict local minimum to the problem. 
\end{proof}
\begin{rem}
We see from the proof that if the the top singular value of $\matA$
has multiplicity larger than $1$, then for each corresponding stationary
point the Hessian will be strictly positive semi-definite, and the
point is only a local minimum, and not a strict local minimum. All
other stationary points with singular value smaller than the maximum,
are not local minimum. 
\end{rem}

\paragraph{Classification of stationary points for $r>1$.}

Consider the problem with $r>1$ and $\lambda=0$, i.e. the constraint
manifold is $\manM_{[k,r,1]}.$ The following lemma characterizes
stationary points of HORRR for this case, and the theorem afterwards
shows that the RRR solution is the only strict local minimum. 
\begin{lem}
\label{lem:sum of critical}A point $\W\in\manM_{[k,r,1]}$ is stationary
point if and only if we can write $\matW=\sum_{i=1}^{r}\matW_{i}$
where each $\matW_{i}\in\manM_{[k,1,1]}$ is a stationary point for
$r=1$.
\end{lem}

\begin{proof}
In the previous section we explained why we can assume without loss
of generality that $\matX$ has orthonormal rows. We continue the
proof under this assumption.

Assume that \textit{$\matW=\sum_{i=1}^{r}\matW_{i}$ }\textit{\emph{where
each $\matW_{i}\in\manM_{[k,r,1]}$ is a stationary point for $r=1$.
Write $\matW_{i}=\sigma_{i}\u_{i}\v_{i}^{\T}$ where both $\u_{i}$
and $\v_{i}$ have unit norm. Since $\matX$ has orthonormal rows,
each $\v_{i}$ is a generalized eigenvalue of $(\I_{m},\matX\matY^{\T}\matY\matX^{\T})$,
which implies that they are classical eigenvalues of $\matX\matY^{\T}\matY\matX^{\T}$.
This matrix is symmetric, hence the eigenvalues are mutually orthogonal.
Since $\u_{i}\propto\matY\matX^{\T}\v_{i}$ for $i=1,\dots,r$, and
$\v_{i}$ are eigenvectors of $\matX\matY^{\T}\matY\matX^{\T}$, it
follows that the $\u_{i}$s are mutually orthogonal too. We can write
$\matW=\matU\Sigma\matV^{\T}$ where $\matU$ has $\u_{1},\dots,\u_{r}$
as columns, $\matV$ has $\v_{1},\dots,\v_{r}$ as columns, and $\Sigma$
has $\sigma_{1},\dots,\sigma_{r}$ on the diagonal. This is a SVD
decomposition of $\matW$. }}

\textit{\emph{Since each $\matW_{i}$ is a stationary point, via Claim~}}\ref{claim:critical d=00003D1},
Eq. (\ref{eq:criticial-d1-eq1}) we have:
\[
\matW_{i}\matX\matX^{\T}=\u_{i}\u_{i}^{\T}\matY\matX^{\T}
\]
Summing the equations, we have
\[
\matW\matX\matX^{\T}=\left(\stackrel[i=1]{r}{\sum}\W_{i}\right)\matX\matX^{\T}=\left(\stackrel[i=1]{r}{\sum}\u_{i}\u_{i}^{\T}\right)\matY\matX^{\T}=\matU\matU^{\T}\matY\matX^{\T}=\matW\matW^{\pinv}\matY\matX^{\T}
\]
From Eq. (\ref{eq:criticial-d1-eq2}):
\[
\matY\matX^{\T}\v_{i}\v_{i}^{\T}=\u_{i}\u_{i}^{\T}\matY\matX^{\T}\v_{i}\v_{i}^{\T}
\]
Thus, 
\[
\matY\matX^{\T}\left(\stackrel[i=1]{r}{\sum}\v\v_{i}^{\T}\right)=\sum_{i=1}^{r}\u_{i}\u_{i}^{\T}\matY\matX^{\T}\v_{i}\v_{i}^{\T}=\stackrel[j=1]{r}{\sum}\left(\left(\stackrel[i=1]{r}{\sum}\u_{i}\u_{i}^{\T}\right)\matY\matX^{\T}\v_{j}\v_{j}^{\T}\right)
\]
where the last equality follows from observing that for $i\neq j$
we have $\u_{i}\u_{i}^{\T}\matY\matX^{\T}\v_{j}\v_{j}^{\T}=0$ due
to orthonormality of the $\v_{i}$s. Thus, 
\[
\matY\matX^{\T}\V\V^{\T}=\matU\matU^{\T}\matY\matX^{\T}\V\V^{\T}
\]
We have shown that $\W$ satisfies both Eqs. (\ref{eq:criticial-d1-eq1})
and (\ref{eq:criticial-d1-eq2}), hence $\grad F_{0}(\matW)=0$, therefore
$\W$ is indeed a critical point. 

For the other direction, assume that $\W\in\manM_{[k,r,1]}$ is critical
point of the problem, and write $\matW=\matU\Sigma\matV^{\T}$ a reduced
SVD of $\matW$. Claim \ref{claim:critical d=00003D1}\textcolor{blue}{{}
}holds for $\W$, so from Eq. (\ref{eq:criticial-d1-eq1}) we have
\[
\matW\matX\matX^{\T}=\matU\matU^{\T}\matY\matX^{\T}
\]
For any $1\leq i\leq r$ , we can left multiply with $\u_{i}\u_{i}^{\T}$
(column $i$ of $\matU$), and utilize the mutual orthogonality of
the singular vectors of $\text{\ensuremath{\W}}$, to obtain:
\begin{equation}
\sigma_{i}\u_{i}\v_{i}^{\T}\matX\matX^{\T}=\u_{i}\u_{i}^{\T}\matY\matX^{\T}\label{eq:criticial-d1-eq1 i}
\end{equation}
Therefore, each $\left\{ \W_{i}\coloneqq\sigma_{i}\u_{i}\v_{i}^{\T}\right\} _{1\leq i\leq r}$
satisfies Eq. (\ref{eq:criticial-d1-eq1}) when $r=1$. From Eq. (\ref{eq:criticial-d1-eq2})
we have
\[
(\matI-\matU\matU^{\T})\matY\matX^{\T}\V\V^{\T}=0
\]
Similarly, for all $1\leq i\leq r$ , we can right multiply with $\v_{i}\v_{i}^{\T}$,
and utilize the mutual orthogonality to obtain:
\begin{align*}
(\matI-\matU\matU^{\T})\matY\matX^{\T}\v_{i}\v_{i}^{\T} & =0\\
(\matI-\stackrel[i=1]{r}{\sum}\u_{i}\u_{i}^{\T})\matY\matX^{\T}\v_{i}\v_{i}^{\T} & =0\\
(\matI-\u_{i}\u_{i}^{\T})\matY\matX^{\T}\v_{i}\v_{i}^{\T} & =\stackrel[j=1,j\neq i]{r}{\sum}\u_{j}\u_{j}^{\T}\matY\matX^{\T}\v_{i}\v_{i}^{\T}
\end{align*}
Using the previous result from Eq. (\ref{eq:criticial-d1-eq1 i}),
we have $\u_{j}\u_{j}^{\T}\matY\X^{\T}=\sigma_{j}\u_{j}\v_{j}^{\T}\X\X^{\T}=\sigma_{j}\u_{j}\v_{j}^{\T}$,
where we also used our assumption that $\matX$ has orthonormal rows,
i.e. $\X\X^{\T}=\I_{m}$. Thus, 
\[
\stackrel[j=1,j\neq i]{r}{\sum}\u_{j}\u_{j}^{\T}\matY\X^{\T}\v_{i}\v_{i}^{\T}=\stackrel[j=1,j\neq i]{r}{\sum}\sigma_{j}\u_{j}\v_{j}^{\T}\v_{i}\v_{i}^{\T}=0
\]
Where the last equality follows from the mutual orthogonality of the
$\v_{i}$'s . Therefore, each $\left\{ \W_{i}\right\} _{1\leq i\leq r}$
also satisfies Eq. (\ref{eq:criticial-d1-eq2}) when $r=1$. From
claim \ref{claim:critical d=00003D1}\textcolor{blue}{, }we obtain
$\grad F_{0}(\matW_{i})=0$ for all $1\leq i\leq r$ , hence all $\left\{ \W_{i}\right\} _{1\leq i\leq r}$
are critical points for the $r=1$ case. 
\end{proof}
\begin{thm}
The Reduced Rank Regression solution to the problem when $r>1$ is
the only strict local minimum of the problem. 
\end{thm}

\begin{proof}
Let $\matW_{1},\dots,\matW_{m}\in\manM_{[k,1,1]}$ denote the $m$
different stationary points of the case $r=1$, and order them by
descending singular values of $\matA$. Stationary points of the problem
are sums of $r$ distinct $\matW_{i}$s, while the RRR solution is
$\matW_{r}^{(\textrm{RRR})}=\sum_{i=1}^{r}\matW_{i}.$ Let $\matW\neq\matW_{r}^{(\textrm{RRR})}$
be a stationary point which is \emph{not} the RRR solution. We will
show that it is not a local minimum, by finding a $\matZ\in T_{\matW}\manM_{[k,r,1]}$
such that $\left\langle \matZ,\matH[\matZ]\right\rangle <0$. 

Since $\W\neq\matW_{r}^{(\textrm{RRR})}$ there exist an index $j\leq r$
for which $\matW_{j}$ is not present when expressing $\matW$ as
a sum of elements from $\matW_{1},\dots,\matW_{m}$. We rescale the
problem, as we did for the case of $r=1$, so that $\sigma_{j}=1$.
An element $\matZ\in T_{\W}\manM_{[k,r,1]}$ can be written $\Z=\dot{\matU}\Sigma\V^{\T}+\matU\dot{\Sigma}\V^{\T}+\matU\Sigma\dot{\V}^{\T}$
where $\dot{\matU}^{\T}\matU=0,\dot{\V^{\T}\V=0}$, and $\dot{\Sigma}$
arbitrary \cite{KochLuibch10tensor-manifold,Kressner2016precond}.
We now show how to choose $\dot{\matU}$, $\dot{\matV}$ and $\dot{\Sigma}$.
First we choose $\dot{\Sigma}=0$ and obtain 
\begin{equation}
\left\langle \matZ,\matH[\matZ]\right\rangle =\Trace{\dot{\matU}^{\T}\dot{\matU}\Sigma^{2}}+\Trace{\dot{\V}^{\T}\dot{\V}\Sigma^{2}}-2\Trace{\Sigma\dot{\matU}^{\T}\matY\matX^{\T}\dot{\V}}\label{eq:<Z,H=00005BZ=00005D> r>1}
\end{equation}

Next, we choose $\dot{\V},\dot{\matU}$ such that all of the columns
of $\dot{\V},\dot{\matU}$ are zeros, except for their last column
which we choose to be $\v_{j}$ and $\u_{j}:=\ITNorm{\matY\matX^{\T}\v_{j}}\matY\matX^{\T}\v_{j}$
(respectively). The orthonormality of the $\v$s and $\u$s yields
$\dot{\matU}^{\T}\matU=0_{r\times r}=\dot{\V}^{\T}\V$ as required,
and 
\[
\dot{\matU}^{\T}\dot{\matU}=\dot{\V}^{\T}\dot{\V}=\left[\begin{array}{cccc}
0 & \cdots & \cdots & 0\\
\vdots & \ddots & \cdots & \vdots\\
\vdots &  & 0 & 0\\
0 & \cdots & 0 & 1
\end{array}\right]
\]
and 
\[
\dot{\matU}^{\T}\matY\matX^{\T}\dot{\V}=\TNorm{\matY\matX^{\T}\v_{j}}\left[\begin{array}{ccc}
0 & \cdots & 0\\
0 & \cdots & 0\\
\vdots &  & \vdots\\
0 & \cdots & 1
\end{array}\right]
\]
Thus, $\Trace{\dot{\matU}^{\T}\dot{\matU}\Sigma^{2}}=\sigma_{p}^{2}=\Trace{\dot{\V}^{\T}\dot{\V}\Sigma^{2}}$
and $\Trace{\Sigma\dot{\matU}^{\T}\matY\matX^{\T}\dot{\V}}=\TNorm{\matY\matX^{\T}\v_{j}}\sigma_{p}$
where $\sigma_{p}$ is the smallest $\sigma$ that is present in $\matW$.
Eq. (\ref{eq:<Z,H=00005BZ=00005D> r>1}) becomes

\[
\left\langle \matZ,\matH[\matZ]\right\rangle =2\left(\sigma_{p}^{2}-\TNorm{\matY\matX^{\T}\v_{j}}\sigma_{p}\right)=2\sigma_{p}\left(\sigma_{p}-\TNorm{\matY\matX^{\T}\v_{j}}\right)
\]
Per our assumption, we have $\sigma_{j}=1$ and $\left\Vert \matY\matX^{\T}\v_{j}\right\Vert _{2}=1$
. Therefore, $0<\sigma_{p}<\sigma_{r}\leq\sigma_{j}=1$, so $\left\langle \matZ,\matH[\matZ]\right\rangle <0$,
concluding the proof.
\end{proof}

\paragraph{Classification of stationary points for $\lambda>0$.}

Finally, for the matrix case we consider $\lambda>0$. The problem
now reads:
\begin{equation}
\matW_{\lambda}^{\star}\coloneqq\argmin_{\matW\in\manM_{[k,r,1]}.}\FNormS{\matW\matX-\matY}+\lambda\FNormS{\matW}\label{eq:d1 r1 problem lambda}
\end{equation}
The Riemannian gradient of Problem (\ref{eq:d1 r1 problem lambda})
is:

\begin{align*}
\grad F_{\lambda}(\matW) & =\matP_{T_{\matW}\manM_{[k,r,1]}}\left(\nabla F_{\lambda}(\matW)\right)\\
 & =\matP_{T_{\matW}\manM_{[k,r,1]}}\left(\nabla F_{0}(\matW)+\lambda\matW\right)\\
 & =\grad F_{0}(\matW)+\lambda\matW\\
 & =\matW\matX\matX^{\T}-\matW\matW^{\pinv}\matY\matX^{\T}-(\matI-\matW\matW^{\pinv})\matY\matX^{\T}\matW^{\pinv}\matW+\lambda\matW\\
 & =\matW\left(\matX\matX^{\T}+\lambda\I_{m}\right)-\matW\matW^{\pinv}\matY\matX^{\T}-(\matI-\matW\matW^{\pinv})\matY\matX^{\T}\matW^{\pinv}\matW
\end{align*}
We wish to use our previous results. To that end, we define $\hat{\X}=[\X\;\sqrt{\lambda}\I_{m}]\in\R^{m\times n+m}$
and $\hat{\matY}=[\matY\;\mathbf{0}_{k\times m}]\in\R^{k\times n+m}$.
For $d=1$, Problem (\ref{eq:target function}) can be rewritten as
\cite{Mukherjee2011kernel}:

\begin{equation}
\hat{F}(\W)=\frac{1}{2}\FNormS{\matW\hat{\X}-\hat{\matY}}\label{eq:d1 r1 problem  X^,Y^}
\end{equation}
We show that we can follow all previous steps, when minimizing $\hat{F}(\W)$
instead of $F_{0}(\W)$, i.e. when replacing $\X,\matY$ with $\hat{\X},\hat{\matY},$
to obtain the same conclusions for the case when $\lambda>0$.

The Riemannian gradient of problem (\ref{eq:d1 r1 problem  X^,Y^})
is:

\begin{align*}
\grad\hat{F}(\matW) & =\matP_{T_{\matW}\manM_{[k,r,1]}}\left(\nabla\hat{F}(\matW)\right)\\
 & =\matW\hat{\X}\hat{\X}^{\T}-\matW\matW^{\pinv}\hat{\matY}\hat{\X}^{\T}-(\matI-\matW\matW^{\pinv})\hat{\matY}\hat{\X}^{\T}\matW^{\pinv}\matW\\
 & =\matW\left(\matX\matX^{\T}+\lambda\I_{m}\right)-\matW\matW^{\pinv}\matY\matX^{\T}-(\matI-\matW\matW^{\pinv})\matY\matX^{\T}\matW^{\pinv}\matW
\end{align*}
Where the last equality follows from $\matX\matX^{\T}+\lambda\I_{m}=\hat{\X}\hat{\X}^{\T}$
and $\matY\matX^{\T}=\hat{\matY}\hat{\X}^{\T}$. We see that $\grad\hat{F}(\matW)=\grad F_{\lambda}(\W)$,
thus the stationary points coincide.

Any \emph{finite} unit-norm generalized eigenpair $(\hat{\gamma},\hat{\v})$
of the pencil $(\hat{\X}\hat{\X}^{\T},\hat{\X}\hat{\matY}^{\T}\hat{\matY}\hat{\X}^{\T})$
can be used to generate the stationary points: 
\[
\matW_{\hat{\gamma},\hat{\v}}=\frac{1}{\hat{\gamma}\TNormS{\hat{\matY}\hat{\X}^{\T}\hat{\v}}}\hat{\matY}\hat{\X}^{\T}\hat{\v}\hat{\v}^{\T}=\frac{1}{\hat{\gamma}\TNormS{\matY\matX^{\T}\hat{\v}}}\matY\matX^{\T}\hat{\v}\hat{\v}^{\T}
\]
and following the same steps, we can conclude that the Reduced Rank
Regression solution to the problem is the only strict local minimum
of the problem.

\subsection{Higher Order}

We now consider the case $d>1$. Unfortunately, the higher order case
is much more complex, and we do not have a complete characterization
of the stationary points, like we have for $d=1$. For simplicity,
we mainly analyze at the case of $d=2$. 

\subsubsection{\label{subsec:Recoring}Recoring}

Consider a stationary point $\left\llbracket \tenC;\matU_{1},\matU_{2}\dots,\matU_{d+1}\right\rrbracket $.
It will generally not uphold the \emph{residual condition }(Eq.~(\ref{eq:residual-condition}))
which is a necessary condition for stationarity, but not a sufficient
condition. However, it is possible to modify $\tenC$ so that the
residual condition holds, \emph{without} modifying $\matU_{1},\dots,\matU_{d+1}$.
Furthermore, we can also fix $\matU_{1}$ to any square orthonormal
matrix. Thus, in a sense, only $\matU_{2},\dots,\matU_{d+1}$ should
be free parameters for the optimization, since there is a natural
selection for $\tenC$ once these are set. To see this, let us first
write $\matZ=\matU_{d+1}^{\T}\matX\odot\cdots\odot\matU_{2}^{\T}\matX$.
Recalling that $\tenW\matX=\matU_{1}\unfold{\tenC}1\matZ$, the residual
condition now becomes
\[
(\matU_{1}\unfold{\tenC}1\matZ-\matY)\matZ^{\T}=-\lambda\matU_{1}\unfold{\tenC}1
\]
Multiplying by $\matU_{1}^{\T}$ yields 
\[
(\unfold{\tenC}1\matZ-\matU_{1}^{\T}\matY)\matZ^{\T}=-\lambda\unfold{\tenC}1
\]
and rearranging shows that the condition translates to 
\[
\unfold{\tenC}1(\matZ\matZ^{\T}+\lambda\matI)=\matU_{1}^{\T}\matY\matZ^{\T}
\]
If $\lambda>0$ the matrix $\matZ\matZ^{\T}+\lambda\matI$ is invertible,
and we can solve for $\tenC$. We can also solve for $\tenC$ when
$\Z$ has full row rank. The cost of modifying $\tenC$ is $O(nr^{2d})$. 

We call the process of modifying $\tenC$ so that Eq.~(\ref{eq:residual-condition})
holds by the name ``recoring''. Though recoring is a heuristic,
numerical experiments have shown that periodically applying a recoring
step every fixed amount of iterations improves cost, reduces the gradient
considerably, and improves downstream performance. 

\subsubsection{Semi-Symmetric Stationary Points}

A tensor is called \textit{semi-symmetric} if its entries are invariable
under any permutations of all indices but the the first. For every
tensor $\tenA$ there exist a semi-symmetric tensor $\tilde{\tenA}$
of the same dimensions, such that $\tenA\x=\tilde{\tenA}\x$ for any
vector $\x$~\cite{ding2015generalized}, i.e. $\tenA$ and $\tilde{\tenA}$
are equivalent in terms of how they are applied to a vector or a matrix.
However, a semi-symmetric $\tilde{\tenA}$ which is equivalent to
a non semi-symmetric $\tenA$ might have a different multilinear rank.
Nevertheless, we can, in a sense, restrict ourselves to trying to
solve HORRR only with semi-symmetric tensors, if we allow modification
of the target rank $r$. We now prove a few lemmas regarding semi-symmetric
tensors of order $d+1$ in $\R^{k\times m\times...\times m}$.
\begin{lem}
Given two modes $p,q$, denote by $\pi(p,q)$ the permutation that
exchanges $p$ and $q$, but keeps other modes in place. For any $p,q\neq1$
and tensor $\tenA$, $\tenA_{(p)}=\tenA_{(q)}$ if and only if $\tenA=\tenA_{\pi(p,q)}$.
\end{lem}

\begin{proof}
Without loss of generality, we assume $p<q$. Let $\tenA=\tenA_{\pi(p,q)}$,
then we have $a_{i_{1},i_{2},..,i_{p},...,i_{q}\dots i_{d+1}}=a_{i_{1},i_{2},..,i_{q},...,i_{p},\dots,i_{d+1}}$
for all $p,q\neq1$. It follows that the mode-$p$ fibers and the
mode-$q$ fibers of $\tenA$ are equal. The columns of $\tenA_{(p)},\tenA_{(q)}$
are the mode-$p$ and mode-$q$ fibers, respectively, so from the
equivalence of the fibers we obtain equivalence of the matricizations. 

For the opposite direction, assume $\tenA_{(p)}=\tenA_{(q)}$. We
can form $\tenA$ by mapping back the columns of $\tenA_{(p)}$ to
the mode-$p$ fibers of $\tenA$. On the other hand, we can also map
columns of $\tenA_{(q)}$ to the mode-$q$ fibers of $\tenA$. Since
$\tenA_{(p)}=\tenA_{(q)}$, we obtain that the mode-$p$ fibers and
mode-$q$ fibers are equal, hence $a_{i_{1},i_{2},..i_{p},...i_{q}.,i_{d+1}}=a_{i_{1},i_{2},..i_{q},...i_{p}.,i_{d+1}}$for
any $p,q\neq1$ and $\tenA=\tenA_{\pi(p,q)}$. 
\end{proof}
\begin{lem}
A tensor $\tenA$ is semi-symmetric if and only if for all $p,q\neq1$
we have $\tenA_{(p)}=\tenA_{(q)}$.
\end{lem}

\begin{proof}
Suppose the $\tenA$ is semi-symmetric. Consider two indices $p,q\neq1$.
Due to semi-symmetry, $\tenA=\tenA_{\pi(p,q)}$ so by the previous
lemma $\tenA_{(p)}=\tenA_{(q)}$.

As for the other direction, assume that $\tenA_{(p)}=\tenA_{(q)}$
for all $p,q\neq1$, and let $\pi$ be a permutation that keeps the
first index in place. we can write $\pi$ as a finite sequence of
permutations such that each permutation in the sequence exchanges
only two modes: $\pi=\pi(p_{l},q_{l})\circ...\circ\pi(p_{2},q_{2})\circ\pi(p_{1},q_{1})$,
with $p_{i},q_{i}\neq1$ for all $i\in[l]$. We denote $\tenA_{\pi_{i}}$
the tensor that is obtained after applying the $i$th permutation
on $\tenA_{\pi_{i-1}}$, and $\tenA_{\pi_{0}}=\tenA$. We show that
$\tenA_{\pi_{i}}=\tenA$. This is obviously true for $i=0$. Next,
assume by induction that $\tenA_{\pi_{i-1}}=\tenA$. $\tenA_{\pi_{i}}$
is obtained from $\tenA_{\pi_{i-1}}$ by exchanging only two indices
that are not 1. By the previous lemma, the result of permuting by
$\pi_{i}$ is the same if the matricizations by the two modes are
equal. Since $\tenA_{\pi_{i-1}}=\tenA$ this holds by assumption.\textcolor{blue}{{} }
\end{proof}
\begin{lem}
\label{lem:HOSVDS-semi-symmetry}Suppose of $\tenA\in\R^{k\times m\times\cdots\times m}$
is semi-symmetric, then for any tuple $\rb=(l,r,\dots,r)$ with $l\leq k$
and $r\leq m$, we can write $\HOSVD{\tenA}{\rb}=\left\llbracket \tenC\,;\,\matU_{1},\matU,\dots,\matU\right\rrbracket $
where $\tenC$ is semi-symmetric, as long as the singular vectors
chosen for different modes in the HOSVD are chosen in a deterministic
manner. In particular, $\HOSVD{\tenA}{\rb}$ is semi-symmetric as
well. A simple corollary is that any semi-symmetric $\tenA\in\R^{k\times m\times\cdots\times m}$
can be written in the form $\tenA=\left\llbracket \tenC\,;\,\matU_{1},\matU,\dots,\matU\right\rrbracket $
where $\tenC$ is semi-symmetric.
\end{lem}

\begin{proof}
In HOSVD, the factors $\matU_{2},\dots,\matU_{d+1}$ are obtained
from the leading $r$ singular vectors for $\ten A_{(2)},\dots,\tenA_{(d+1)}$.
Since $\tenA$ is semi-symmetric we have $\tenA_{(2)}=\tenA_{(3)}=\cdots=\tenA_{(d+1)}$,
so as long as the singular vectors are chosen in a deterministic manner
we will have $\matU_{2}=\matU_{3}=\cdots=\matU_{d+1}$. Next, $\tenC=\tenA\modeprod 1\matU_{1}^{\T}\modeprod 2\matU^{\T}\cdots\modeprod{d+1}\matU^{\T}$.
Thus, for $p\neq1$:
\[
\tenC_{(p)}=\matU^{\T}\tenA_{(p)}\left(\matU\otimes\cdots\otimes\matU\otimes\matU_{1}\right)
\]
Since $\tenA_{(p)}=\tenA_{(q)}$ for every $p,q\neq1$, we find that
$\tenC_{(p)}=\tenC_{(q)}$ for any such $p,q$. Hence, by the previous
lemma, $\tenC$ is semi-symmetric. 

Semi-symmetry of $\HOSVD{\tenA}{\rb}$ follows from a similar argument.
Denote $\tenB=\HOSVD{\tenA}{\rb}$. Then, since $\tenB=\tenC\modeprod 1\matU_{1}\modeprod 2\matU\cdots\modeprod{d+1}\matU$
we have for $p\neq1$:
\[
\tenB_{(p)}=\matU\tenC_{(p)}\left(\matU\otimes\cdots\otimes\matU\otimes\matU_{1}\right)^{\T}
\]
Now, since $\tenC$ is semi-symmetric, we have $\tenC_{(p)}=\tenC_{(q)}$
for any $p,q\neq1$, so $\tenB_{(p)}=\tenB_{(q)}$ for any such $p,q$. 

For the final part of the lemma, note that $\tenA_{(p)}=\tenA_{(q)}$
for any $p,q\neq1$, so the multilinear rank of $\tenA$ is of the
form $\rank(\ten A)=(l,r,\dots,r)$ with $l\leq k$ and $r\leq m$.
Since $\tenA=\HOSVD{\tenA}{\rank(\ten A)}$ (\cite{de2000multilinearSVD}),
the claim follows from the first part of the lemma.
\end{proof}
This brings us to the main result of this subsection: that if the
initial point is semi-symmetric, the natural Riemannian gradient descent
iteration stays semi-symmetric, and thus any stationary point which
is the limit of such iterations is semi-symmetric. Namely, only a
minor modification of the algorithm (choosing a semi-symmetric initial
point) is needed to incorporate semi-symmetry as a constraint. 
\begin{prop}
Consider the following Riemannian gradient descent iteration: 
\[
\tenW^{(j+1)}\gets\HOSVD{\tenW^{(j)}+\alpha^{(j)}\grad F_{\lambda}\left(\tenW^{(j)}\right)}{[k,r,d]}
\]
 where $\alpha^{(1)},\alpha^{(2)},\dots$ are step sizes. In the above,
HOSVD plays the role of retraction~\cite{Kressner2013completion}.
If $\tenW^{(0)}\in\manM_{[k,r,d]}$ is semi-symmetric, then subsequent
iterates $\tenW^{(1)},\tenW^{(2)},\dots$ are semi-symmetric.
\end{prop}

\begin{proof}
Assume by induction that $\tenW^{(j)}\in\manM_{[k,r,d]}$ is semi-symmetric.
We necessarily have $\rank\left(\tenW^{(j)}\right)\leq(k,r,\dots,r)$
(element-wise). By Lemma \ref{lem:HOSVDS-semi-symmetry} we can write
$\tenW^{(j)}=\left\llbracket \tenC\,;\,\matU_{1},\matU,\dots,\matU\right\rrbracket $
where $\tenC$ is semi-symmetric (possibly padding with zeros and
adding orthogonal columns to $\matU$ if $\rank\left(\tenW^{(j)}\right)$
does not equal $(k,r,\dots,r)$). In the previous, we did not include
an iteration index for $\tenC,\matU_{1}$ and $\matU$ for ease of
notation.

First consider the Riemannian gradient. We have $\grad\left(\tenW^{(j)}\right)=\P_{T_{\tenW^{(j)}}\manM_{[k,r,d]}}\left(\nabla F_{\lambda}\left(\tenW^{(j)}\right)\right)$
and $\nabla F_{\lambda}\left(\tenW^{(j)}\right)=\left\llbracket \e;\tenW^{(j)}\matX-\matY,\matX,...,\matX\right\rrbracket +\lambda\tenW^{(j)}$.
The first term in $\nabla F_{\lambda}\left(\tenW^{(j)}\right)$ is
semi-symmetric since obviously mode unfolding for modes different
from the first are the same. This can be seen from formulas for unfolding
of CP decompositions. The second term is also semi-symmetric by the
inductive assumption. Thus, their sum, $\nabla F_{\lambda}\left(\tenW^{(j)}\right)$,
is semi-symmetric. Next, we project the sum to obtain the Riemannian
gradient. Let $\left\{ \tenG\,;\,0,\matV_{2},\dots,\matV_{d+1}\right\} =\P_{T_{\tenW^{(j)}}\manM_{[k,r,d]}}\left(\nabla F_{\lambda}\left(\tenW^{(j)}\right)\right)$.
Then, 
\[
\tenG=\nabla F_{\lambda}\left(\tenW^{(j)}\right)\modeprod{_{1}}\matU_{1}\modeprodrange{i=2}{d+1}\matU^{\T}\;\;,\;\;\matV_{i}=\matP_{\matU}^{\perp}\unfold{\left[\nabla F_{\lambda}\left(\tenW^{(j)}\right)\modeprod 1\matU_{1}\modeprod{l\neq i}\matU^{\T}\right]}i\unfold{\tenC}i^{\pinv}\,\,i=2,\dots,d+1
\]
Since $\tenC$ is semi-symmetric, $\tenC_{(2)}=\dots=\tenC_{(d+1)}$
then $\matV_{2}=\dots=\matV_{d+1}\eqqcolon\matV$. As for $\tenG$,
since $\nabla F_{\lambda}\left(\tenW^{(j)}\right)$ is semi-symmetric,
and mode $2$ to $d+1$ are multiplied by the same matrix, the mode
unfolding of modes $2,\dots,d+1$ are the same, and so it too is semi-symmetric.
Now, 
\[
\P_{T_{\tenW^{(j)}}\manM_{[k,r,d]}}\left(\nabla F_{\lambda}\left(\tenW^{(j)}\right)\right)=\tenG\modeprod 1\matU_{1}\modeprodrange{i=2}{d+1}\matU+\stackrel[i=2]{d+1}{\sum}\tenC\modeprod 1\matU_{1}\modeprod i\V\modeprod{l\neq i}\matU
\]
The first term is semi-symmetric since $\tenG$ is semi-symmetric
and we multiply by the same matrix in all modes but the first. As
for the second term, for $p=2,\dots,d+1$: {\small
\[
\left(\stackrel[i=2]{d+1}{\sum}\tenC\modeprod 1\matU_{1}\modeprod i\V\modeprod{l\neq i}\matU\right)_{(p)}=\matV\tenC_{(p)}\left(\matU\otimes\cdots\otimes\matU\otimes\matU_{1}\right)+\sum_{q\neq p}\matU\tenC_{(q)}\left(\matU\otimes\cdots\otimes\matU\otimes\underbrace{\matV}_{p\,\text{index}}\otimes\matU\otimes\cdots\otimes\matU\otimes\matU_{1}\right)^{\T}
\]
}We see that the mode unfoldings are the same for $p=2,\dots,d+1$
hence the second term is semi-symmetric. Since $\P_{T_{\tenW^{(j)}}\manM_{[k,r,d]}}\left(\nabla F_{\lambda}\left(\tenW^{(j)}\right)\right)$
is the sum of two semi-symmetric tensors, it too is semi-symmetric.

Finally, $\ten W^{(j+1)}$ is obtained from HOSVD of a semi-symmetric
tensor, so the result is semi-symmetric as well by Lemma~\ref{lem:HOSVDS-semi-symmetry}.
\end{proof}

\subsubsection{Semi-Symmetric Stationary Points for $d=2$, $r=1$, $\lambda=0$}

We have seen that for $d=1$ the stationary points for $r=1$ correspond
to generalized eigenvalues of a certain matrix pencil, and that stationary
points of higher ranks ($r>1$) are obtained by combining stationary
points for $r=1$. We now show that $r=1,d=2$ stationary points correspond
to generalized \emph{tensor} eigenpairs. However, it is unclear whether
stationary points of higher ranks can be obtained by combining stationary
points of $r=1$.

Given a HORRR problem defined by matrices $\matX$ and $\matY$, we
define the following two $\R^{m\times m\times m\times m}$ tensors:
\[
\tenZ_{\matX\matX}=\left\llbracket \e;\matX,\matX,\matX,\matX\right\rrbracket 
\]
\[
\tenZ_{\matX\matY}=\left\langle \left\llbracket \e;\matY,\matX,\matX\right\rrbracket ,\left\llbracket \e;\matY,\matX,\matX\right\rrbracket \right\rangle _{1}
\]
where $\e$ denotes the all-ones vector, and $\left\langle \cdot,\cdot\right\rangle _{1}$
denotes outer-product and contraction along the first mode of each
tensor (see~\cite[Section 3.3]{KoldaBader2006algorithm}). Both tensors
can be viewed as a form of higher-order Gram tensors. 

A pair $(\gamma,\u)$ is \textit{B-eigenpair }\cite{chang2009-B-eigen}
of the tensor pencil $\left(\tenZ_{\matX\matX},\tenZ_{\matX\matY}\right)$
if it satisfies 
\begin{align*}
\tenZ_{\matX\matX}\u & =\gamma\tenZ_{\matX\matY}\u\\
\u^{\T}\u & =1
\end{align*}
Obviously, this is a tensor generalization of generalized matrix eigenpairs. 

Consider a semi-symmetric tensor $\tenW\in\manM_{[k,1,2]}$. We can
write $\tenW=\left\llbracket \tenC\,;\,\matU_{1},\u,\u\right\rrbracket $,
where $\tenC\in\R^{k\times1\times1}$, $\matU_{1}\in\R^{k\times k}$,
and $\u\in\R^{m}$ with unit norm. Denote $\mathbf{c}=\squeeze{2,3}\left(\tenC\right)\in\R^{k}$.
We have $\tenW=\left\llbracket \matU_{1}\mathbf{c},\u,\u\right\rrbracket $,
where here we use $\left\llbracket \cdots\right\rrbracket $ to denote
a CP factorization.
\begin{thm}
\label{thm:tensor eigenpairs}If a semi-symmetric tensor $\tenW=\left\llbracket \matU_{1}\mathbf{c},\u,\u\right\rrbracket $
with $\TNorm{\u}=1$ is a stationary point, then there exists a $\gamma\in\R$
such that $(\gamma,\u)$ is B-eigenpair of the tensor pencil $\left(\tenZ_{\matX\matX},\tenZ_{\matX\matY}\right)$,
i.e. 
\[
\tenZ_{\matX\matX}\u=\gamma\tenZ_{\matX\matY}\u
\]
Furthermore, if $\u$ is a B-eigenpair of the tensor pencil $\left(\tenZ_{\matX\matX},\tenZ_{\matX\matY}\right)$
then there exist $\matU_{1}$ and $\mathbf{c}$ such that $\tenW=\left\llbracket \matU_{1}\mathbf{c},\u,\u\right\rrbracket $
is a stationary point. 
\end{thm}

\begin{proof}
First, we look at an arbitrary stationary point $\tenW=\left\llbracket \tenC\,;\,\matU_{1},\u_{2},\u_{3}\right\rrbracket .$
Then $\tenW$ satisfies $\grad F_{0}(\tenW)=\matP_{\tenW}\nabla F_{0}(\tenW)=\left\{ \ten G_{0};0,\v_{2},\v_{3}\right\} =0$,
where $\tenG_{0}=\nabla F_{0}(\tenW)\modeprod 1\matU_{1}^{\T}\modeprod 2\u_{2}^{\T}\modeprod 3\u_{3}^{\T}$
and $\V_{i}=\P_{\matU_{i}}^{\perp}\left[\nabla F_{0}(\tenW)\times_{j\neq i}\matU_{j}^{\T}\right]_{(i)}\tenC_{(i)}^{+}$
for $i=2,3$. From Lemma \ref{lem:zero-factors-critical}, we have
$\tenG_{0}=0$ which leads to 
\begin{equation}
\nabla F_{0}(\tenW)\modeprod 1\mat{\I_{k}}\modeprod 2\u_{2}^{\T}\modeprod 3\u_{3}^{\T}=0_{k\times1\times1}\label{eq:critical d=00003D2 g=00003D0}
\end{equation}
For the sake of readability, we denote $\B:=\u_{3}^{\T}\matX\odot\u_{2}^{\T}\matX=(\u_{3}\otimes\u_{2})^{\T}(\matX\odot\matX)$
(note that this is a row vector), then the non-zero scalar $\left(\u_{3}^{\T}\matX\odot\u_{2}^{\T}\matX\right)\left(\u_{3}^{\T}\matX\odot\u_{2}^{\T}\matX\right)^{\T}=\norm{\B}_{2}^{2}$
. Using the formula for $\nabla F_{0}(\tenW)$ in Eq. (\ref{eq:critical d=00003D2 g=00003D0})
and unfolding by the first mode, we find that:
\begin{equation}
\matU_{1}\mathbf{c}=\norm{\B}_{2}^{-2}\matY\left(\u_{3}^{\T}\matX\odot\u_{2}^{\T}\matX\right)^{\T}=\norm{\B}_{2}^{-2}\matY\B^{\T}\label{eq:U1c=00003Dfunc(B)}
\end{equation}
(the fact that $\matB\neq0$ follows from the fact that $\matX$ has
full row rank). Furthermore, unfolding Eq. (\ref{eq:critical d=00003D2 g=00003D0})
using mode 2 and $3$, we find that 
\[
\u_{2}^{\T}\nabla F_{0}(\tenW)_{(2)}(\u_{3}\otimes\matU_{1}\mathbf{c})=\u_{3}^{\T}\nabla F_{0}(\tenW)_{(3)}(\u_{2}\otimes\matU_{1}\mathbf{c})=0
\]
All the $\v_{i}$s are zero as well, so using the fact that $(\c^{\T})^{\pinv}=\TNorm{\c}^{-2}\c$:
\begin{align}
\v_{2} & =(\I_{m}-\u_{2}\u_{2}^{\T})\nabla F_{0}(\tenW)_{(2)}(\u_{3}\otimes\matU_{1}\mathbf{c})=\nabla F_{0}(\tenW)_{(2)}(\u_{3}\otimes\matU_{1}\c)=0\label{eq:critical d=00003D2 Vi=00003D0}\\
\v_{3} & =(\I_{m}-\u_{3}\u_{3}^{\T})\nabla F_{0}(\tenW)_{(3)}(\u_{2}\otimes\matU_{1}\mathbf{c})=\nabla F_{0}(\tenW)_{(3)}(\u_{2}\otimes\matU_{1}\c)=0\nonumber 
\end{align}

We first consider $\tenW=\left\llbracket \matU_{1}\mathbf{c},\u,\u\right\rrbracket $
which is a stationary point. Substituting in the unfolding of the
Euclidean gradient $\nabla F_{0}(\tenW)=\left\llbracket \tenW\matX-\matY,\matX,\matX\right\rrbracket $
in Eq. (\ref{eq:critical d=00003D2 Vi=00003D0}) and rearranging,
we obtain:
\begin{align*}
\X\left(\X\odot\tenW\matX\right)^{\T}(\u\otimes\matU_{1}\mathbf{c})= & \X\left(\X\odot\matY\right)^{\T}(\u\otimes\matU_{1}\mathbf{c})
\end{align*}
First, we look at the left hand side:
\begin{equation}
\X\left(\X\odot\tenW\matX\right)^{\T}(\u\otimes\matU_{1}\mathbf{c})=\X\left[(\u\otimes\matU_{1}\mathbf{c})^{\T}\left(\X\odot\tenW\matX\right)\right]{}^{\T}=\X\left[\left(\u^{\T}\X\right)\odot(\matU_{1}\mathbf{c})^{\T}\left(\tenW\matX\right)\right]{}^{\T}\label{eq:lhs}
\end{equation}
Note that $\tenZ_{\matX\matX}\u=\X\left(\u^{\T}\X\odot\u^{\T}\X\odot\u^{\T}\X\right)^{\T}$.
Substituting $\tenW\matX=\text{\ensuremath{\matU_{1}}\ensuremath{\mathbf{c}\left(\u^{\T}\matX\odot\u^{\T}\matX\right)}}$in
Eq.~(\ref{eq:lhs}), we have
\[
\X\left[\left(\u^{\T}\X\right)\odot\mathbf{c}^{\T}\matU_{1}^{\T}\matU_{1}\mathbf{c}\left(\u^{\T}\matX\odot\u^{\T}\matX\right)\right]{}^{\T}=\TNormS{\c}\X\left(\u^{\T}\X\odot\u^{\T}\matX\odot\u^{\T}\matX\right)^{\T}=\TNormS{\c}\mathcal{Z}_{\matX\matX}\u
\]
Next, we look at the right hand side. Substituting in Eq. (\ref{eq:U1c=00003Dfunc(B)}):
\begin{align*}
\X\left(\X\odot\matY\right)^{\T}(\u\otimes\matU_{1}\mathbf{c}) & =\norm{\mathbf{B}}_{2}^{-2}\X\left(\X\odot\matY\right)^{\T}\left(\u\otimes\matY\left(\u^{\T}\X\odot\u^{\T}\matX\right)^{\T}\right)\\
 & =\norm{\mathbf{B}}_{2}^{-2}\X\left[\left(\u^{\T}\X\odot\left(\u^{\T}\X\odot\u^{\T}\matX\right)\matY^{\T}\matY\right)\right]^{\T}\\
 & =\norm{\mathbf{B}}_{2}^{-2}\X\left[\left(\u^{\T}\X\odot\u^{\T}\matX\right)\matY^{\T}\matY\left(\u^{T}\X\odot\I_{n}\right)\right]^{\T}\\
 & =\norm{\mathbf{B}}_{2}^{-2}\X\left(\u^{\T}\X\odot\I_{n}\right)^{\T}\left[\left(\u^{\T}\X\odot\u^{\T}\matX\right)\matY^{\T}\matY\right]^{\T}\\
 & =\norm{\mathbf{B}}_{2}^{-2}\left(1\otimes\X\right)\left(\u^{\T}\X\odot\I_{n}\right)\matY^{\T}\matY\left(\u^{\T}\X\odot\u^{\T}\matX\right)^{\T}\\
 & =\norm{\mathbf{B}}_{2}^{-2}\left(\u^{\T}\X\odot\X\right)\matY^{\T}\matY\left(\u^{\T}\X\odot\u^{\T}\matX\right)^{\T}=\norm{\mathbf{B}}_{2}^{-2}\tenZ_{\matX\matY}\u
\end{align*}
where in the third equality we used the fact that for any two row
vectors $\a,\b$ of same size we have $\a\odot\b=\b\diag(\a)=\b\left(\a\odot\I\right)$,
and in the fifth equality we used the fact that $\u^{\T}\X\odot\I_{n}=\diag(\u^{\T}\X)$,
a diagonal matrix, hence symmetric, and the final equality follows
from observing that $(\tenZ_{\matX\matY})_{(1)}=(\matX\odot\matX)\matY^{\T}\matY(\matX\odot\matX)^{\T}$
so $\tenZ_{\matX\matY}\u=(\matX\odot\matX)\matY^{\T}\matY(\matX\odot\matX)^{\T}(\u\otimes\u\otimes\u)$.
Altogether, we obtained 
\[
\tenZ_{\matX\X}\u=\left\Vert \mathbf{c}\right\Vert _{2}^{-2}\norm{\mathbf{B}}_{2}^{-2}\tenZ_{\matX\matY}\u
\]
(both $\c$ and $\matB$ can not be equal to $0$. If $\c=0$ then
$\tenW=0$). Concluding that $\u$ is a B-eigenvector of $\left(\tenZ_{\matX\X},\tenZ_{\matX\matY}\right)$,
with the corresponding B-eigenvalue 
\begin{equation}
\gamma=\left\Vert \mathbf{c}\right\Vert _{2}^{-2}\norm{\mathbf{B}}_{2}^{-2}=\left\Vert \mathbf{c}\right\Vert _{2}^{-2}\left[\left(\u^{\T}\matX\odot\u^{\T}\matX\right)\left(\u^{\T}\X\odot\u^{\T}\matX\right)^{\T}\right]^{-1}\label{eq: gamma =00003D (cB)^2}
\end{equation}

We now show that any B-eigenpair corresponds to a semi-symmetric stationary
point. Let $\left(\gamma,\u\right)$ be a B-eigenpair of the tensor
pencil $\left(\tenZ_{\matX\X},\tenZ_{\matX\matY}\right)$ and we let
$\tenW=\left\llbracket \mathbf{c},\u,\u\right\rrbracket =\left\llbracket \tenC\,;\,\matI_{k},\u,\u\right\rrbracket \in\manM_{[k,1,2]}$
with $\mathbf{c}=\norm{\u^{\T}\matX\odot\u^{\T}\matX}_{2}^{-2}\matY\left(\u^{\T}\matX\odot\u^{\T}\matX\right)^{\T}$.
Setting again $\B:=\u^{\T}\matX\odot\u^{\T}\matX$, we can write $\mathbf{c}=\norm{\B}_{2}^{-2}\matY\B^{\T}$
and its norm $\TNormS{\c}=\norm{\B}_{2}^{-4}\B\matY^{\T}\matY\B^{\T}$,
as well as $\u^{\T}\mathcal{Z}_{\X\X}\u=\norm{\B}_{2}^{2}$ and $\mathcal{\u^{\T}Z}_{\X\matY}\u=\B\matY^{\T}\matY\B^{\T}$.
Since $\u$ is a B-eigenvector, we have $\u^{\T}\mathcal{Z}_{\X\X}\u=\gamma\mathcal{\u^{\T}Z}_{\X\matY}\u$
and we can express the B-eigenvalue as: 
\begin{equation}
\gamma=\left\Vert \mathbf{c}\right\Vert _{2}^{-2}\norm{\B}_{2}^{-2}.\label{eq: gamma=00003Dc^-2 * B^-2}
\end{equation}
$\tenW$ is obviously semi-symmetric, so we need only to prove that
it is a stationary point. We prove this by showing that $\grad F_{0}(\tenW)=\left\{ \ten G_{0};0,\v_{2},\v_{3}\right\} =0$. 

First mode unfolding of $\tenG_{0}$ :

\begin{doublespace}
\begin{align*}
\tenG_{0_{(1)}} & =\left(\nabla F_{0}(\tenW)\modeprod 1\matI_{k}\modeprod 2\u^{\T}\modeprod 3\u^{\T}\right)_{(1)}=\nabla F_{0}(\tenW)_{(1)}\left(\u\otimes\u\right)\\
 & =\left(\tenW\matX-\matY\right)\left(\X\odot\X\right)^{\T}\left(\u\otimes\u\right)\\
 & =\mathbf{c}\B\B^{\T}-\matY\B^{\T}\\
 & =\norm{\B}_{2}^{2}\mathbf{c}-\matY\B^{\T}\\
 & =\norm{\B}_{2}^{2}\norm{\B}_{2}^{-2}\matY\B^{\T}-\matY\B^{\T}=0
\end{align*}
As for the factors, we can write $\tenC_{(2)}=\tenC_{(3)}=\mathbf{c^{\T}}$
hence $\tenC_{(2)}^{\pinv}=\tenC_{(3)}^{\pinv}=\left\Vert \mathbf{c}\right\Vert _{2}^{-2}\mathbf{c}$,
and note that $\v_{2}=\v_{3}$, so it is enough to analyze just one
of them.

\begin{align*}
\v_{2} & =\left\Vert \mathbf{c}\right\Vert _{2}^{-2}\P_{\u}^{\perp}\nabla F_{0}(\tenW)_{(2)}\left(\u\otimes\mathbf{c}\right)\\
 & =\left\Vert \mathbf{c}\right\Vert _{2}^{-2}\P_{\u}^{\perp}\X\left(\X\odot(\tenW\matX-\matY)\right)^{\T}\left(\u\otimes\mathbf{c}\right)\\
 & =\left\Vert \mathbf{c}\right\Vert _{2}^{-2}\P_{\u}^{\perp}\left[\X\left(\X\odot\tenW\matX\right)^{\T}\left(\u\otimes\mathbf{c}\right)-\X\left(\X\odot\matY\right)^{\T}\left(\u\otimes\mathbf{c}\right)\right]
\end{align*}
Substituting $\tenW\matX=\mathbf{c}\left(\u^{\T}\matX\odot\u^{\T}\matX\right)$
and $\mathbf{c}=\norm{\B}_{2}^{-2}\matY\B^{\T}$, the left term in
the parentheses becomes
\begin{align*}
\X\left(\X\odot\tenW\matX\right)^{\T}(\u\otimes\mathbf{c}) & =\X\left[\left(\u^{\T}\X\right)\odot\mathbf{c}^{\T}\mathbf{c}\left(\u^{T}\matX\odot\u^{T}\matX\right)\right]=\left\Vert \mathbf{c}\right\Vert _{2}^{2}\X\left(\u^{\T}\X\odot\u^{\T}\matX\odot\u^{\T}\matX\right)^{\T}=\left\Vert \mathbf{c}\right\Vert _{2}^{2}\mathcal{Z}_{\X\X}\u
\end{align*}
And the right term in the parentheses becomes\textcolor{blue}{:}
\begin{align*}
\X\left(\X\odot\matY\right)^{\T}\left(\u\otimes\mathbf{c}\right) & =\norm{\B}_{2}^{-2}\X\left[\left(\u^{\T}\X\odot\left(\left(\u^{\T}\X\odot\u^{\T}\matX\right)\matY^{\T}\matY\right)\right)\right]^{\T}\\
 & =\norm{\B}_{2}^{-2}\X\left[\left(\u^{\T}\X\odot\u^{\T}\matX\right)\matY^{\T}\matY\left(\u^{\T}\X\odot\I_{n}\right)\right]^{\T}\\
 & =\norm{\B}_{2}^{-2}\X\left(\u^{\T}\X\odot\I_{n}\right)^{\T}\left[\left(\u^{\T}\X\odot\u^{\T}\matX\right)\matY^{\T}\matY\right]^{\T}\\
 & =\norm{\B}_{2}^{-2}\left(1\otimes\X\right)\left(\u^{\T}\X\odot\I_{n}\right)\matY^{\T}\matY\left(\u^{\T}\X\odot\u^{\T}\matX\right)^{\T}\\
 & =\norm{\B}_{2}^{-2}\left(\u^{\T}\X\odot\X\right)\matY^{\T}\matY\left(\u^{\T}\X\odot\u^{\T}\matX\right)^{\T}=\norm{\B}_{2}^{-2}\mathcal{Z}_{\X\matY}\u
\end{align*}
using the same arguments as before. 

Concluding, we obtained 
\[
\v_{2}=\left\Vert \mathbf{c}\right\Vert _{2}^{-2}\P_{\u}^{\perp}\left[\left\Vert \mathbf{c}\right\Vert _{2}^{2}\mathcal{Z}_{\X\X}\u-\norm{\B}_{2}^{-2}\mathcal{Z}_{\X\matY}\u\right]=\P_{\u}^{\perp}\left[\gamma\mathcal{Z}_{\X\matY}\u-\left\Vert \mathbf{c}\right\Vert _{2}^{-2}\norm{\B}_{2}^{-2}\mathcal{Z}_{\X\matY}\u\right]=0
\]
Similarly, $\v_{3}=0$, and the point is a stationary point. 
\end{doublespace}
\end{proof}
The next lemma shows that the order of such stationary points based
on their eigenvalues is the same as their order based on the value
of cost function.
\begin{lem}
Suppose that $\tenW,\tilde{\tenW}\in\manM_{[k,1,2]}$ are two semi-symmetric
stationary points, and let $(\gamma,\u)$,$(\tilde{\gamma},\tilde{\u})$
be their corresponding eigenpairs. Then $\FNorm{\tenW\matX-\matY}\leq\FNorm{\tilde{\tenW}\matX-\matY}$
if and only if $\gamma\leq\tilde{\gamma}$.
\end{lem}

\begin{proof}
For any $\tenW'=\left\llbracket \matU_{1}\mathbf{c},\u_{2},\u_{3}\right\rrbracket \in\manM_{[k,1,2]}$
, using the notation of $\B:=\u_{3}^{\T}\matX\odot\u_{2}^{\T}\matX$
, we can write
\begin{align*}
\FNormS{\tenW'\matX-\matY} & =\Trace{\left(\tenW\matX-\matY\right)^{\T}\left(\tenW\matX-\matY\right)}=\\
 & =\Trace{\left(\tenW\matX\right)^{T}\left(\tenW\matX\right)}-2\Trace{\left(\tenW\matX\right)^{\T}\matY}+\FNormS{\matY}\\
 & =\FNormS{\matY}+\Trace{\matB^{\T}\c^{\T}\matU_{1}^{\T}\matU_{1}\c\matB}-2\Trace{\matY^{\T}\mathbf{U}_{1}\mathbf{c}\B}\\
 & =\FNormS{\matY}+\TNormS{\c}\cdot\Trace{\matB^{\T}\matB}-2\Trace{\matY^{\T}\mathbf{U}_{1}\mathbf{c}\B}\\
 & =\FNormS{\matY}+\norm{\mathbf{c}}_{2}^{2}\norm{\B}_{2}^{2}-2\Trace{\matY^{\T}\mathbf{U}_{1}\mathbf{c}\B}
\end{align*}
where we used the fact that $\matB$ is a row vector. For any stationary
point, not necessarily semi-symmetric, we have shown in Eq. (\ref{eq:U1c=00003Dfunc(B)}),
as part of the proof of Theorem~\ref{thm:tensor eigenpairs}, that
$\matU_{1}\mathbf{c}=\norm{\B}_{2}^{-2}\matY\B^{\T}$ , so $\norm{\mathbf{c}}_{2}^{2}=\norm{\B}_{2}^{-4}\B\mathbf{Y}^{\T}\matY\B^{\T}$.
Using these results and the cyclic property of the trace, we can write
\[
\Trace{\matY^{\T}\mathbf{U}_{1}\mathbf{c}\B}=\norm{\B}_{2}^{-2}\Trace{\matY^{\T}\matY\B^{\T}\B}=\norm{\B}_{2}^{-2}\Trace{\B\matY^{\T}\matY\B^{\T}}=\norm{\B}_{2}^{-2}\B\matY^{\T}\matY\B^{\T}=\norm{\mathbf{c}}_{2}^{2}\norm{\B}_{2}^{2}
\]
Substituting, we have
\begin{align}
\FNormS{\tenW'\matX-\matY} & =\FNormS{\matY}-\norm{\B}_{2}^{-2}\B\matY^{\T}\matY\B^{\T}=\FNormS{\matY}-\norm{\mathbf{c}}_{2}^{2}\norm{\B}_{2}^{2}\label{eq:cost: BYYB=00003D(CB)^2}
\end{align}
Note that in the above formula, $\matB$ corresponds to $\tenW'$
which is a stationary point.

Now, let $\tenW,\tilde{\tenW}\in\manM_{[k,1,2]}$ be two semi-symmetric
stationary points, with cores $\tenC,\tilde{\tenC}$ respectively,
and let $(\gamma,\u)$,$(\tilde{\gamma},\tilde{\u})$ be their corresponding
eigenpairs. From Eq. (\ref{eq: gamma=00003Dc^-2 * B^-2}), shown as
part of the proof of Theorem~\ref{thm:tensor eigenpairs}, we have
$\gamma=\left\Vert \mathbf{c}\right\Vert _{2}^{-2}\norm{\u^{\T}\matX\odot\u^{\T}\matX}_{2}^{-2}$
and $\tilde{\gamma}=\left\Vert \mathbf{\tilde{c}}\right\Vert _{2}^{-2}\norm{\tilde{\u}^{\T}\matX\odot\tilde{\u}^{\T}\matX}_{2}^{-2}.$
We note that both are non-negative scalars since $\tenW,\tilde{\tenW}$
are non-zero. Then

\begin{doublespace}
\begin{align*}
\FNorm{\tenW\matX-\matY} & \leq\FNorm{\tilde{\tenW}\matX-\matY}\\
 & \Updownarrow\\
\FNormS{\tenW\matX-\matY} & \leq\FNormS{\tilde{\tenW}\matX-\matY}\\
 & \Updownarrow\\
\FNormS{\matY}-\gamma^{-1} & \leq\FNormS{\matY}-\tilde{\gamma}^{-1}\\
 & \Updownarrow\\
\gamma & \leq\tilde{\gamma}
\end{align*}
\end{doublespace}
\end{proof}

\section{\label{sec:experiments}Experiments}

In this section we report experimental results with an implementation
of our proposed HORRR algorithm. The goal of the experiments is to
demonstrate the ability of Riemannian optimization to solve HORRR
problems, and to illustrate the use HORRR for downstream learning
tasks. We implement the HORRR solver in MATLAB, using the MANOPT library
for Riemannian optimization \cite{manopt2014}. Specifically, we use
the manifold of tensors with fixed multilinear rank in Tucker format
(create via \texttt{fixedranktensorembeddedfactory}) \footnote{Though, we had to fix several issues, and improve the efficiency of
a some operations.}, and the Riemannian Conjugate Gradient method. The experiments are
not designed to be exhaustive, nor the methods designed for efficiency
or are optimized, making clock time an unsuitable metric for assessing
performance. Thus, we only assess the quality of the solution found
by the method, and downstream classification error, and do not measure
running time.

\subsection{Synthetic Data}

The first set of experiments is on synthetically generated problems.
Problems are generated in the following manner. First, we generate
the feature matrix $\X\sim\mathcal{N}(0,1)^{m\times n}$ (i.i.d entries),
and the coefficient tensor $\tenW_{\text{true}}\in\manM_{[k,r,2]}$
(so, we use $d=2$). $\tenW_{\text{true}}$ is sampled randomly on
$\manM_{[k,r,2]}$ by first sampling the core $\mathcal{\sim N}(0,1)^{k\times r\times r}$
and then for each factor, sampling entries i.i.d from the normal distribution
and then orthogonalizing the matrix. Finally, we set $\Y=(\tenW_{\text{true}}+a\cdot\xi)\X$,
where $\xi\sim\mathcal{N}(0,1)^{k\times m\times m}$ (i.i.d entries)
represents random noise and $a>0$ is the noise level. 

We report results for $k=100,m=100$ and $n=10,000$ samples. We test
the algorithm with varying multilinear rank $\mathbf{r}$, (Figure
\ref{fig: synth multi data}), varying regularization parameter ($\lambda$),
and noise level ($a$) (Figure \ref{fig: synth noise}). We report
the Relative Recovery Error (RRE)
\[
\text{RRE }(\tenW)\coloneqq\frac{\norm{\tenW\X-\tenW_{\text{true}}\X}_{F}}{\norm{\tenW_{\text{true}}\X}_{F}}
\]
as the main figure-of-merit. Predictive error (on out of sample $\X$)
was also inspected, but found to be similar to the recovery error,
and as such was excluded from the results presented here. For each
configuration, we generated five synthetic problem instances, truncated
collected statistics to the lowest number of iterations between runs,
and plot mean RRE, as well as a range between min and max. Two versions
of our algorithm are used: plain HORRR without recoring (Section \ref{subsec:Recoring}),
and another where we perform a single recore operation midway through
the iterations. 

Results for varying the multilinear rank $\mathbf{r}$ are shown in
Figure \ref{fig: synth multi data}, where we plot RRE as a function
of the number of iterations for several values of $r$. The regularization
parameter was set to $\lambda=10^{-3}$, and the noise level to $a=10^{-3}$
or $10^{-2}$. We see the ability of HORRR to achieve low values of
RRE (recall that due to the noise it is unrealistic to expect very
small values of RRE). An important observation is with regards to
the recore operation. We see that strategically choosing the iteration
in which we do a recore does accelerate convergence, though the plain
algorithm catches up eventually, and achieves similar accuracy. So,
recore is more effective if we settle to stop the iterations early
with a suboptimal solution. 

Next, we inspected how the final RRE is effected by the regularization
parameter $\lambda$ as well as different noise levels. RRE was almost
not affected at all by the regularization parameter, and the effects
observed are not consistent. Thus, we do not report results for final
RRE as a function of $\lambda$. As for the noise level, results are
reported in Figure \ref{fig: synth noise}, for $r=20$ and $\lambda=0.01$.
As expected, increasing noise levels makes it more challenging to
reconstruct the tensor accurately.

\begin{figure}
\begin{centering}
\includegraphics[width=0.4\textwidth]{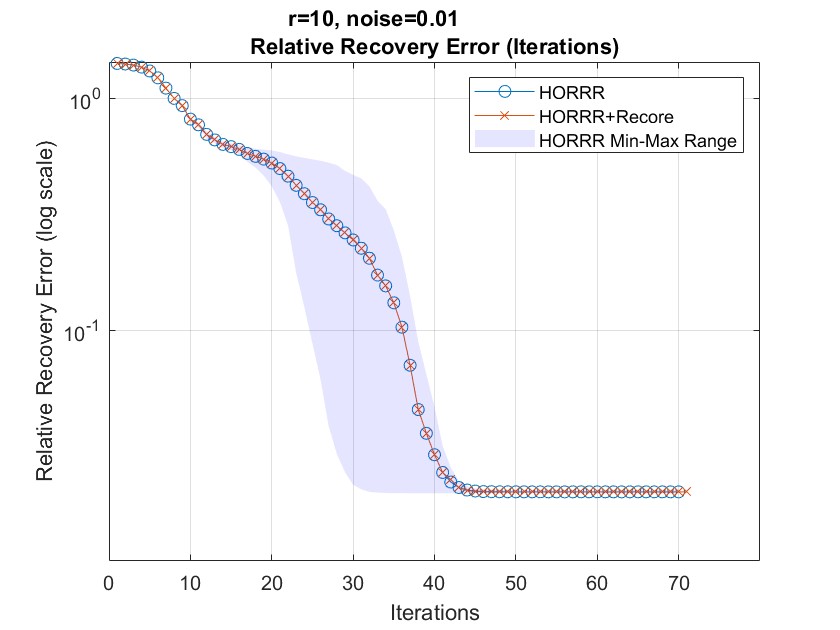}\includegraphics[width=0.4\textwidth]{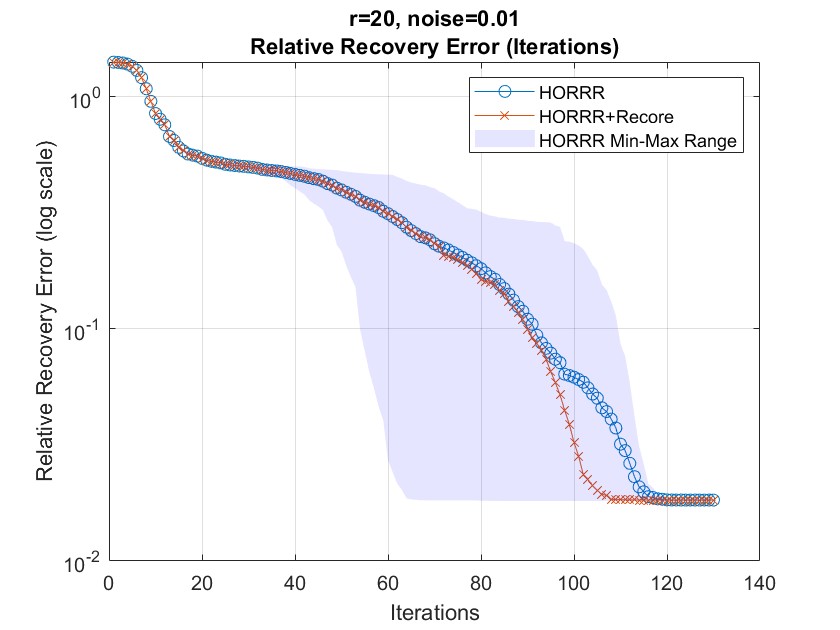}
\par\end{centering}
\begin{centering}
\includegraphics[width=0.4\textwidth]{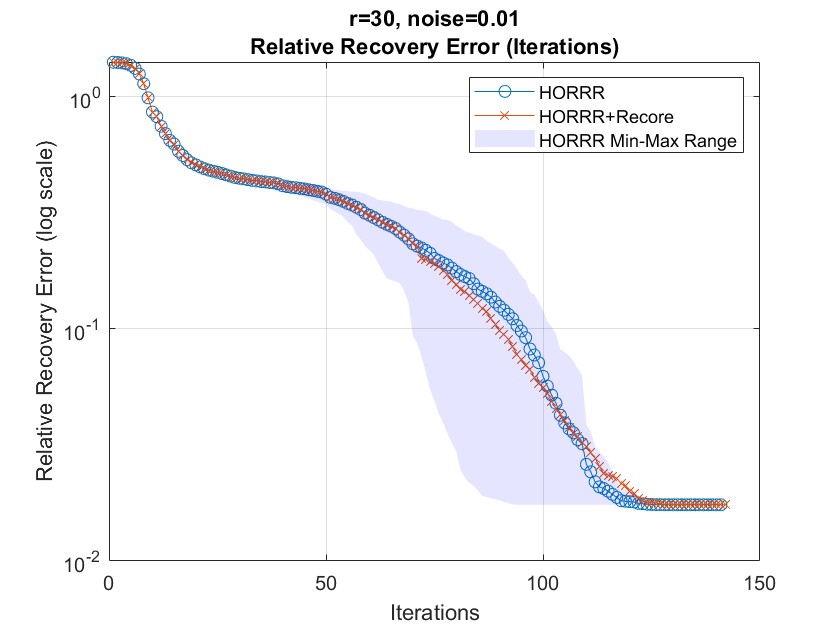}\includegraphics[width=0.4\textwidth]{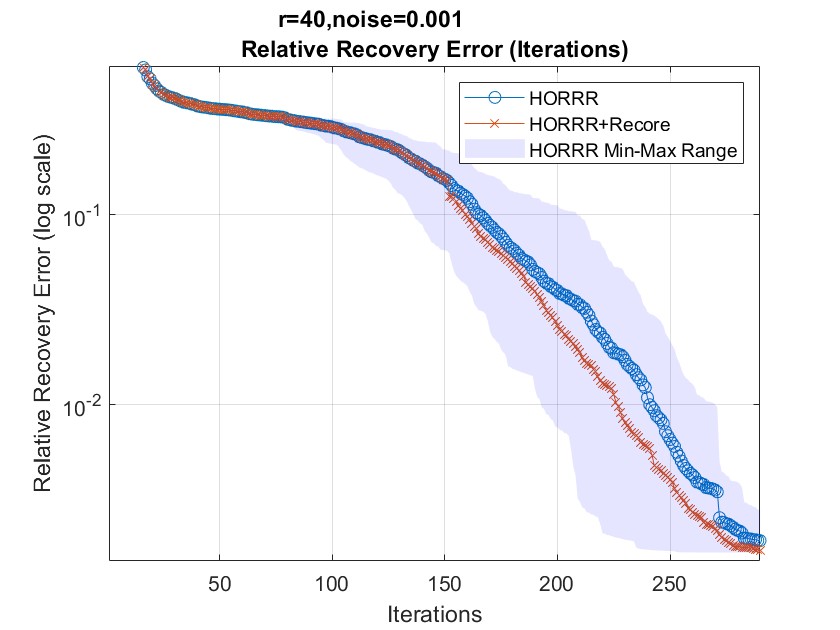}
\par\end{centering}
\caption{RRE as a function of number of iterations $(d=2)$.}
\label{fig: synth multi data}
\end{figure}

\begin{figure}
\centering{}\label{fig: synth noise}\includegraphics[scale=0.3]{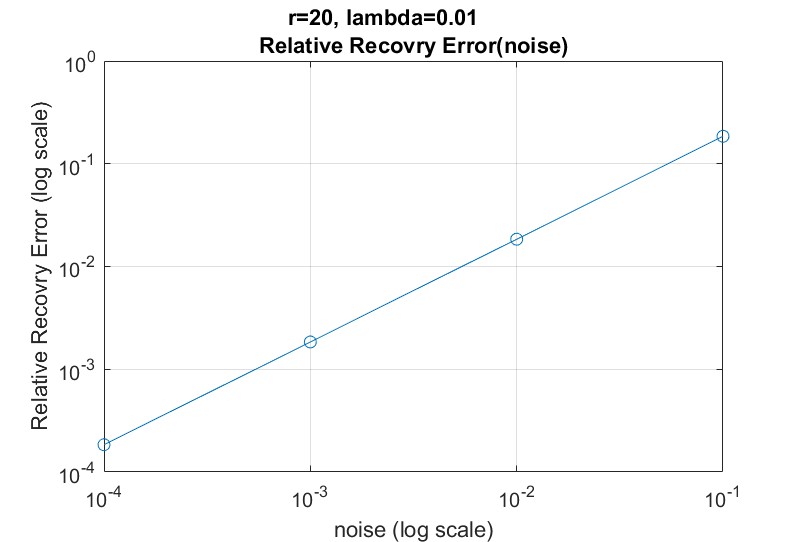}\caption{Final RRE as a function of $a$ (noise level).}
\end{figure}

\subsection{Using HORRR for Classification}

\begin{figure}
\begin{centering}
\includegraphics[width=0.8\textwidth]{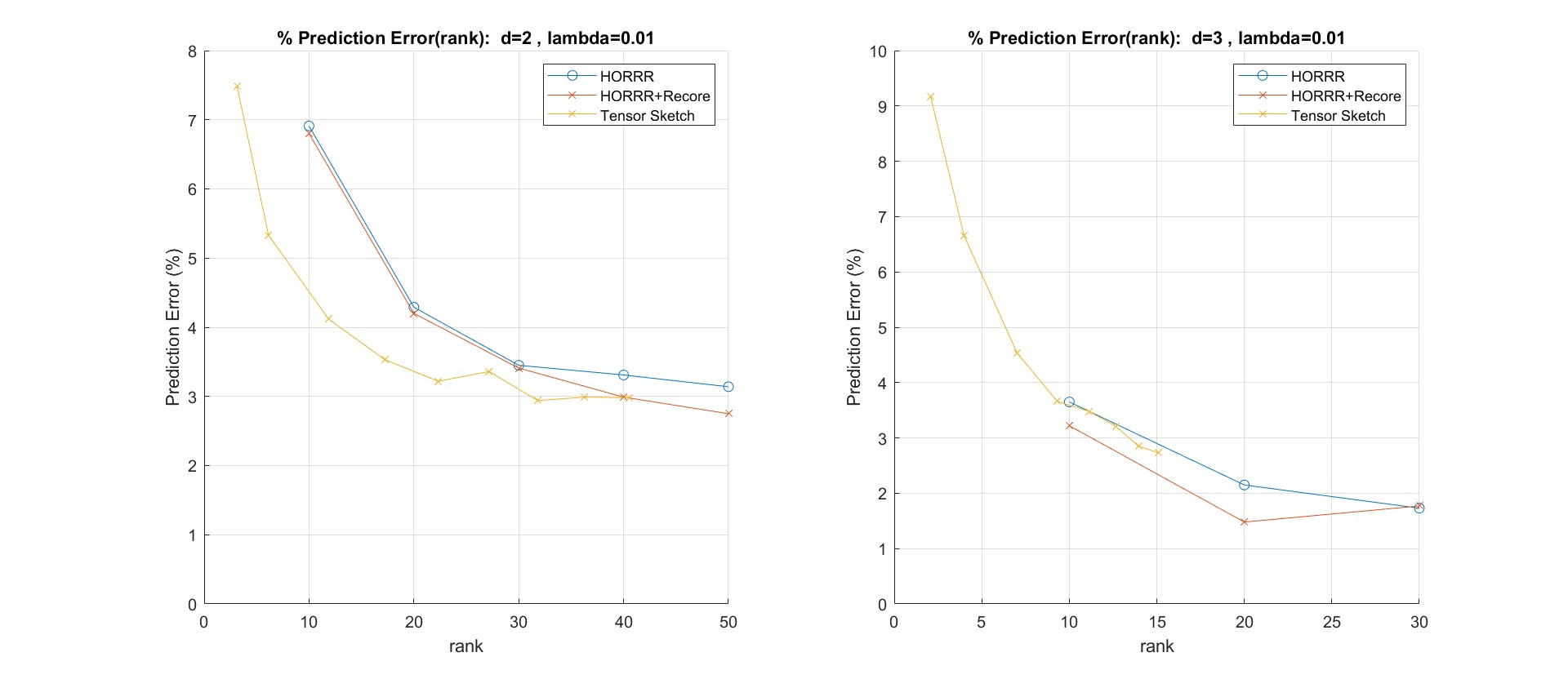}
\par\end{centering}
\centering{}\caption{MNIST, $\lambda=10^{-2}$.}
\label{fig:MNIST}
\end{figure}

We demonstrate the potential use of HORRR for real-world tasks by
using HORRR to build a non-linear classifier for MNIST. We do not
claim that our HORRR based classification method is the best method
for building a classifier for MNIST, or in general, but merely aim
to give an illustration that HORRR can be leveraged for downstream
learning tasks.

MNIST is is a dataset of $n=60000$ small $28$-by-$28$ gray-scale
images of handwritten digits between $0$ and $9$. The goal is to
learn a model that is able to deduce the digit present in future images.
Models are evaluated on an additional test dataset consisting of $10000$
images. We setup HORRR to solve this problem by using Regularized
Least Squares Classification, i.e. encode the class (digit label)
using a one-hot vector embedding (i.e vector of size $10$). The model
is described by a tensor $\tenW\in\manM_{[10,r,d]}$, where $r$ and
$d$ are hyper-parameters, and the prediction function is $f_{\ten W}(\x)=\tenW\x$
where $\x\in\R^{784}$ is a vectorization of the image. The output
of $f_{\ten W}(\x)$ is a score vector in $\R^{10}$, where entry
$i$ corresponds to the digit $i$, from which a prediction can be
extracted by finding the maximum entry. The coefficients $\tenW$
are fitted by solving the HORRR problem $\min_{\tenW}\FNormS{\tenW\matX-\matY}+\lambda\FNormS{\tenW}$
where $\matX\in\R^{784\times60000}$ and $\matY\in\R^{10\times60000}$
is the horizontal stacking of the one-hot embedding of the responses. 

We remark that RRR (i.e. $d=1$) learns a very poor classifier. The
reason is that for any rank deficient matrix $\matW$, the prediction
$\matW\matX$ is rank deficient as well, but $\matY$ is not well
approximated by a rank deficient matrix (this is not special to the
MNIST datasets, but a typical situation for one-hot encodings of labels).
In contrast, for $d\geq2$, the matrix $\tenW\matX$ can be, and usually
is, full rank, even though $\tenW$ might have low multilinear rank.

Predictions errors for $d=2$ and $d=3$ are presented in Figure \textcolor{blue}{\ref{fig:MNIST}}
for various ranks. The best prediction error achieved is $1.48\%$
for $d=3$ and $r=20$ with recoring. Curiously, in $d=3$ with recoring
we see improved results for $r=20$ over $r=30$, which suggest that
multilinear rank constraints acting as regularization might have improved
the model. However, we do not run the optimization long enough to
find a minima, and the benefit $r=20$ over $r=30$ might be due to
the found result being closer to a minima.

The prediction function learned by our model is multivariate homogeneous
polynomial. Thus, a natural baseline is the result of kernel ridge
regression (KRR) with polynomial kernel, which too learns a multivariate
polynomial, though only ridge regularization is applied to the coefficients
(see Section \ref{sec:HORRR problem statement}). It is challenging
to apply KRR to the full MNIST dataset, since that requires creating
and inverting a $60000$-by-$60000$ matrix, and this requires substantial
memory and compute time. Instead, we rely on published results in
\cite{avron2017faster}, which reports for $d=3$ an error of $1.45\%$.
We see that our method achieved similar results to the full kernel
method, but with only a fraction of the memory required. For $d=3,r=20$,
holding the model in memory for our method (i.e. keeping $\tenW$
in memory in factored form) requires $0.27\%$ of the memory required
to keep all the data points (required for applying the model). In
terms of the learning process, KRR requires forming the $60000$-by-$60000$
kernel matrix, and the storage for $\tenW$ requires only $0.0035\%$
of the memory required for that.

The solution of KRR can also be approximated. To provide another baseline,
we consider approximate KRR using \noun{TensorSketch}. In \noun{TensorSketch}
the kernel matrix $\matK\in\R^{n\times n}$ is approximated by $\matZ\matZ^{\T}$where
$\matZ\in\R^{n\times s}$ and $s$ is an hyperparameter (number of
random features). Higher values of $s$ lead to a better approximation
of $\matK$, while also increasing the cost of learning the model
(which scales as $O(nks^{2})$) and downstream predictions. In Figure
\textcolor{blue}{\ref{fig:MNIST}} we also plot results for the \noun{TensorSketch}
based classifier. To put both our method and the \noun{TensorSketch}
based method on the same graph, while both having fundamentally different
parameters (rank $r$ for our method, number of features $s$ for
\noun{TensorSketch}) we look for parity in the number of parameters
between models. That is, given a result with some value of $s$, we
compute $r_{\text{equivalent}}$ such that the number of parameters
in the learned \noun{TensorSketch} model (which equals to $sk$) is
equal to the the number of parameters with a rank $r_{\text{equivalent}}$
HORRR model (which equals $kr^{d}+dmr$). Figure \ref{fig:MNIST}
illustrates that the prediction error achieved by HORRR is comparable
to TensorSketch for $d=2$ as the rank increases, and improved accuracy
when applying a recore step for $d=3$, reinforcing the validity of
our approach.

\section{Conclusions}

This paper introduced Higher Order Reduced Rank Regression (HORRR)
as an extension of Reduced Rank Regression (RRR), capable of capturing
nonlinear relationships in multi-response regression problems. HORRR
utilizes tensor-based representations and incorporates Tucker decomposition
to impose multilinear rank constraints, leveraging Riemannian manifold
optimization to efficiently solve the constrained regression problem.
Empirical results demonstrated the efficacy of HORRR in both synthetic
and real-world scenarios. On synthetic datasets, HORRR achieved low
relative recovery errors (RRE) and exhibited robust performance under
varying noise levels. The recore operation further enhanced convergence
speed. When applied to the MNIST dataset for classification, we were
able to learn on a desktop computer a classifier with accuracy on
par with exact kernel based method using the polynomial kernel as
well as a sketched approximation of it. Such learning requires significantly
more storage resources then available to our solver. These results
demonstrate the potential of HORRR to handle complex, nonlinear, and
multi-dimensional regression workloads. 

In terms of theoretical analysis, we have shown that for $d=1$, in
which HORRR coincides with classic RRR, the global solution is the
only stable minima of our algorithm's underlying Riemannian optimization.
We have also partially established this for $d=2$ and $r=1$. In
practice our algorithm works well even for higher $d$. Further progress
on establishing stability of results for high order is challenging,
and we leave it for future work. An additional future research topic
is the incorporation of randomized preconditioning (possibly via sketching)
to accelerate the convergence. 

\section*{Acknowledgements}

This research was supported by the Israel Science Foundation (Grant
no. 1524/23).

\bibliographystyle{plain}
\bibliography{tensors}

\begin{thebibliography}{10}

\bibitem{absil2008optimization-book}
P-A Absil, Robert Mahony, and Rodolphe Sepulchre.
\newblock {\em Optimization algorithms on matrix manifolds}.
\newblock Princeton University Press, 2008.

\bibitem{absil2013extrinsic-riem-hess}
P.~A. Absil, Robert Mahony, and Jochen Trumpf.
\newblock An extrinsic look at the {Riemannian Hessian}.
\newblock In Frank Nielsen and Fr{\'e}d{\'e}ric Barbaresco, editors, {\em
  Geometric Science of Information}, pages 361--368, Berlin, Heidelberg, 2013.
  Springer Berlin Heidelberg.

\bibitem{anderson1951RRR}
Theodore~Wilbur Anderson.
\newblock Estimating linear restrictions on regression coefficients for
  multivariate normal distributions.
\newblock {\em The Annals of Mathematical Statistics}, pages 327--351, 1951.

\bibitem{anderson1999asymptotic}
Theodore~Wilbur Anderson.
\newblock Asymptotic distribution of the reduced rank regression estimator
  under general conditions.
\newblock {\em The Annals of Statistics}, 27(4):1141--1154, 1999.

\bibitem{avron2017faster}
H.~Avron, K.L. Clarkson, and D.P. Woodruff.
\newblock {Faster kernel ridge regression using sketching and preconditioning}.
\newblock {\em SIAM Journal on Matrix Analysis and Applications}, 38(4), 2017.

\bibitem{KoldaBader2006algorithm}
Brett~W Bader and Tamara~G Kolda.
\newblock Algorithm 862: Matlab tensor classes for fast algorithm prototyping.
\newblock {\em ACM Transactions on Mathematical Software (TOMS)},
  32(4):635--653, 2006.

\bibitem{BensonGleich19eigen}
Austin~R. Benson and David~F. Gleich.
\newblock {Computing Tensor Z-Eigenvectors with Dynamical Systems}.
\newblock {\em SIAM Journal on Matrix Analysis and Applications},
  40(4):1311--1324, 11 2019.

\bibitem{manopt2014}
N.~Boumal, B.~Mishra, P.-A. Absil, and R.~Sepulchre.
\newblock {M}anopt, a {M}atlab toolbox for optimization on manifolds.
\newblock {\em Journal of Machine Learning Research}, 15(42):1455--1459, 2014.

\bibitem{boumal2023book}
Nicolas Boumal.
\newblock {\em An introduction to optimization on smooth manifolds}.
\newblock Cambridge University Press, 2023.

\bibitem{chang2009-B-eigen}
Kung-Ching Chang, Kelly Pearson, and Tan Zhang.
\newblock On eigenvalue problems of real symmetric tensors.
\newblock {\em Journal of Mathematical Analysis and Applications},
  350(1):416--422, 2009.

\bibitem{chen2012reduced}
Kun Chen, Kung-Sik Chan, and Nils~Chr Stenseth.
\newblock Reduced rank stochastic regression with a sparse singular value
  decomposition.
\newblock {\em Journal of the Royal Statistical Society Series B: Statistical
  Methodology}, 74(2):203--221, 2012.

\bibitem{chen2013reduced}
Kun Chen, Hongbo Dong, and Kung-Sik Chan.
\newblock Reduced rank regression via adaptive nuclear norm penalization.
\newblock {\em Biometrika}, 100(4):901--920, 2013.

\bibitem{davies1982-RRR-svd}
PT~Davies and M~KS Tso.
\newblock Procedures for reduced-rank regression.
\newblock {\em Journal of the Royal Statistical Society Series C: Applied
  Statistics}, 31(3):244--255, 1982.

\bibitem{de2000multilinearSVD}
Lieven De~Lathauwer, Bart De~Moor, and Joos Vandewalle.
\newblock A multilinear singular value decomposition.
\newblock {\em SIAM journal on Matrix Analysis and Applications},
  21(4):1253--1278, 2000.

\bibitem{silva-2008-tensor-rank}
Vin de~Silva and Lek-Heng Lim.
\newblock Tensor rank and the ill-posedness of the best low-rank approximation
  problem.
\newblock {\em SIAM Journal on Matrix Analysis and Applications},
  30(3):1084--1127, 2008.

\bibitem{ding2015generalized}
Weiyang Ding and Yimin Wei.
\newblock Generalized tensor eigenvalue problems.
\newblock {\em SIAM Journal on Matrix Analysis and Applications},
  36(3):1073--1099, 2015.

\bibitem{eckart-young-1936}
Carl Eckart and Gale Young.
\newblock The approximation of one matrix by another of lower rank.
\newblock {\em Psychometrika}, 1(3):211--218, 1936.

\bibitem{heidel2018riemannian}
Gennadij Heidel and Volker Schulz.
\newblock A {Riemannian} trust-region method for low-rank tensor completion.
\newblock {\em Numerical Linear Algebra with Applications}, 25(6):e2175, 2018.

\bibitem{izenman2008modern}
Alan~J Izenman.
\newblock {\em Modern multivariate statistical techniques}, volume~1.
\newblock Springer, 2008.

\bibitem{izenman1975-RRR}
Alan~Julian Izenman.
\newblock Reduced-rank regression for the multivariate linear model.
\newblock {\em Journal of multivariate analysis}, 5(2):248--264, 1975.

\bibitem{kasai2016low}
Hiroyuki Kasai and Bamdev Mishra.
\newblock Low-rank tensor completion: a {Riemannian} manifold preconditioning
  approach.
\newblock In {\em International conference on machine learning}, pages
  1012--1021. PMLR, 2016.

\bibitem{KochLuibch10tensor-manifold}
Othmar Koch and Christian Lubich.
\newblock Dynamical tensor approximation.
\newblock {\em SIAM Journal on Matrix Analysis and Applications},
  31(5):2360--2375, 2010.

\bibitem{KoldaBader09}
Tamara~G. Kolda and Brett~W. Bader.
\newblock {Tensor Decompositions and Applications}.
\newblock {\em SIAM Review}, 51(3):455--500, 8 2009.

\bibitem{Kressner2013completion}
Daniel Kressner, Michael Steinlechner, and Bart Vandereycken.
\newblock {Low-rank tensor completion by Riemannian optimization}.
\newblock {\em BIT Numerical Mathematics 2013 54:2}, 54(2):447--468, 11 2013.

\bibitem{Kressner2016precond}
Daniel Kressner, Michael Steinlechner, and Bart Vandereycken.
\newblock Preconditioned low-rank {Riemannian} optimization for linear systems
  with tensor product structure.
\newblock {\em SIAM Journal on Scientific Computing}, 38(4):A2018--A2044, 2016.

\bibitem{levakova2024penalisation}
Marie Levakova and Susanne Ditlevsen.
\newblock Penalisation methods in fitting high-dimensional cointegrated vector
  autoregressive models: A review.
\newblock {\em International Statistical Review}, 92(2):160--193, 2024.

\bibitem{liu2024efficient}
Xiao Liu, Weidong Liu, and Xiaojun Mao.
\newblock Efficient and provable online reduced rank regression via online
  gradient descent.
\newblock {\em Machine Learning}, pages 1--38, 2024.

\bibitem{llosa2022reduced}
Carlos Llosa-Vite and Ranjan Maitra.
\newblock Reduced-rank tensor-on-tensor regression and tensor-variate analysis
  of variance.
\newblock {\em IEEE Transactions on Pattern Analysis and Machine Intelligence},
  45(2):2282--2296, 2022.

\bibitem{Mukherjee2011kernel}
Ashin Mukherjee and Ji~Zhu.
\newblock Reduced rank ridge regression and its kernel extensions.
\newblock {\em Statistical analysis and data mining: the ASA data science
  journal}, 4(6):612--622, 2011.

\bibitem{NRO22}
Alexander Novikov, Maxim Rakhuba, and Ivan Oseledets.
\newblock Automatic differentiation for {Riemannian} optimization on low-rank
  matrix and tensor-train manifolds.
\newblock {\em SIAM Journal on Scientific Computing}, 44(2):A843--A869, 2022.

\bibitem{Qi05}
Liqun Qi.
\newblock {Eigenvalues of a real supersymmetric tensor}.
\newblock {\em Journal of Symbolic Computation}, 40(6):1302--1324, 12 2005.

\bibitem{rabusseau2016HOLRR}
Guillaume Rabusseau and Hachem Kadri.
\newblock Low-rank regression with tensor responses.
\newblock {\em Advances in Neural Information Processing Systems}, 29, 2016.

\bibitem{reinsel-velu-1998RRR}
Gregory~C Reinsel and Raja~P Velu.
\newblock {\em Multivariate reduced-rank regression}.
\newblock Springer, 1998.

\bibitem{she2017robust}
Yiyuan She and Kun Chen.
\newblock Robust reduced-rank regression.
\newblock {\em Biometrika}, 104(3):633--647, 2017.

\bibitem{Uschmajew2013}
Andre Uschmajew and Bart Vandereycken.
\newblock The geometry of algorithms using hierarchical tensors.
\newblock {\em Linear Algebra and its Applications}, 439:133--166, 7 2013.

\bibitem{vandereycken2013low}
Bart Vandereycken.
\newblock Low-rank matrix completion by {Riemannian} optimization.
\newblock {\em SIAM Journal on Optimization}, 23(2):1214--1236, 2013.

\bibitem{wold1975PLS}
Herman Wold.
\newblock Soft modelling by latent variables: the non-linear iterative partial
  least squares (nipals) approach.
\newblock {\em Journal of Applied Probability}, 12(S1):117--142, 1975.

\bibitem{yu2016learning}
Rose Yu and Yan Liu.
\newblock Learning from multiway data: Simple and efficient tensor regression.
\newblock In {\em International Conference on Machine Learning}, pages
  373--381. PMLR, 2016.

\bibitem{yuan2007dimension}
Ming Yuan, Ali Ekici, Zhaosong Lu, and Renato Monteiro.
\newblock Dimension reduction and coefficient estimation in multivariate linear
  regression.
\newblock {\em Journal of the Royal Statistical Society Series B: Statistical
  Methodology}, 69(3):329--346, 2007.

\bibitem{zhao2012HOPLS}
Qibin Zhao, Cesar~F Caiafa, Danilo~P Mandic, Zenas~C Chao, Yasuo Nagasaka,
  Naotaka Fujii, Liqing Zhang, and Andrzej Cichocki.
\newblock Higher order partial least squares ({HOPLS}): A generalized
  multilinear regression method.
\newblock {\em IEEE transactions on pattern analysis and machine intelligence},
  35(7):1660--1673, 2012.

\end{thebibliography}

\appendix

\section{\label{appendix: Invariance under orthogonalization}Invariance of
the Classification of the Stationary Points Under Linear Transformation
for $d=1$}

In this appendix we prove two technical lemmas needed for the analysis
of stationary points in $d=1$. They essentially show that under linear
transformations, the transformed stationary points keep the properties
such as local/global minima. 
\begin{lem}
\label{claim:open sets}Let $g:\manM_{[k,r,1]}\rightarrow\manM_{[k,r,1]}$
defined by $g(\W)\coloneqq\W\cdot\mathbf{\matL}$, where $\mathbf{\matL}$
is a $m\times m$ invertible matrix. Then $g(\W)$ transfers open
sets (neighborhoods) in $\manM_{[k,r,1]}$ to open sets (neighborhoods)
in $\manM_{[k,r,1]}$, with the Frobenius norm as a metric.
\end{lem}

\begin{proof}
Recall that $\manM_{[k,r,1]}$ is the manifold of $k$-by-$m$ matrices
of rank $r$ ($m$ is implied by context). We first note that $g:\manM_{[k,r,1]}\rightarrow\manM_{[k,r,1]}$
is well defined since $\matW\mathbf{\matL}\in\manM_{[k,r,1]}$. The
reason is that multiplying $\W$ by an invertible matrix preserves
the rank. Thus the range is indeed $\manM_{[k,r,1]}$.

Let $U\subseteq\manM_{[k,r,1]}$ be an open set. Let $U':=g(U)=\left\{ g(\W)|\W\in U\right\} $.
We want to show that $U'$ is an open set. Let $\matW'=g(\W)\in U'$.
We need to build a neighborhood of $\matW'$ contained in $U'$. By
assumption $U$ is an open set, so there exists $\epsilon>0$ such
that every point $\tilde{\matW}$ such that $\FNorm{\matW-\tilde{\matW}}<\epsilon$
we have $\tilde{\matW}\in U.$ Let $\epsilon'=\epsilon/\FNorm{\matL^{-1}}$.
We will show that any $\tilde{\matW}'\in\manM_{[k,r,1]}$ such that
$\FNorm{\matW'-\tilde{\matW}'}<\epsilon'$, we have $\tilde{\matW}'\in U'$.
Now let $\tilde{\matW}=\tilde{\matW}'\matL^{-1}$, then $\tilde{\matW}'=g(\tilde{\matW}$).
So 
\[
\epsilon'>\FNorm{\matW'-\tilde{\matW}'}=\FNorm{g(\matW)-g(\tilde{\matW})}=\FNorm{\matW\matL-\tilde{\matW}\matL}\geq\FNorm{\matW-\tilde{\matW}}/\FNorm{\matL^{-1}}
\]
So
\[
\FNorm{\matW-\tilde{\matW}}<\epsilon'\cdot\FNorm{\matL^{-1}}=\epsilon
\]
which implies that $\tilde{\matW}\in U$ which, in turn, implies that
$\tilde{\matW}'=g(\tilde{\matW})\in U'$. Thus, we have shown that
for any point in $U'$ there is a neighborhood contained in $U'$,
so by definition it is an open set. 
\end{proof}
\begin{lem}
Let $F(\W)=\FNormS{\matW\matX-\matY}$ , $\X=\mathbf{\matL\Q}$ such
that $\Q$ has orthonormal rows and $\mathbf{\matL}$ is lower triangular
$m\times m$ invertible matrix. Define $g(\W)=\W\mathbf{\matL}$ and
$F'(g(\W))=\FNormS{g(\W)\Q-\matY}$. Then $\W^{*}$ is a local minimum
of $F$ in $\manM_{[k,r,1]}$ if and only if $g(\W^{*})$ is a local
minimum of $F'$ in $\manM_{[k,r,1]}$.
\end{lem}

\begin{proof}
Let $\W^{*}$ be a local minimum of $F$ in $\manM_{[k,r,1]}$, then
there exists some neighborhood $E$ of $\W^{*}$ in $\manM_{[k,r,1]}$,
such that $\forall\W\in E$, $F(\W^{*})<F(\W)$. From the previous
claim, we deduce that $E':=g(E)=\left\{ g(\W)|\W\in E\right\} $ is
a neighborhood of $g(\W^{*})$. Suppose, for the sake of contradiction,
that $g(\W^{*})$ is \uline{not} a local minimum of $F'$ in $\manM_{[k,r,1]}$:
then there exists some $g(\W)\in E'$ such that $F'(g(\W))\leq F'(g(\W^{*}))$.
then 

\[
\FNormS{\W\X-\matY}=\FNormS{\W\matL\Q-\matY}=\FNormS{g(\W)\Q-\matY}\leq\FNormS{g(\W^{*})\Q-\matY}=\FNormS{\W^{*}\mathbf{\matL}\Q-\matY}=\FNormS{\W^{*}\X-\matY}
\]
which is a contradiction to $\W^{*}$ being a local minimum of $F$
in $E$.

Next, we can define $g^{-1}:\manM_{[k,r,1]}\rightarrow\manM_{[k,r,1]},$$g^{-1}(\W)=\W\mathbf{\matL^{-1}}$
and note the the previous claim holds for $g^{-1}$ as well: if $E'$
is a neighborhood of $g(\W^{*})$ in $\manM_{[k,r,1]}$, than $E:=g^{-1}(E')=\left\{ g^{-1}(g(\W))=\W|g(\W)\in E'\right\} $
is a neighborhood of $\W^{*}$, and we can prove the opposite direction
in the same manner. 
\end{proof}

\section{\label{appendix: hessian}Riemannian Hessian: Calculation of the
Curvature term}

In this appendix we show the development of the curvature term of
the Riemannian Hessian for our HORRR problem. Our starting point is
formulas for $\tilde{\tenC}$ (Eq. (\ref{eq:c_tilde})) and $\tilde{\mathbf{U}}_{i}$
(Eq. (\ref{eq:u_tilde})) originally introduced in \cite{heidel2018riemannian}.

We start with $\tilde{\tenC}:$\textcolor{blue}{{} }

\[
\tilde{\tenC}=\sum_{j=1}^{d+1}\left(\nabla F_{\lambda}(\tenW)\modeprod j\V_{j}^{\T}\modeprod{l\neq j}\matU_{l}^{\T}-\tenC\modeprod j\left(\V_{j}^{\T}\left[\nabla F_{\lambda}(\tenW)\modeprod{l\neq j}\matU_{l}^{\T}\right]_{(j)}\unfold{\tenC}j^{+}\right)\right)
\]
The first term in the summation vanishes, since $\V_{1}=0_{k\times k}$,
though we keep it for ease of notation. Denoting $\matR:=\tenW\matX-\matY$,
we have, for all $j\in[d+1]$:

\begin{align*}
\nabla F_{\lambda}(\tenW)\modeprod j\V_{j}^{\T}\modeprod{l\neq j}\matU_{l}^{\T} & =\left(\tenJ_{n}^{(d+1)}\modeprod 1\matR\modeprodrange{j=2}{d+1}\matX\right)\modeprod j\V_{j}^{\T}\modeprod{l\neq j}\matU_{l}^{\T}+\lambda\left(\tenC\modeprodrange{p=1}{d+1}\matU_{p}\right)\modeprod j\V_{j}^{\T}\modeprod{l\neq j}\matU_{l}^{\T}\\
 & =\tenJ_{n}^{(d+1)}\modeprod 1\matU_{1}^{\T}\matR\modeprod j\V_{j}^{\T}\X\modeprod{l\neq j}\matU_{l}^{\T}\X+\lambda\left(\tenC\modeprod j\V_{j}^{\T}\matU_{j}\modeprod{l\neq j}\matU_{l}^{\T}\matU_{l}\right)\\
 & =\tenJ_{n}^{(d+1)}\modeprod 1\matU_{1}^{\T}\matR\modeprod j\V_{j}^{\T}\X\modeprod{l\neq j}\matU_{l}^{\T}\X+0
\end{align*}
The last equality follows from $\V_{j}^{T}\matU_{j}=0$ for all $j\in[d+1]$.
Continuing to the next term, for $j\in[d+1]$: 

\begin{align*}
\V_{j}^{\T}\left[\nabla F_{\lambda}(\tenW)\modeprod{l\neq j}\matU_{l}^{\T}\right]_{(j)} & \tenC_{(j)}^{+}=\V_{j}^{\T}\left[\left(\tenJ_{n}^{(d+1)}\modeprod 1\matR\modeprodrange{j=2}{d+1}\matX\right)\modeprod{l\neq j}\matU_{l}^{\T}+\lambda\tenW\modeprod{l\neq j}\matU_{l}^{\T}\right]_{(j)}\tenC_{(j)}^{+}\\
 & =\V_{j}^{\T}\left[\left(\tenJ_{n}^{(d+1)}\modeprod 1\matU_{1}^{\T}\matR\modeprodrange j{}\matX\modeprod{l\neq j}\matU_{l}^{\T}\matX\right)+\lambda\left(\tenC\modeprodrange{p=1}{d+1}\matU_{p}\right)\modeprod{l\neq j}\matU_{l}^{\T}\right]_{(j)}\tenC_{(j)}^{+}\\
 & =\V_{j}^{\T}\left[\left(\tenJ_{n}^{(d+1)}\modeprod 1\matU_{1}^{\T}\matR\modeprodrange j{}\matX\modeprod{l\neq j}\matU_{l}^{\T}\matX\right)+\lambda\tenC\modeprod j\matU_{j}\right]_{(j)}\tenC_{(j)}^{+}\\
 & =\V_{j}^{\T}\X\left(\matU_{d+1}^{\T}\matX\odot\cdots\odot\matU_{j+1}^{\T}\matX\odot\matU_{j-1}^{\T}\matX\odot\cdots\odot\matU_{2}^{\T}\matX\odot\matU_{1}^{T}\matR\right)^{\T}\tenC_{(j)}^{\dagger}+\lambda\V_{j}^{\T}\matU_{j}\tenC_{(j)}\tenC_{(j)}^{+}\\
 & =\V_{j}^{\T}\X\left(\matU_{d+1}^{\T}\matX\odot\cdots\odot\matU_{j+1}^{\T}\matX\odot\matU_{j-1}^{\T}\matX\odot\cdots\odot\matU_{2}^{\T}\matX\odot\matU_{1}^{T}\matR\right)^{\T}\tenC_{(j)}^{+}
\end{align*}
Where the last equality follows from $\V_{j}^{T}\matU_{j}=0$. Concluding
the calculation for $\tilde{\mathcal{C}}$: 
\begin{align*}
\tilde{\tenC} & =\sum_{j=2}^{d+1}\left(\tenJ_{n}^{(d+1)}\modeprod 1\matU_{1}^{\T}\matR\modeprod j\V_{j}^{\T}\X\modeprod{l\neq j}\matU_{l}^{\T}\X\right.\\
 & \;\;\;\;\;\;\;\;\;\;\left.-\tenC\modeprod j\left(\V_{j}^{\T}\X\left(\matU_{d+1}^{\T}\matX\odot\cdots\odot\matU_{j+1}^{\T}\matX\odot\matU_{j-1}^{\T}\matX\odot\cdots\odot\matU_{2}^{\T}\matX\odot\matU_{1}^{T}\matR\right)^{\T}\tenC_{(j)}^{+}\right)\right)
\end{align*}

For a stationary point, the last addend, using lemma \ref{lem:condition for stationary points},
will be zero:

\[
\V_{j}^{\T}\X\left(\matU_{d+1}^{\T}\matX\odot\cdots\odot\matU_{j+1}^{\T}\matX\odot\matU_{j-1}^{\T}\matX\odot\cdots\odot\matU_{2}^{\T}\matX\odot\matU_{1}^{T}\matR\right)^{\T}=\lambda\V_{j}^{\T}\matU_{j}=0
\]
Thus we have a simplified formula: 

\[
\tilde{\tenC}=\sum_{j=2}^{d+1}\left(\tenJ_{n}^{(d+1)}\modeprod 1\matU_{1}^{\T}\matR\modeprod j\V_{j}^{\T}\X\modeprod{l\neq j}\matU_{l}^{\T}\X\right)
\]

We move on to calculate $\tilde{\mathbf{U}}_{i}$'s. Since it is in
the tangent space, we have $\tilde{\mathbf{U}}_{1}=0_{k\times k}$.
For $i=2,...,d+1$:

\[
\tilde{\mathbf{U}}_{i}=\P_{U_{i}}^{\perp}\left(\left[\nabla F_{\lambda}(\tenW)\modeprod{i\neq j}\matU_{j}^{\T}\right]_{(i)}\left(\I-\tenC_{(i)}^{+}\tenC_{(i)}\right)\tenG_{(i)}^{\T}\tenC_{(i)}^{+\T}+\sum_{l\neq i}\left[\nabla F_{\lambda}(\tenW)\modeprod l\V_{l}^{\T}\modeprod{l\neq j\neq i}\matU_{j}^{\T}\right]_{(i)}\right)\tenC_{(i)}^{+}
\]
As calculated before, we have

\begin{align*}
\left[\nabla F_{\lambda}(\tenW)\modeprod{i\neq j}\matU_{j}^{\T}\right]_{(i)} & =\X\left(\matU_{d+1}^{\T}\matX\odot\cdots\odot\matU_{i+1}^{\T}\matX\odot\matU_{i-1}^{\T}\matX\odot\cdots\odot\matU_{2}^{\T}\matX\odot\matU_{1}^{T}\matR\right)^{\T}+\lambda\matU_{i}\tenC_{(i)}
\end{align*}
and 

\[
\left[\nabla F_{\lambda}(\tenW)\modeprod{i\neq j}\matU_{j}^{\T}\right]_{(i)}\left(\I-\tenC_{(i)}^{+}\tenC_{(i)}\right)=\X\left(\matU_{d+1}^{\T}\matX\odot\cdots\odot\matU_{i+1}^{\T}\matX\odot\matU_{i-1}^{\T}\matX\odot\cdots\odot\matU_{2}^{\T}\matX\odot\matU_{1}^{T}\matR\right)^{\T}\left(\I-\tenC_{(i)}^{+}\tenC_{(i)}\right)
\]
Where the regularization term nullified since $\unfold{\tenC}i\unfold{\tenC}i^{\pinv}=\I_{r}$.

For a stationary point, once again, when applying the 2nd condition
from lemma \ref{lem:condition for stationary points}, in the following
form: 

\begin{align*}
\lambda\unfold{\tenC}i & =-\matU_{i}^{\T}\matX\left(\matU_{d+1}^{\T}\matX\odot\cdots\odot\matU_{i+1}^{\T}\matX\odot\matU_{i-1}^{\T}\matX\odot\cdots\odot\matU_{2}^{\T}\matX\odot\matU_{1}^{T}\matR\right)^{\T}
\end{align*}
we obtain

\[
\matU_{i}^{\T}\left[\nabla F_{0}(\tenW)\modeprod{i\neq j}\matU_{j}^{\T}\right]_{(i)}\left(\I-\tenC_{(i)}^{+}\tenC_{(i)}\right)=-\lambda\unfold{\tenC}i\left(\I-\tenC_{(i)}^{+}\tenC_{(i)}\right)=0
\]
Thus we have 

\[
\P_{U_{i}}^{\perp}\left(\left[\nabla F_{\lambda}(\tenW)\modeprod{i\neq j}\matU_{j}^{\T}\right]_{(i)}\left(\I-\tenC_{(i)}^{+}\tenC_{(i)}\right)\tenG_{(i)}^{\T}\tenC_{(i)}^{+\T}\tenC_{(i)}^{+}\right)=\left[\nabla F_{\lambda}(\tenW)\modeprod{i\neq j}\matU_{j}^{\T}\right]_{(i)}\left(\I-\tenC_{(i)}^{+}\tenC_{(i)}\right)\tenG_{(i)}^{\T}\left(\tenC_{(i)}\tenC_{(i)}^{T}\right)^{-1}
\]

As for the last addend:

\begin{align*}
\left[\nabla F_{\lambda}(\tenW)\modeprod{l\neq i}\V_{l}^{\T}\modeprod{l\neq j\neq i}\matU_{j}^{\T}\right]_{(i)} & =\left(\tenJ_{n}^{(d+1)}\modeprod 1\matR\modeprodrange{p=2}{d+1}\matX\right)\modeprod{l\neq i}\V_{l}^{\T}\modeprod{j\neq l\neq i}\matU_{j}^{\T}+\lambda\left(\tenC\modeprodrange{p=1}{d+1}\matU_{p}\right)\modeprod{l\neq i}\V_{l}^{\T}\modeprod{j\neq l\neq j}\matU_{j}^{\T}\\
 & =\left[\tenJ_{n}^{(d+1)}\modeprod 1\matU_{1}^{\T}\matR\modeprod i\X\modeprod l\V_{l}^{\T}\X\modeprod{j\neq l\neq i}\matU_{j}^{\T}\X+\lambda\left(\tenC\modeprod i\matU_{i}\modeprod j\V_{j}^{\T}\matU_{j}\modeprod{l\neq j}\matU_{l}^{\T}\matU_{l}\right)\right]_{(i)}\\
 & =\left[\tenJ_{n}^{(d+1)}\modeprod 1\matU_{1}^{\T}\matR\modeprod i\X\modeprod l\V_{l}^{\T}\X\modeprod{j\neq l\neq i}\matU_{j}^{\T}\X\right]_{(i)}\\
 & =\X\left(\matU_{d+1}^{\T}\matX\odot\cdots\odot\V_{l}^{\T}\X\odot...\odot\matU_{i+1}^{\T}\matX\odot\matU_{i-1}^{\T}\matX\odot\cdots\odot\matU_{2}^{\T}\matX\odot\matU_{1}^{T}\matR\right)^{\T}
\end{align*}
where in the 3rd equality we see that the regularization term nullifies
again. 

Concluding the calculation for $\tilde{\mathcal{\U}_{i}}$: 

\begin{align*}
\tilde{\mathbf{U}}_{i} & =\P_{U_{i}}^{\perp}\left(\X\left(\matU_{d+1}^{\T}\matX\odot\cdots\odot\matU_{i+1}^{\T}\matX\odot\matU_{i-1}^{\T}\matX\odot\cdots\odot\matU_{2}^{\T}\matX\odot\matU_{1}^{T}\matR\right)^{\T}\left(\I-\tenC_{(i)}^{+}\tenC_{(i)}\right)\tenG_{(i)}^{\T}\tenC_{(i)}^{+\T}\right.\\
 & \left.\;\;\;\;\;\;\;\;\;+\sum_{l\neq i}\left[\X\left(\matU_{d+1}^{\T}\matX\odot\cdots\odot\V_{l}^{\T}\X\odot...\odot\matU_{i+1}^{\T}\matX\odot\matU_{i-1}^{\T}\matX\odot\cdots\odot\matU_{2}^{\T}\matX\odot\matU_{1}^{T}\matR\right)^{\T}\right]\right)\tenC_{(i)}^{+}
\end{align*}

\section{\label{appendix: convergence}On Convergence of Riemannian Optimization
for HORRR}

Proposition \ref{prop: convergence} suggests the convergence of HORRR.
However, the set $\manM_{\rb}$ is not closed. To enforce that every
accumulation point $\tenW_{*}\in\manM_{\rb}$, we regularize the cost
function in such a way that the iterates $\tenW_{k}$ stay inside
a compact subset of $\manM_{\rb}$. Thus, we discuss the convergence
of a modification of the original cost function, which is done by
regularizing the singular values of the matricizations of $\tenW$.
We define the modified cost function $\widehat{F_{\lambda}}$ with
regularization parameter $\tau>0:$

\[
\widehat{F_{\lambda}}:\manM_{\rb}\rightarrow\R,\;\tenW\longmapsto F_{\lambda}(\tenW)+\tau^{2}\sum_{i_{1}=1}^{d+1}\left(\left\Vert \tenW_{(i)}\right\Vert _{F}^{2}+\left\Vert \unfold{\tenW}i^{\pinv}\right\Vert _{F}^{2}\right)
\]

\begin{prop}
Let $\left(\tenW_{k}\right)_{k\in\N}$ be an infinite sequence of
iterates generated by HORRR Algorithm with the modified function $\widehat{F_{\lambda}}$.
Then $\lim_{k\to\infty}\FNorm{\grad(\hat{F_{\lambda}}(\ten W_{k})}=0.$ 
\end{prop}

\begin{proof}
The proof is essentially the same as the proof of \cite[Prop. 3.2]{Kressner2013completion},
only with a different cost function. The only property of the cost
function that the proof of \cite[Prop. 3.2]{Kressner2013completion}
uses is that it is non-negative, which holds for our $F_{\lambda}$
as well. 
\end{proof}
The regularization parameter $\tau$ can be chosen arbitrarily small.
If the matricizations are always of full row rank during the optimization
procedure (even when $\tau\rightarrow0$), then the accumulation points
$\tenW_{*}$ are guaranteed to stay inside $\manM_{\mathbf{r}}$ and
grad $\widehat{F_{\lambda}}(\tenW_{k})\rightarrow0$ as $\tau\rightarrow0$.
In this case, optimizing the modified cost function $\widehat{F_{\lambda}}$
is equivalent to the original cost function $F_{\lambda}$.
\end{document}